\documentclass[11pt]{article}

\usepackage[top=2.5cm,bottom=2.5cm,left=2.5cm,right=2.5cm,a4paper]{geometry}
\usepackage{subfigure}
\usepackage{graphicx}
\usepackage[normalem]{ulem}

\usepackage{natbib}

\usepackage{marvosym}

\renewcommand{\cite}[1]{\citep{#1}}

\usepackage{hyperref}       % hyperlinks
\hypersetup{
  colorlinks=true,
  linkcolor=red,
  urlcolor=magenta,
  citecolor=blue
}

\usepackage{amsmath,amsfonts,amssymb,amsthm}
\usepackage{graphicx,color}
\usepackage{url}
\usepackage{subfiles}
\usepackage{booktabs}
\usepackage{mathtools}
\usepackage{microtype}
\usepackage[english]{babel}
\usepackage{caption}
\usepackage{algorithm, algorithmic}
\usepackage{bm}
\usepackage{multirow,multicol}
\usepackage[toc,page,title,titletoc]{appendix}

\newcommand{\E}{\mathbb{E}}
\newcommand{\R}{\mathbb{R}}

\newcommand{\mF}{\mathcal{F}}

\newcommand{\SCerror}{C_{sc}}

\newtheorem{case}{\textbf{Case}}

\newcommand{\strongly}{\mu}
\newcommand{\gradvar}{G}
\newcommand{\smooth}{\beta}
\newcommand{\lip}{L}
\newcommand{\diam}{D}
\newcommand{\dataset}{\mathcal{D}} % dataset
\newcommand{\dist}{P} % underlying distribution
\newcommand{\smallprob}{\alpha_1}
\newcommand{\smallprobtail}{\alpha_2}
\newcommand{\boundary}{\gamma}
\newcommand{\margin}{\nu}
\newcommand{\Lx}{M}
\newcommand{\fs}{t_0}
\newcommand{\Cpara}{C_{in}}
\newcommand{\subopt}{h_{sub}}

\newcounter{daggerfootnote}
\newcommand*{\daggerfootnote}[1]{%
    \setcounter{daggerfootnote}{\value{footnote}}%
    \renewcommand*{\thefootnote}{\fnsymbol{footnote}}%
    \footnotetext[1]{#1}%
    \setcounter{footnote}{\value{daggerfootnote}}%
    \renewcommand*{\thefootnote}{\arabic{footnote}}%
    }
\newtheorem{definition}{Definition}[section]
\newtheorem{assumption}{Assumption}[section]
\newtheorem{theorem}{Theorem}[section]
\newtheorem{lemma}{Lemma}[section]
\newtheorem{proposition}{Proposition}[section]
\newtheorem{corollary}{Corollary}[section]
\newtheorem{remark}{Remark}[section]
\usepackage{amsmath}
\numberwithin{equation}{section}
\numberwithin{table}{section}
\numberwithin{figure}{section}

\title{
  On Private Online Convex Optimization: \\ Optimal Algorithms in $\ell_p$-Geometry and High Dimensional Contextual Bandits
}

\usepackage{authblk}
\author[*$\dagger$]{Yuxuan Han}
\author[*$\dagger$]{Zhicong Liang}
\author[*$\ddagger$]{Zhipeng Liang}
\author[$\dagger$$\ddagger$]{\\Yang Wang}
\author[$\dagger$\Letter]{Yuan Yao}
\author[$\dagger$$\ddagger$\Letter]{Jiheng Zhang}

\affil[$\dagger$]{Department of Mathematics}
\affil[$\ddagger$]{Department of Industrial Engineering and Decision Analytics}
\affil[$\ $]{The Hong Kong University of Science and Technology}
\affil[\Letter]{Emails: \url{yuany@ust.hk},~\url{jiheng@ust.hk}}
\date{}

\begin{document}
\maketitle

\begin{abstract}
Differentially private (DP) stochastic convex optimization (SCO) is ubiquitous in trustworthy machine learning algorithm design.
This paper studies the DP-SCO problem with streaming data sampled from a distribution and arrives sequentially.
We also consider the continual release model where parameters related to private information are updated and released upon each new data, often known as the online algorithms. 
Despite that numerous algorithms have been developed to achieve the optimal excess risks in different $\ell_p$ norm geometries, yet none of the existing ones can be adapted to the streaming and continual release setting. To address such a challenge as the online convex optimization with privacy protection, 
we propose a private variant of online Frank-Wolfe algorithm with recursive gradients for variance reduction to update and reveal the parameters upon each data.
Combined with the adaptive differential privacy analysis, our online algorithm achieves in linear time the optimal excess risk when $1<p\leq 2$ and
the state-of-the-art excess risk meeting the non-private lower ones when $2<p\leq\infty$.
Our algorithm can also be extended to the case $p=1$ to achieve nearly dimension-independent excess risk.
While previous variance reduction results on recursive gradient have theoretical guarantee only in the independent and identically distributed sample setting, we establish such a guarantee in a non-stationary setting. 
To demonstrate the virtues of our method, we design the first DP algorithm for high-dimensional generalized linear bandits with logarithmic regret. Comparative experiments with a variety of DP-SCO and DP-Bandit algorithms exhibit the efficacy and utility of the proposed algorithms. 
\end{abstract}

    
\daggerfootnote{Equal contributions.}

\newpage
{
  \hypersetup{linkcolor=black}
  \tableofcontents
}

\section{Introduction}Stochastic convex optimization (SCO) is a fundamental
problem widely studied in machine learning, statistics, and operations research. 
The goal of SCO is to minimize a population loss function $F_P(\theta) = \E_{x\sim P}[f(\theta, x)] $ over a $d$-dimensional support set $\mathcal{C}$, with only access to the independent and identically distributed (i.i.d.) or exchangeable samples $\{x_{t}\}_{t=1}^n$ from some distribution $P$. 
The performance of an algorithm is often measured in terms of the excess population risk of its solution $\theta$, i.e., $F_P(\theta)-\min_{v\in \mathcal{C}}F_P(v)$. 
In practice, samples related to users' profiles might contain sensitive information; thus, it is important to solve stochastic convex optimization problems with \emph{differential privacy} guarantees (DP-SCO) 
\cite{bassily2014private,bassily2019private,bassily2021differentially}. 

In this paper, we consider the DP-SCO with \textit{streaming data}, where samples arrive sequentially and cannot be stored in memory for long, often known as \emph{online} algorithms in literature.
\textit{Streaming data} has been studied in the context of online learning \cite{SmaYao06,Yao10,tarres2014online}, online statistical inference \cite{Vovk01,Vovk09,steinhardt2014statistics, fang2018online}, and online optimization \cite{zinkevich2003online,cesa2006prediction, hazan2019introduction, hoi2021online}. 
Moreover, data release is concerned due to privacy requirements.
Our online method can also accommodate \emph{continual release} \cite{jain2021price, dwork2010differential, chan2011private}, i.e., receiving sensitive data as a stream of input and releases an output of it immediately after processing while satisfying differential privacy requirements. 
Such a private online algorithm can be formulated as a recursive update process $\theta_t = \Theta_t(\theta_{t-1},x_t,\varepsilon_t)$ where $\varepsilon_t$ encodes the differential privacy noise and $\Theta_t$ is the update mapping, e.g. the online Frank-Wolfe algorithm considered in this paper. 

A closely related setting considered as an extension in this paper is so-called online decision making \cite{slivkins2019introduction, lattimore2020bandit}, where a decision needs to be made at each time, and the performance is measured in terms of accumulative regret, the gap between actual reward and the best possible reward, over the time. 
Recent work starts introducing streaming algorithms in (private) SCO into the solution of the online decision-making \cite{ding2021efficient, han2021generalized} to enjoy high computational efficiency and flexibility to handle different reward structures.
In particular, \citet{han2021generalized} propose to solve private contextual bandits with stochastic gradient descent (SGD).
However, the extension of other streaming algorithms, including the Frank-Wolfe and the stochastic mirror descent, remains elusive in this setting.

Compared with non-private SCO, private SCO has an inherent dependence on the dimension $d$ \cite{agarwal2012information, bassily2021non}.
Therefore, in DP-SCO, the optimal convergence rate also has a crucial dependence on the space metric.
Remarkable progress has been made in achieving optimal rate in $\ell_p$ norm with various $1\le p\le \infty$ as shown in Table~\ref{tbl:table-comparison}.
However, no existing rate-optimal DP-SCO algorithms can be adopted in the streaming and continual release setting since they rely on either Frank-Wolfe or mirror descent  with batched-gradient estimator. 
Algorithms relying on Frank-Wolfe require a batch size of $\tilde{\Omega}(n)$ \cite{bassily2021non, asi2021private} for variance reduction, which is unacceptable in the streaming setting.
Algorithms based on mirror descent require the same batch size and need a superlinear number of gradient query of $\tilde{\Omega}(n^{3/2})$ \cite{asi2021private}.

% &   &  & Thm.~\ref{thm-convergence-lp-strongly} & $O(\log^2(T)+ d \log^3/\varepsilon^2)$ \\
  
  %Change to equation (xxx).

  \begin{table*}[t]
	\centering
	  \caption{Bounds for excess population risk of $(\varepsilon, \delta)$-DP-SCO. $\dagger$ denotes bounds in expectation while $\ddagger$ denotes bounds with high probability. And $^*$ denotes bounds without smoothness assumption. Here $\kappa=\min\{\frac{1}{p-1}, 2\log d\}$.  
	  Most of the DP-SCO algorithms require large batch size and thus fail to be adopted in streaming data, except for \cite{feldman2020private}. However their algorithm can only release the last variable for privacy protection and thus contradict to the requirements of continual release.
	  }
	  \label{tbl:table-comparison}
	  \resizebox{\textwidth}{!}{
	  \begin{tabular}{ccccccc}
	  \toprule
	  Loss & $\ell_p$ & Theorem & Gradient Queries & Rate & Batch Size \\
	  \midrule
	  \multirow{18}{*}{Convex} &  \multirow{2}{*}{$p=2$} & Thm. 3.2 \cite{bassily2019private} & $O(\min \{n^{3 / 2}, \frac{n^{5/2}}{d}\})$ &  $O(\sqrt{\frac{1}{n}}+\frac{\sqrt{d}}{\epsilon n})^{\dagger}$  & $O(\sqrt{\varepsilon n})$\\
	   &  & Thm. 3.5 \cite{feldman2020private} & $O(\min \{n, \frac{n^{2}}{d}\})$ & $O(\sqrt{\frac{1}{n}}+\frac{\sqrt{d}}{\epsilon n})^{\dagger}$ & $ O(\frac{\sqrt{d}}{\varepsilon})$\\
	   \cline{2-6}\\
	   & \multirow{3}{*}{$p=1$} & Thm. 7 \cite{asi2021private} & $O(n)$ & $\tilde{O}(\sqrt{\frac{\log d }{n}}+(\frac{\log d}{\varepsilon n})^{2/3})^{\dagger}$ & $O(\frac{n}{\log^2 n})$
	%    \footnote{\cite{asi2021private} solve $p=1$ case with localization technique proposed in \cite{feldman2020private}}
	   \\
	   & & Thm. 3.2 \cite{bassily2021non}  &  $O(n)$ & $\tilde{O}(\frac{\log d}{\varepsilon \sqrt{n}})^{\dagger}$ & $O(\frac{n}{2}$)\\
	   &  & Theorem~\ref{thm-convergence-l1-general} & $O(n)$ & $\tilde{O}(\frac{\log d}{\varepsilon\sqrt{n}})^{\ddagger}$ & 1\\
	   \cline{2-6}\\
	   & $1<p<2$ & Thm. 5.4 \cite{bassily2021non}  & $O(n)$ & $\tilde{O}(\frac{\kappa}{\sqrt{n}}+\frac{\kappa \sqrt{d}}{\varepsilon n^{3 / 4}})^{\dagger}$&$O(\frac{n}{2})$
	%    \footnote{Algorithm 2 in fact only needs $\frac{1}{\sqrt{\varepsilon}}$ maximum batch size but can still achieve sub-optimal rate and thus it is omitted.}
	   \\
	   \cline{2-6}\\
	   & \multirow{2}{*}{$1<p\le 2$} & Thm. 13 \cite{asi2021private}  & $O(n^{3/2})$ & $\tilde{O}(\frac{1}{\sqrt{(p-1)n}}+\frac{\sqrt{d}}{(p-1)n\varepsilon})^{\dagger *}$& $O(\frac{n}{2})$\\
	   &  & Theorem~\ref{thm-convergence-lp-general} & $O(n)$ & $\tilde{O}(\sqrt{\frac{\kappa}{n}}+\frac{\sqrt{\kappa d}}{n \varepsilon})^{\ddagger}$ & 1\\
	   
	   \cline{2-6} \\
	   & \multirow{2}{*}{$2<p\leq \infty$} & Prop. 6.1 \cite{bassily2021non} & $O(n^2)$ & $\tilde{O} (\frac{d^{1/2-1/p}}{\sqrt{n}} + \frac{d^{1-1/p}}{\varepsilon n})^{\dagger *}$ & $O(n)$ \\
	   & & Theorem~\ref{thm-convergence-lp-general} & $O(n)$ & $\tilde{O} (\frac{d^{1/2-1/p}}{\sqrt{n}} + \frac{d^{1-1/p}}{\varepsilon n})^\ddagger$ & 1	 \\
	  \cline{1-6}\\
	  
	  \multirow{6}{*}{Strongly Convex}&  \multirow{2}{*}{$p=1$} & Thm. 9 \cite{asi2021private} & $O(n)$ & $\tilde{O}(\frac{\log d}{n}+(\frac{\log d}{\varepsilon n})^{4/3})^{\dagger}$&$O(\frac{n}{2\log n })$\\
	  & & Theorem~\ref{thm-convergence-l1-strongly} & $O(n)$ & $\tilde{O}(\frac{\log^2 d}{\varepsilon^2 n})^{\ddagger}$ & 1\\
	  \cline{2-6}\\
	  & $1<p\le 2$ & Theorem~\ref{thm-convergence-lp-strongly} & $O(n)$ & $\tilde{O}(\frac{\kappa}{n} + \frac{\kappa d}{\varepsilon^2 n^2})^{\ddagger}$ & 1\\
	  \cline{2-6} \\
	  & $2<p\leq \infty$ & Theorem~\ref{thm-convergence-lp-strongly} & $O(n)$ & $\tilde{O}(\frac{d^{1-2/p}}{n} + \frac{d^{2-2/p}}{ \varepsilon^2 n^2})^\ddagger$ & 1\\
	  
	  \bottomrule
	  \end{tabular}
	  }
  \end{table*}

\subsection{Our Contributions}

Note that in the online setting, the total time step $T$ equals the sample size $n$. 
So we will use $n$ instead of $T$ for the total number of iterations. 
Excess population risk bounds denoted by $t$ hold for every time step $t\in[n]$, while those denoted by $n$ only hold after $\Omega(n)$ time steps.

\paragraph{Case of $1<p\leq 2$.}

We present a systematic study on a differentially private  online Frank-Wolfe algorithm with recursive gradient
in various $\ell_p$ geometries, which is rate-optimal for $1<p\leq 2$. Our algorithm is based on the observation that the non-private recursive variance reduction scheme used in \citet{xie2020efficient} can be written as a normalized incremental summation of gradients. According to this observation, we can apply the tree-based mechanism in \citet{guha2013nearly}, and utilize an adaptive argument to show that our noise accumulates logarithmically as the total number of iteration grows, comparing with the polynomial grow rate in \citet{bassily2021non} and \citet{asi2021private}. In this case, our algorithm can fit in the online setting where a large number of updates is required. Such an analysis leads to a variance reduced gradient error bound of $\tilde{O}(\frac{1}{\sqrt{t}} + \frac{\sqrt{d}}{t\varepsilon})$ with high probability 

The recursive gradient method we used here is closely related to \citet{bassily2021non}, while their algorithm uses $\frac{n}{2}$ samples for variance reduction, and their gradient error is of the order $\tilde{O}(\frac{1}{\sqrt{n}}+\frac{\sqrt{d}}{\varepsilon n^{3/4}})$ in the worst-case.
Our improvement on the variance reduction reduces their $\tilde{O}(\frac{1}{\sqrt{n}}+ \frac{\sqrt{d}}{n^{3/4}\varepsilon} )$ excess risk to $\tilde{O}(\frac{1}{\sqrt{t}} +\frac{\sqrt{d}}{t\varepsilon} )$, which is optimal up to a logarithmic factor. \citet{asi2021private} achieves the optimal rates in terms of $n$ and $d$ at the cost of $O(n^{3/2})$ gradient queries while we achieve the same rate with only $O(n)$ gradient queries.
Moreover, their rate will explode to $+\infty$ when $p$ approaches 1, while our dependency on $p$ is upper bounded by $\log d$.

One thing we need to mention here is that, Theorem 13 of \citet{asi2021private} achieves the optimal rate without smoothness assumption, as mentioned in Table~\ref{tbl:table-comparison}. As shown in \cite{bassily2021non} and \cite{asi2021private}, smooth and non-smooth settings of DP-SCO share the same optimal rate of excess risk for $1<p<2$. The benefit of smoothness mainly lies in the complexity of gradient query. Smoothness enables us to use variance reduction to achieve linear gradient query time, while the complexity of  \cite{asi2021private} is supper-linear. 

\paragraph{Case of $2<p \leq \infty$.} The analysis above can be generalized to the case of $2<p\leq \infty$. 
We achieve a  $\tilde{O}\big({d^{1/2-1/p}}{\sqrt{T}}+ {d^{1-1/p}}/{\varepsilon }\big)$ regret bound and a convergence rate of $\tilde{O}\big(\frac{d^{1/2-1/p}}{\sqrt{t}}+ \frac{d^{1-1/p}}{t\varepsilon }\big)$, which matches the non-private lower bound $\Omega\big(\frac{d^{1/2-1/p}}{\sqrt{n}}\big)$ in non-private SCO and is thus optimal when $d=\tilde{O}(n\varepsilon^2)$. Previously, \citet{bassily2021non} achieve the same convergence rate by reducing their $2<p\leq\infty$ case to $p=2$ by bounding the $\ell_2$ diameter and Lipschitz constant for the $\ell_p$-setup.

\paragraph{Case of $p = 1$.} 
The challenge of this case is that the tree-based mechanism is no longer applicable to achieve a logarithmic dependence on $d$ because the tree-based method will lead to an $O(\sqrt{d})$ factor.  To overcome the difficulty, we combine the analysis of adaptive composition and Report Noisy Max mechanism \cite{dwork2014algorithmic} to show that the noise with variance $O(\frac{\log d}{t \varepsilon})$ is enough to protect the privacy. Such a result then leads to $O(\sqrt{\frac{\log d}{t}}+\frac{\log d}{\sqrt{t}\varepsilon})$ convergence rate.
Comparing with the rate-optimal DP-SCO algorithm with excess risk $O\big(\sqrt{\frac{\log d}{n}} + \big(\frac{\log d}{n\varepsilon}\big)^{2/3}\big)$ proposed in \citet{asi2021private}, ours SCO result is sub-optimal. 
The gap is not due to our technique of the variance reduction analysis but the difficulty of the online setting.
The optimal rates in \citet{asi2021private} rely on the privacy amplification via shuffling the dataset.
However, accessing all information at beginning is impossible in the online setting.

\paragraph{Privacy-Preserving Online Decision Making.} 
A salient feature of our algorithm is that we provide $\tilde{O}(1/t)$ convergence guarantee for \emph{each time step} while previous works (e.g., \cite{asi2021private,feldman2020private}) can only hold after observing $\tilde{\Omega}(n)$ samples.
Such a convergence result is not of purely intellectual interest.
Instead, it is one of the foundations for extending our algorithm to the online decision-making setting.
Despite the adaptivity of our algorithm to the streaming nature, it is highly non-trivial to extend the SCO guarantee to the online decision setting. 
The recursive gradient variance reduction method needs the stationary distribution assumption of coming data $x_t$. 
In contrast, the distribution of collected sample $x_t$ depends on the decision before and at time $t$, and thus our previous SCO results would fail in this non-stationary setting.
By carefully analyzing the structure of bandit problems, we establish a novel variance reduction guarantee that involves a total-variation term to describe the non-stationarity. 
Then we show that under suitable assumptions, such total-variation term decays at a favorable rate to ensure the desired estimation error guarantee.

While our results can be generalized easily in the case of $1<p\leq \infty$ and various reward structures, we consider the high-dimensional (where $p=1$) online decision-making problem with generalized linear reward \cite{bastani2020online}, which has received lots of recent attention, to illustrate the virtue of our method.
While several remarkable progress has been made on the low-dimensional setting with DP guarantee, \cite{chen2020privacy,shariff2018differentially}, no existing work provides sub-linear regret bound in the high-dimensional setting with DP protection even for linear rewards.
Instead, we provide the first logarithmic regret bound ( Theorems~\ref{thm:regret}) based on our \emph{private online Frank-Wolfe} based bandit algorithms.

This paper is a journal extension of \cite{han2022dpsteaming} that reports the main theoretical results above. Our main extensions in this version are as follows. 
\begin{enumerate}
    \item Complete proofs of all the theoretical statements are provided in details, with further discussions on related literature.
    \item Systematic experiments are conducted with different dataset sizes and dimensions to comprehensively demonstrate the empirical superiority of our online Frank-Wolfe algorithm against some popular algorithms in literature, including NoisySFW (Algorithm 3 in \citet{bassily2021non}), LocalMD (Algorithm 6 in \citet{asi2021private}) when $p=1.5$ and NoisySGD (Algorithm 2 in \citet{bassily2020stability}) when $p=\infty$. Additionally, we also compare our high dimensional bandit algorithm with the DP-UCB algorithm in \cite{shariff2018differentially}. 
    \item We provide an algorithm to generate the generalized Gaussian noise based on Lemma 3.2 in \citet{han2022dpsteaming} (Lemma~\ref{lemma-gamma} in this paper), which will be used by NosiySFW and our algorithms in experiment.
\end{enumerate}

Recently we also noticed that a new arXiv preprint \cite{bassily2022linear} widely extended their previous results in \cite{bassily2021non} in the following three aspects. (a) In  $1<p\leq 2$ regime, they combined the binary-tree based variance reduction technique in \cite{asi2021private} with Frank-Wolfe based algorithm  to improve their previous $\tilde{O}(\frac{1}{\sqrt{n}}+ \frac{\sqrt{d}}{\varepsilon n^{3/4}})$ risk bounds and achieve the same optimal excess risk in linear time as ours; (b) In $2<p\leq \infty$ regime, they replace the multi-pass SGD in \cite{bassily2020stability} by phased SGD in \cite{feldman2020private} to achieve the same risk as ours and \cite{bassily2021non} in linear time. They also explore the concentration property of generalized Gaussian distribution via developing similar results as our Lemma~\ref{lemma-gamma} and improved the in-expectation risk bound in \cite{bassily2021non} to high probability bounds.

A crucial difference between our results and theirs lies in that, while our algorithms are adapted to the online setting, the algorithms in  \cite{bassily2022linear} for $1<p\leq 2$ and $2<p\leq \infty$ need the same batch size as Theorem~7 in \cite{asi2021private} and Theorem~3.5 in \cite{feldman2020private}, respectively. Thus their algorithms cannot be applied to online setting with streaming data and the continual release.

\subsection{Other Related Work}
Our paper is most related to the DP-SCO community. 
In addition, there are two streams of literature that are related to ours: online convex optimization with differential privacy and DP bandits. Below we present a review on them.

\textbf{Online Convex Optimization and Privacy Preserving:} 
Online convex optimization (OCO) algorithms \cite{zinkevich2003online}, learning from a stream of data samples and releasing an output upon new data, provide some of the most successful solutions for many machine learning problems, both in terms of the speed of optimization and the ability of generalization \cite{hazan2019introduction}. Similar to the streaming setting, developing OCO algorithms under DP constraint is harder
than the DP guarantee in the offline learning setting since the whole sequence of outputs along the time horizon is required to be protected \cite{jain2012differentially,guha2013nearly,agarwal2017price}.

\citet{jain2012differentially} provide a generic framework to convert proper online convex programming algorithm into a private one while maintaining $\tilde{O}(\sqrt{dT}/\varepsilon)$ regret for Lipshitz-bounded strongly convex functions and $\tilde{O}(\sqrt{d}T^{2/3}/\varepsilon)$ for general Lipshitz convex functions.
\citet{guha2013nearly} propose algorithms with $\tilde{O}(\sqrt{dT} / \varepsilon)$ regret bound for Lipschtiz convex functions. 
In contrast to the DP SCO works, all above bounds paid a price of privacy factor in the leading order term. 
The only existing work with privacy-free regret bounds $\tilde{O}(\sqrt{Td}+1/\varepsilon)$ is given by \citet{agarwal2017price} for linear losses, while their results and arguments cannot be generalized to more general convex losses.
Our results contribute to the literature by showing privacy-for-free bounds $\tilde{O}(\sqrt{T}+1/\varepsilon)$ are also available for general convex functions  under stochastic setting.

We provide a summary about the comparison with them in Table~\ref{tbl:table-regret-comparison} and the derivation is in Section~\ref{sub:online-learning}.

\begin{table*}[t]
	\centering
	  \caption{Regret bounds (defined in \eqref{eq:regret}) of $(\varepsilon, \delta)$-DP algorithms. $\dagger$ denotes bounds in expectation while $\ddagger$ denotes bounds with high probability. And $^*$ denotes bounds without smoothness assumption.}
	  \label{tbl:table-regret-comparison}
	  \resizebox{\textwidth}{!}{
	  \begin{tabular}{ccccc}
	  \toprule
	  Loss &  Type & $\ell_p$ & Theorem  & Regret \\
	  \midrule
	  Linear  & Adversarial  & $p = 2$  & Cor.~3.2 \cite{agarwal2017price} & $O(\sqrt{T}+d/\varepsilon)^\dagger $  \\
	   \midrule
	  \multirow{3}{*}{Convex} & \multirow{2}{*}{Adversarial} &    \multirow{2}{*}{$p=2$} & Thm.~2 \cite{jain2012differentially} & $O(\sqrt{d}T^{2/3}/\varepsilon )^{\dagger *} $ \\
	   & &  & Thm.~11 \cite{guha2013nearly} & $O(\sqrt{dT}/\varepsilon )^{\dagger *}$ \\
	   \cline{2-5}
	      & Stochastic & $1<p\leq 2$  & Thm.~\ref{thm-convergence-lp-general} &  $O(\sqrt{T}+ \sqrt{d}/\varepsilon)^\ddagger $ \\ 
	  \bottomrule
	  \end{tabular}
	  }
 \end{table*}

\textbf{DP-SCO in $\ell_p$ Geometry:} In the case of $p=2$, \citet{bassily2014private} give the first excess population risk of
$\tilde{O}(\frac{d^{1/4}}{\sqrt{n\varepsilon}})$ by adding a strongly convex regularizer to control the gap between excess population risk and empirical risk. \citet{bassily2019private} further show that with min-batch and multi-pass SGD, the optimal rate $\tilde{O}(\sqrt{\frac{1}{n}} + \frac{\sqrt{d}}{n\varepsilon})$ is achievable. And they further relax the smoothness assumption by applying the smoothing technique based on Moreau-Yosida envelope operator. \citet{bassily2021differentially} consider non-smooth DP-SCO with generalized linear losses (GLL). In $p=2$, their algorithm achieves optimal excess risk in $O(n\log n)$ time. In $p=1$, they bypass the lower bound in non-smoothing setting given by \citet{asi2021private} and achieve the optimal risk in non-private case when $\varepsilon=\Theta(1)$.
\citet{wang2020differentially} consider the heavy-tailed data where the Lipschitz condition of the loss function no longer holds and the the gradient can be unbounded. They achieved excess population risk of $\tilde{O}(\frac{d}{n^{1/3}\varepsilon^{2/3}})$ given that each coordinate of the gradient has bounded second-order moment. \citet{hu2021high} further extend their results to high dimensional space. And \citet{kamath2021improved} improve the rates in \cite{wang2020differentially} and extends to their results are applicable to bounded moment conditions of all orders.

\textbf{DP-Bandits:} Designing bandits algorithm with DP guarantee is an emerging topic in the recent years and we only mention the work which utilize the side-information (context). 
\citet{shariff2018differentially} propose the notion of joint differential privacy (JDP) under which bandits algorithm can achieve nontrivial regret and then they design a scheme to convert the classic linear-UCB algorithm into a joint differential private counterpart to match the non-private regret bound. \citet{dubey2020differentially} extend \citet{shariff2018differentially} algorithms to the federated learning setting.
\citet{Chen2020} tackle private dynamic pricing problem under generalized linear demand model by combining the tree-based mechanism, differentially private empirical risk minimization and UCB algorithm and obtain both excellent DP and performance guarantee for oblivious adversarial and stochastic settings.
\citet{chen2021differential} develop two algorithms which make pricing decisions and learn the unknown non-parametric demand on the fly, while satisfying the DP and LDP gurantees respectively.

\section{Preliminaries}\paragraph{Notations.}
Let $(\textbf{E}, \lVert\cdot\rVert)$ be a normed space of dimension $d$, and $\mathcal{C}\subseteq \textbf{E}$ is a compact convex set of diameter $\diam$. Let $\langle \cdot \rangle$ be an arbitrary inner product over $\textbf{E}$ (not necessarily inducing the norm $\lVert\cdot\rVert$). The dual norm over $\textbf{E}$ is defined as $\| y\|_* \coloneqq \max_{ \|x\|\leq 1} \langle x, y \rangle$. With this definition, $(\textbf{E}, \Vert\cdot \rVert_*)$ is also a $d$-dimensional normed space. 
We use $[K]$ to denote $\{1,2,\cdots, K\}$ and for any $Z\in \R^{d}$ we denote $Z_{1:t}=\{Z_1,Z_2,\cdots, Z_t\}$.
We denote $\bm{0}$ as an all-zero matrix whose size is adjusted according to the context.  
We adopt the standard asymptotic notations. For two non-negative sequences $\{a_n\}$ and $\{b_n\}$, we denote $\{a_n\}=O(\{b_n\})$ or $\{a_n\}\lesssim\{b_n\}$ iff $\lim \sup_{n\rightarrow \infty}a_n/b_n<\infty$, $a_n = \Omega (b_n)$ iff $b_n = O(a_n)$, and $a_n = \Theta(b_n)$ iff $a_n = O(b_n)$ and $b_n = O(a_n)$. We also use $\tilde{O}(\cdot)$, $\tilde{\Omega}(\cdot)$ and $\tilde{\Theta}(\cdot)$ to denote the respective meanings within multiplicative logarithmic factors in $n$ and $\delta$.
    
\subsection{SCO with Streaming Data}
We formally introduce the excess risk below.
Given a parameter set $\mathcal{C}\subset \R^d$, and an unknown distribution $P$ over $\mathcal{X}$ and a function $f: \mathcal{C} \times \mathcal{X} \to \R$, we consider the following optimization problem,
\begin{align*}
    \min_{\theta\in\mathcal{C}} F_P(\theta): = \E_{x\sim P}[ f(\theta,x)],
\end{align*}
and 
\begin{align*}
    \theta^{*} \coloneqq \operatorname{argmin}_{\theta\in \mathcal{C}} F_P(\theta).
\end{align*}
We assume the population loss function is a convex function, i.e., $$F_P(\theta)\ge F_P(\theta^{\prime}) + \langle \nabla F_P(\theta^{\prime}), \theta- \theta^{\prime}\rangle, \quad \forall \theta, \theta^{\prime}\in \mathcal{C}.$$
We will abbreviate $F_\dist$ as $F$ when the context is clear for simplicity.
In practice, the population loss $F(\cdot)$ is unknown and one can only access it via empirical approximation from a set of i.i.d.\ samples $\{x_i\}_{i=1}^n$. 
In the literature, the study of such SCO problems focuses on designing efficient algorithms to find a parameter $\theta$ over the samples $\{x_i\}_{i=1}^n$ such that the excess population risk is acceptable.

In this work, we consider SCO under streaming and continual release setting.
In each time step $t$, one sample $x_t\sim P$ arrives, and our algorithm needs to output a parameter $\theta_t$ with convergence guarantee regarding $F$.
We list the following standard assumptions under a general norm $\lVert \cdot \rVert$ and its dual $\lVert \cdot \rVert_*$ for future reference.

\begin{assumption}[Strongly-convex]
\label{assumption-strongly-convex}
For any $\theta_1,\theta_2\in\mathcal{C}$, the population loss $F$ is said to be $\mu$-strongly convex if $F(\theta_1)\geq  F(\theta_2)+ \langle \nabla F(\theta_2), \theta_1-\theta_2\rangle +\dfrac{\strongly}{2} \lVert \theta_1-\theta_2\rVert^2$ for some $\mu\geq0$.
\end{assumption}

\begin{assumption}[Smoothness]
\label{assumption-smooth}
For any $\theta_1, \theta_2 \in\mathcal{C}$ and $x\in\mathcal{X}$, the loss function $f$ is saied to be $\smooth$-smooth if $\|\nabla f(\theta_1, x) - \nabla f(\theta_2, x)\|_* \leq \smooth \|\theta_1 - \theta_2\|$.
\end{assumption}

\begin{assumption}
\label{assumption-grad-bound}
For any $\theta\in\mathcal{C}$ and $x\in\mathcal{X}$, the loss function $f$ satisfies: $\|\nabla f(\theta, x) - \nabla F(\theta)\|_* \leq \gradvar$.
\end{assumption}

\begin{assumption}[Lipschitz]
\label{assumption-lip}
For any $\theta\in\mathcal{C}$ and $x\in\mathcal{X}$, the loss function $f$ satisfies: $\|\nabla f(\theta, x)\|_* \leq \lip$.
\end{assumption}

\subsection{Differential Privacy}

Our work also extends to the privacy-preserving setting, where the sequence $(\theta_1,\dots,\theta_n)$ satisfies the differential privacy constraint (see Definition~\ref{def-dp}) with respect to the data.
Here we recall the definition of $(\varepsilon,\delta)$-differential privacy.

\begin{definition}[Differential Privacy \cite{dwork2014algorithmic},  $(\varepsilon,\delta)$-DP]\label{def-dp}
A randomized algorithm $\mathcal{A}$ is said to be $(\varepsilon,\delta)$ differentially private if for any pair of datasets $\dataset$ and $\dataset'$ differing in one entry and any event $\mathcal{E}$ in the range of $\mathcal{A}$ it holds that $\mathbb{P}[\mathcal{A}(\dataset) \in \mathcal{E}] \leq e^{\varepsilon} \mathbb{P}[\mathcal{A}(\dataset')\in \mathcal{E}] +\delta$.
\end{definition}

To design the DP-SCO algorithm under $\ell_p$ norm with $1<p\leq 2$, we recall the generalized Gaussian mechanism proposed in \cite{bassily2021non} that leverages the regularity of the dual normed space.

\begin{definition}[Regular Normed Space]\label{lemma-kappa-regularity}
For a  normed space  $(\mathbf{E},\lVert \cdot \rVert)$, we say that the norm $\lVert \cdot \rVert$ is $\kappa$-regular  associated with $\lVert \cdot \rVert_+$, if there exists $1\leq \kappa_+\leq \kappa$ so that $\lVert\cdot \rVert_+$ is $\kappa_+$-smooth and $\lVert\cdot\rVert$ and $\lVert\cdot\rVert_+$ are equivalent with constant $\sqrt{\kappa/\kappa_+}$: 
\begin{align*}
    \lVert  x \rVert^2 \leq  \lVert x\rVert_{+}^2\leq \dfrac{\kappa}{\kappa_{+}}\lVert  x \rVert^2,\quad \forall x \in \mathbf{E}.
\end{align*}
$\ell_q$ norm for $q\geq 1$ is a important class of regular norms, we specify the regularity constant $\kappa_{q}$ and the associated smooth norm $\lVert \cdot \rVert_{q,+}$ later in Lemma~\ref{lem-regularity norm-q>2} and Lemma~\ref{lem-regularity-q<2}. 
\end{definition}

\begin{lemma}[Generalized Gaussian Distribution and Mechanism \cite{bassily2021non}]\label{lemma-general-gaussian}
Given a $\kappa$-regular norm $\lVert \cdot \rVert$ associated with smooth norm $\lVert \cdot \rVert_+$ in $d$-dimensional space, and the generalized Gaussian distribution $\mathcal{G}_{\lVert\cdot \rVert_+ }(\mu,\sigma^2)$ with density:
\begin{align*}
    g(z;\sigma)  =C(\sigma, d) \exp( - \lVert z-\mu \rVert_+^2/ [2\sigma^2] ),
\end{align*}
where $C(\sigma, d) =\big(\text{Area}\{\|x\|_+=1\}\frac{(2\sigma^2)^{d/2}}{2} \Gamma(d/2) \big)^{-1}$, and $\text{Area}$ is the $(d-1)$-dimension surface measure on $\mathbb{R}^d$, then for any function $f$ with $\lVert \cdot \rVert$ sensitivity $s>0$, we have that the mechanism output: \begin{align*}
    f+\mathcal{G}_{\lVert \cdot \rVert_+}(0,2\kappa \log(1/\delta)s^2/\varepsilon^2),
\end{align*}
is ($\varepsilon,\delta$)-differentially private.
\end{lemma}

\section{DP Online Frank-Wolfe Algorithms }\label{sec:dp-sco}
In this section, we present the DP online Frank-Wolfe algorithm framework in solving the $\ell_p$ DP-SCO problem as well as the corresponding excess risk and the regret bounds.

\subsection{$\ell_p$-setup for $1<p\leq \infty$}\label{sec-lp-sco}

In this section, we provide a unified design and analysis 
for optimization in $\ell_p$ geometry with  $1<p\leq \infty$. As a consequence of the H\"{o}lder's inequality, the dual of $\ell_p$ norm is $\ell_q$ norm, where $q$ satisfies $\frac{1}{p}+ \frac{1}{q}=1$, i.e., $q\coloneqq \frac{p}{p-1}$.

\begin{algorithm}[ht]
 \caption{Private Tree-Based Online Frank-Wolfe (DP-TOFW).}
 \label{alg-dp-sco-lp}
\begin{algorithmic}[1]
    \STATE {\bfseries Input:} privacy parameters $(\varepsilon, \delta)$,  $\{\rho_t \}_{t=1}^{n} = \{\eta_t\}_{t=1}^{n}=\frac{1}{1+t}$, $p$ considered in $\ell_p$, and its dual norm $\lVert\cdot \rVert_q$ associated with regular norm $\lVert \cdot \rVert_{q,+}$, initial point $\theta_0=\theta_1=0 \in \mathcal{C}$
  \FOR {$t=1$ {\bfseries to} $n$}
  \STATE Compute and pass $g_t$ in Eq. \eqref{eq-dt-summation} and $\sigma_+(q,\varepsilon,\delta)$ according to Theorem~\ref{thm-privacy-lp} into the tree-based mechanism (Algorithm~\ref{alg-tree-based}).
  \STATE Get noisy summation $\tilde{G}_t =\text{noisy}( \sum_{i=1}^t g_i)$ from the tree-based mechanism (Algorithm~\ref{alg-tree-based}).
  \STATE Set $d_t = \frac{1}{t+1} \cdot \tilde{G}_t$
  \STATE $v_t = \arg \min_{v\in \mathcal{C}} \langle d_t, v \rangle$.
  \STATE $\theta_{t+1} \leftarrow \theta_t + \eta_t (v_t - \theta_t)$.
  \ENDFOR
\end{algorithmic}
\end{algorithm}

Our proposed algorithm is shown in Algorithm~\ref{alg-dp-sco-lp}. At iteration $t$, we consider the following recursive gradient estimator $d_t$ \cite{xie2020efficient}  as an unbiased estimator of the population gradient $\nabla F(\theta_t)$: \begin{align*}
    d_{t} &= \nabla f(\theta_t,x_t)+(1-\rho_t)( d_{t-1} - \nabla f(\theta_{t-1},x_{t})),
\end{align*}
where $d_1 = \nabla f(\theta_1;x_1)$ and $\rho_t = \frac{1}{1+t}$.

A similar recursive gradient scheme is also used in \citet{bassily2021non} for $1<p\leq 2$. 
However, their algorithms use additive noise to ensure the privacy of $d_t$ at each iteration, which accumulate linearly in $t$.
To alleviate the influence of the noise induced by DP, they initialize $d_1$ with the first $\frac{n}{2}$ samples and begin to take mini-batch updates with batch size $\frac{\sqrt{n}}{2}$ for $\sqrt{n}$ iterations, which helps control the sensitivity of $d_t$ and maintain a lower number of noise accumulations. 
However, this strategy leads to a gradient estimation error of $O(\frac{1}{n^{1/2}} + \frac{\sqrt{d}}{\varepsilon n^{3/4}})$. And it also fails in the streaming setting where only one sample is available in initialization.

To improve the error rate and fit the streaming setting, our key observation is that the recursive gradient estimation $d_t$ can be represented as the following summation of empirical gradients,
\begin{align}\label{eq-dt-summation}
    d_t = \dfrac{1}{t+1} \sum_{i = 1}^t\underbrace{\big((i+1)\nabla f(\theta_i;x_i) - i\nabla f(\theta_{i-1};x_i)  \big)}_{g_i}.
\end{align}
Now we reduce the problem of privately releasing $d_t$ in every step $t$ to the problem of privately releasing the incremental summation of $g_i$ in Eq. \eqref{eq-dt-summation}, which motivates us to apply the tree-based mechanism in \citet{guha2013nearly}. In the tree-based mechanism, the leave nodes store the vectors $g_i$. Each internal node stores a private version of the summation of all the leaves in its sub-tree. In this case, any partial summation over $g_i$ can be represented by at most $\lceil \log_2 n\rceil$ nodes. This critical property ensures that the DP noise on $d_t$ would not accumulate linearly in $t$. In this case, our algorithm fits in the streaming setting, where a relatively large number of iterations is required.

One difficulty of applying the tree-based mechanism is the sensitivity analysis. 
Suppose without loss of generality that for adjacent datasets $\mathcal{D}\sim \mathcal{D}'$, we have $x_1\neq x_1'$. Such difference will affect the whole trajectory of the parameters: $\theta_i\neq \theta_i',\forall i\geq 2.$ In other words, the sensitivity will be very large. Fortunately, we can show that such sensitivity can be dramatically reduced by the adaptive analysis similar to \citet{guha2013nearly}. It turns out that noise with variance $\tilde{O}(\frac{\kappa_q^2}{t^2\varepsilon^2})$ is enough to maintain $(\varepsilon,\delta)$-differential privacy guarantee when reporting the $t$-th recursive gradient over the whole time horizon.

With the tree-based mechanism and the adaptive analysis mentioned above, we achieve a gradient error rate of $\tilde{O}(\sqrt{\frac{\kappa_1}{{n}}} + \frac{\sqrt{d \kappa_q }}{\varepsilon t})$ (see Proposition~\ref{prop-lp-variance-reduction}). Furthermore, to report private incremental summation $\sum_{i=1}^t g_i$ for all $t\in[n]$, the amount of space required by the tree-based mechanism is $O(\log_2 n)$. Detailed description can be found in Algorithm~\ref{alg-tree-based} in the Appendix.

In the following theorem, we characterize the privacy guarantee of Algorithm \ref{alg-dp-sco-lp}. The proof can be found in Section~\ref{proof-privacy-lp}.
\begin{theorem}[Privacy Guarantee]
\label{thm-privacy-lp} 
Algorithm ~\ref{alg-dp-sco-lp} is $(\varepsilon,\delta)$-differentially private when $\sigma_+^2$ is selected to be 
\begin{align}\label{eq-noise-variance}
\sigma_+^2 =\frac{8(\lceil\log_2 n\rceil +1)^2 \kappa_q \log((\lceil \log_2 n\rceil + 1)/\delta)(\smooth \diam +\lip)^2}{\varepsilon^2}.
\end{align}
\end{theorem}

Existing results only concern the excess population risk in expectation \cite{bassily2021non}, thus the moment information of generalized Gaussian mechanism is enough for their derivation.
While in our high-probability analysis, the tail behaviour of generalized Gaussian mechanism is characterized. \begin{lemma}[Gamma Distribution]\label{lemma-gamma}
Assume that $Z\sim \mathcal{G}_{\lVert\cdot\rVert_+}(0,\sigma_+^2)$ in $d$-dimensional space, then $\lVert Z \rVert_+^2$ follows Gamma distribution $\Gamma(d/2,2\sigma_+)$. Furthermore, $\lVert Z\rVert_+^2-\mathbb{E}[\lVert Z\rVert_+^2]$ follows $ \text{sub-Gamma}(2\sigma_+^4 d, 2\sigma_+^2)$, which implies that for any $\lambda > 0$, we have
\begin{align*}
    \mathbb{P}(\lVert Z\rVert_{+}^2 > \E[\lVert Z\rVert_{+}^2] + 2\sqrt{\sigma_+^4 d\lambda  }+    2\sigma_+^2\lambda   ) \leq  \exp(-\lambda). 
\end{align*}
\end{lemma}

As a result, we have the following high-probability variance reduction guarantee for the recursive gradient estimator. The proof can be found in Section~\ref{proof-lemma-gamma}.

\begin{proposition}\label{prop-lp-variance-reduction}
Under Assumption~\ref{assumption-smooth} and \ref{assumption-grad-bound}, with probability at least $1-\alpha$, for $t\in[n]$, Algorithm~\ref{alg-dp-sco-lp} satisfies:
\begin{align*}
   \lVert d_t - \nabla F(\theta_t)\rVert_{q} \lesssim & \dfrac{(\sqrt{ \kappa_q}+\sqrt{\log(1/\alpha) })(\smooth \diam + \gradvar  ) }{\sqrt{t+1}} + \frac{\log n \cdot\sigma_+\sqrt{d \log(\log n/\alpha)}}{t+1}.
  \end{align*}
\end{proposition}
\begin{remark}
Noticing that $\sigma_+$ is in scaling of $\tilde{O} (\frac{1}{\varepsilon})$, thus our gradient error for $1<p\leq 2$ is in scaling of $\tilde{O}( \frac{1}{\sqrt{t}} + \frac{\sqrt{d}}{t\varepsilon})$, which improves over the $O(\frac{1}{\sqrt{n}} + \frac{\sqrt{d}}{\varepsilon n^{3/4}})$ in-expectation one in \citet{bassily2021differentially} under the same condition.
\end{remark}
Now we have the following convergence guarantee.
\begin{theorem}[Convergence Guarantee for General Convexity]
\label{thm-convergence-lp-general} 
Consider Algorithm~\ref{alg-dp-sco-lp} with convex function $F$ and assumptions~\ref{assumption-smooth} to \ref{assumption-lip}, for $t\in[n]$, we have with probability at least $1-\alpha$, 
\begin{align*}
    &F(\theta_t)-F(\theta^*) \lesssim \dfrac{\diam(\smooth\diam + \gradvar)\big( \sqrt{\kappa_q} +  \sqrt{\log(n/\alpha)}\big)}{{\sqrt{t}}}+ \dfrac{\log t\big( \beta D^2 + {\diam\sigma_+}\sqrt{d\log(\log n/\alpha)} \log n \big)}{t}.
\end{align*}
\end{theorem}
\begin{remark}
	Later, we will show that the result of Theorem~\ref{thm-convergence-lp-general} is nearly tight for $1<p\leq 2$ and matches the best existing convergence rate for $p>2.$
\end{remark}

One known drawback of Frank-Wolfe is that its convergence rate is slow when the solution lies at the boundary, and it cannot be improved in general even the objection function is strongly convex \cite{lacoste2015global, garber2015faster}. In this case, additional assumption is necessary to improve the convergence rate of Frank-Wolfe in the strongly convex setting.  In the following, we introduce a geometric assumption, which is typical for Frank-Wolfe in the strongly convex setting, even for the non-private case \cite{guelat1986some, lafond2015online}. Denoted by $\partial \mathcal{C}$ the boundary set of $\mathcal{C}$.

\begin{assumption}\cite{lafond2015online}\label{assumption-boundary}
There is a minimizer $\theta^*$ of $F$ that lies in the interior of $\mathcal{C}$, i.e., $\boundary \coloneqq \inf_{v\in\partial \mathcal{C}}\|v-\theta^*\| > 0.$ 
\end{assumption}

\begin{theorem}[Convergence Guarantee for Strong Convexity]
\label{thm-convergence-lp-strongly} Consider Algorithm \ref{alg-dp-sco-lp} with Assumptions~\ref{assumption-strongly-convex}, to \ref{assumption-lip} and \ref{assumption-boundary}, for $1<p\leq \infty$ and $t\in[n]$, we have with probability at least $1-\alpha$,
\begin{equation*}
    \begin{aligned}
        F(\theta_t) - F(\theta^*) &\lesssim \frac{1}{\boundary^2 \strongly} \frac{\diam^2 (\smooth \diam +\gradvar)^2 (\kappa_q + \log(n/\alpha))}{t} + \dots \\
        & \quad + \frac{1}{\boundary^2\strongly} \frac{\big(\smooth ^2\diam^4 + d\diam^2 \sigma_+^2\log(\log n/\alpha)\log^2 n\big)\log n}{t^2}.
    \end{aligned}
\end{equation*}
\end{theorem}

\subsection*{Discussions about $\ell_p$-setup for $1 <p\leq 2$}
When $p\in (1,2]$, we have the following lemma: \begin{lemma}\label{lem-regularity norm-q>2} [Regularity for $q\geq 2$ , \cite{bassily2021non}] When $2\leq q\leq \infty$ , the $\ell_q$ norm is regular with  $$\kappa_q: = \min\{q-1,e^2(\log d - 1)\},\quad  \lVert \cdot\rVert_+ =  \lVert\cdot \rVert_{q^+},\quad q^+ = \min\{q-1, \log d - 1\}. $$
\end{lemma}
\noindent Now noticing that $q = \frac{p}{p-1}\in [2,\infty)$, we bring the $\kappa_q$ claimed in Lemma~\ref{lem-regularity norm-q>2} into Theorem~\ref{thm-convergence-lp-general}, Theorem~\ref{thm-convergence-lp-strongly} and formula \ref{eq-noise-variance} to get the convergence rate of ours algorithm when $1<p\leq 2:$ \\
\textbf{Excess-Risk:}
\begin{align}\label{eq-lp-convex,p<2}
   & \textbf{Convex:} \hspace{2.2cm}  F(\theta_t) - F(\theta^*) \lesssim \sqrt{\dfrac{ {\log (n/\alpha)} }{{t}}}+  \dfrac{\sqrt{d} \log(\log(n)/\delta)}{t\varepsilon }  \\\label{eq-lp-strongly,p<2}
    &\textbf{Strongly Convex:}\quad   F(\theta_t)- F(\theta^*)\lesssim {\dfrac{ {\log (n/\alpha)} }{{t}}}+  \dfrac{ d \log(\log(n)/\delta)}{t^2\varepsilon^2 }
\end{align}
 The bound in equation~\eqref{eq-lp-convex,p<2} is optimal, up to a logarithmic factor, comparing with the $\Omega(\frac{1}{\sqrt{t}} +\frac{\sqrt{d}}{t\varepsilon})$ lower bound shown in \citet{bassily2021non} in the case of $1<p\leq 2$. 

In strongly convex case,  equation~\eqref{eq-lp-strongly,p<2} is tight comparing with the $\Omega(\frac{1}{t} +\frac{d}{t^2\varepsilon^2})$ lower bound shown in \citet{bassily2014private} in the case of $p=2$.  And we conjecture that such bound is also tight for general $1<p\leq 2,$ developing the corresponding lower bound is leaved in the future work.

\subsection*{Discussions about $\ell_p$-setup for $2 <p\leq \infty$.}

When $2<p\leq \infty$, we have $q\in [1,2]$ and the following lemma, \begin{lemma}[Regularity for $1\leq q<2$]\label{lem-regularity-q<2} When $1\leq  q <2$, the $\ell_q$ norm is regular with \begin{align*}
    \kappa_q = d^{1-2/p},\quad \lVert \cdot\rVert_+ =d^{1/2-1/p} \lVert \cdot \rVert_2.
\end{align*}
\end{lemma}
Despite noticing that regularity constant of $\ell_q$ norm has a worse dependence on $d$, we can still get a satisfactory convergence rate by plugging the constants in Lemma~\ref{lem-regularity-q<2} to Theorem~\ref{thm-convergence-lp-general} and Theorem~\ref{thm-convergence-lp-strongly}:\\
\textbf{Excess-Risk:}
\begin{align}\label{eq-lp-convex,p>2}
   & \textbf{Convex:} \hspace{2.2cm}  F(\theta_t) - F(\theta^*) \lesssim d^{1/2-1/p} \sqrt{\dfrac{ {\log (n/\alpha)} }{{t}}}+  \dfrac{d^{1-1/p} \log(\log(n)/\delta)}{t\varepsilon }.  \\\label{eq-lp-strongly,p>2}
    &\textbf{Strongly Convex:}\quad   F(\theta_t)- F(\theta^*)\lesssim d^{1-2/p} {\dfrac{ {\log (n/\alpha)} }{{t}}}+  \dfrac{ d^{2-2/p} \log^2(\log(n)/\delta)}{t^2\varepsilon^2 }.
\end{align}
Comparing with the optimal non-private lower bound $\Omega(\frac{d^{1/2-1/p}}{ \sqrt{n}})$ \cite{agarwal2012information} in convex setting when $2 < p \leq \infty$, our result \eqref{eq-lp-convex,p>2} nearly matches the optimal non-private rate  and is optimal when $d = \tilde{O}(n\varepsilon^2)$.  

The same private-SCO rate is also attained by \citet{bassily2021non} using the  the multi-pass noisy SGD in \citet{bassily2020stability} for $\ell_2$-setup. While the multi-pass SGD has super-linear complexity.

\begin{remark}
one may ask whether there exists other linear-time algorithm achieve the same rate as ours in smooth setting. The answer is `YES': One may replace the multi-pass SGD by the snow-ball SGD\cite{feldman2020private}, which achieves optimal rate under $\ell_2$ setting in linear time: \begin{lemma} Under the same assumption as in Theorem~\ref{thm-convergence-lp-general} when $p>2$ and assume moreover
$\smooth \lesssim \dfrac{n\varepsilon}{d^{1-1/p}}, $ the last iteration output of Algorithm~2 in \cite{feldman2020private} satisfies $$ F(\theta_t) - F(\theta^*) \lesssim d^{1/2-1/p} \sqrt{\dfrac{ {\log (n/\alpha)} }{{t}}}+  \dfrac{d^{1-1/p} \log(\log(n)/\delta)}{t\varepsilon }. $$
\end{lemma}
As stated above, to achieve the same optimal bound when $p>2$, the snow-ball SGD need more that the smoothness constant $\smooth\lesssim \dfrac{n\varepsilon}{d^{1-1/p}}, $ while ours result make no additional assumption on $\smooth$. 
\end{remark}

\subsection{$\ell_p$-setup for $p=1$}\label{sec-l1-sco}

\begin{algorithm}[ht]
 \caption{Private Polyhedral Online Frank-Wolfe (DP-POFW)}
 \label{alg-dp-sco-l1}
\begin{algorithmic}[1]
    \STATE {\bfseries Input:} praivacy parameters $(\varepsilon,\delta)$,  $\{\rho_t \}_{t=1}^{n} =\{\eta_t\}_{t=1}^{n}= \frac{1}{1+t}$, and initial point $\theta_0=\theta_1=0 \in \mathcal{C}$.
   
   \FOR {$t=1$ {\bfseries to} $n$}
   
   \IF{t=1}
   
   \STATE $d_t = \nabla f(\theta_t, x_t)$.
   
   \ELSE
   
   \STATE$d_t = \nabla f(\theta_t, x_t) + (1-\rho_t) (d_{t-1} - \nabla f (\theta_{t-1}, x_t))$.
   
   \ENDIF
   
   \STATE {\small{$\forall v\in \mathcal{C}$, sample $\textbf{n}_v^t \sim \text{Lap}\Big(\frac{4\diam (\smooth \diam +\lip)\sqrt{\log n \cdot \log(1/\delta)}}{\varepsilon\sqrt{t}}\Big)$.}}
   
   \STATE $v_t = \arg \min_{v\in \mathcal{C}} (\langle d_t, v \rangle + \textbf{n}_v^t)$.
   
   \STATE $\theta_{t+1} \leftarrow \theta_t + \eta_t (v_t - \theta_t)$.
   
   \ENDFOR

\end{algorithmic}
\end{algorithm}

In this section, we consider the $\ell_p$-setup for $p=1$. In Algorithm~\ref{alg-dp-sco-l1}, we combine the analysis of the adaptive composition, and the Report Noisy Max mechanism \cite{dwork2014algorithmic} to ensure differential privacy, which reduces the $O(\sqrt{d})$ factor in the excess population risk incurred by the tree-based mechanism in Section~\ref{sec-lp-sco}. In the following, we characterize the privacy guarantee of Algorithm~\ref{alg-dp-sco-l1}. The proof can be found in Section~\ref{proof-privacy-l1}.

\begin{theorem}[Privacy Guarantee]\label{thm-privacy-l1} 
Algorithm \ref{alg-dp-sco-l1} is $(\varepsilon,\delta)$-differentially private.
\end{theorem}

\begin{theorem}[Convergence Guarantee for General Convexity]\label{thm-convergence-l1-general} 
Consider Algorithm \ref{alg-dp-sco-l1} with convex function $F$, Assumptions \ref{assumption-smooth}-\ref{assumption-lip} and \ref{assumption-boundary}, for $t\in[n]$, we have with probability at least $1-\alpha$,
\begin{equation*}
    F(\theta_t) - F(\theta^*)\leq \frac{3}{\sqrt{t+1}}(\smooth \diam^2 +A),
\end{equation*}
where
\begin{equation*}
    \begin{aligned}
        A &= 8\diam(\smooth \diam+\gradvar)\sqrt{\log(8dn/\alpha)}  +\dfrac{16\diam (\smooth \diam+\lip)\log(4dn/\alpha)\sqrt{\log n \cdot \log (1/\delta)} }{\varepsilon}.
    \end{aligned}
\end{equation*}
\end{theorem}
The gradient error in our algorithm (see Lemma~\ref{lemma-fw-gradient-error-l1}) is of the same rate $O(\frac{1}{n})$ as the one in \citet{asi2021private}. Comparing with their excess population risk of $\tilde{O}(\sqrt{\frac{\log d}{n}}+ (\frac{\log d}{n \varepsilon})^{2/3})$, our bound achieves the rate of $\tilde{O}(\sqrt{\frac{\log d}{t}} + \frac{\log d}{\sqrt{t}\varepsilon})$. However, the analysis in \citet{asi2021private} relies on the privacy amplification via shuffling the dataset, which is unacceptable in streaming setting. The proof of the above theorem can be found in Section~\ref{proof-l1-general}.

\begin{theorem}[Convergence Guarantee for Strong Convexity]\label{thm-convergence-l1-strongly} 
Consider Algorithm \ref{alg-dp-sco-l1} with Assumptions~\ref{assumption-strongly-convex}-\ref{assumption-lip} and \ref{assumption-boundary}, for $t\in[n]$, we have with probability at least $1-\alpha$,
\begin{equation*}
    F(\theta_t) - F(\theta^*) \leq \frac{1}{t+1}\bigg(\frac{9(\smooth \diam^2 + A)^2}{\boundary^2\strongly}\bigg),
\end{equation*}
where $A$ is defined in Theorem~\ref{thm-convergence-l1-general}.
\end{theorem}
The above theorem achieves a rate of $\tilde{O}(\frac{\log d}{t}+\frac{\log^2 d}{t\varepsilon})$ comparing with the rate of $\tilde{O}(\frac{\log d}{n} + (\frac{\log d}{n \varepsilon})^{4/3})$ in \citet{asi2021private}, which relies on the privacy amplification via shuffling the dataset as we mentioned in the comment under Theorem~\ref{thm-convergence-l1-general}. The proof of this theorem can be found in Section~\ref{proof-l1-general}.

\begin{remark}
Our results can be generalized to the case that the population loss is strongly convex. 
Although it is appealing to use a folklore reduction from convex setting to strongly convex setting as in \citet{asi2021private} and \citet{feldman2020private} to attain the same $\tilde{O}(\frac{1}{n})$ convergence rate, the reduction relies on the batch splitting.
Specifically, a batch size of the order $O(n/\log n)$ is required.
However, in practice, the ground-truth time horizon $n^*$ can hardly be known in advance.
Thus, one may need to overestimate the time horizon to ensure sufficient privacy protection.
Once the estimated time horizon $n \gtrsim n^*/\log n^*$, the batch-based method will fail, and the last iteration only has the same guarantee as in the convex setting. 
\end{remark}

\subsection{Conversion from Excess Risk to Regret Bounds}
\label{sub:online-learning}
DP online convex optimization considers the learning algorithms design with continual release feature and privacy guarantee and thus is comparable with our algorithms. We formally introduce the online stochastic convex optimization problem: for a given time horizon $T$, at each time one single sample $x_t\sim P$ in $\mathcal{X}$ comes and the player choose a point $\theta_t$ from a set $\mathcal{C}$. 
Then the player observes a random cost/reward $f(x_t, \theta_t)   $ and try to minimize/maximize her population cumulative cost/reward $F_P(\theta_t): = \E_{x\sim P}[f(\theta_t,x)]$ in the whole time horizon. 
The objective of the decision is to minimize the population cumulative regret, which is the absolute difference between the population cost/reward incurred by the algorithm and the possible smallest (highest) cost (reward): 
\begin{align}
\label{eq:regret}
    \text{Reg}(T): =  \sum_{t=1}^T\big[ F_P(\theta_t) - \text{argmin}_{\theta\in \mathcal{C}} F_P(\theta)\big] 
\end{align}
In online SCO, $F_P$ is assumed to be a convex function: \begin{displaymath}
F_P(\theta) \geq F_P(\theta') + \langle \nabla F_P(\theta'), \theta-\theta'\rangle, \quad \forall \theta, \theta'\in \mathcal{C}.
\end{displaymath}
We will denote $\theta^* = \text{argmin}_{\mathcal{C}}F_P(\theta)$ and  abbreviate $F_\dist$ as $F$ when the context is clear for simplicity. 

Sum up by Theorem~\ref{thm-convergence-lp-general}, we can derive the regret 
\begin{corollary}\label{cor-regret-lp-general} Under the same assumption of Theorem~\ref{thm-convergence-lp-general}, we have with probability at least $1-\alpha$, \begin{align*}
    \text{Reg}(T) \lesssim {\sqrt{T}}\cdot\big[{\diam(\smooth\diam + \gradvar)\big( \sqrt{\kappa_q} +  \sqrt{\log(T/\alpha)}\big)}\big]+ (\log T )^2\cdot\big[{\big( \beta D^2 + {\diam\sigma_+}\sqrt{d\log(\log T/\alpha)} \log T \big)}\big].
    \end{align*}
\end{corollary}

\noindent Similarly, by Theorem~\ref{thm-convergence-lp-strongly} we have \begin{corollary}\label{cor-regret-lp-strongly} Under the same assumption of Theorem~\ref{thm-convergence-lp-strongly}, we have with probability at least $1-\alpha$, \begin{align*}
    \text{Reg}(T) \lesssim  \dfrac{\log T}{{\gamma^2 \mu }} \cdot\big[{\diam^2(\smooth\diam + \gradvar)^2\big( {\kappa_q} +  {\log(T/\alpha)}\big)}\big]+ \dfrac{\log T}{\gamma^2 \mu } \cdot\big[{\big( \beta^2 D^4 + {\diam^2\sigma_+^2}d\log(\log T/\alpha) \log^2 T \big)}\big].
\end{align*}
\end{corollary}

For $1<p\le 2$, we bring the $\kappa_q$ claimed in Lemma~\ref{lem-regularity norm-q>2} into Corollary~\ref{cor-regret-lp-general}, Corollary~\ref{cor-regret-lp-strongly} and formula \ref{eq-noise-variance} to get the regert of ours algorithm when $1<p\leq 2:$ \\
\textbf{Regret:}
\begin{align}\label{eq-reg-lp-convex,p<2}
   & \textbf{Convex:} \hspace{2.2cm}  \text{Reg}(T) \lesssim \sqrt{ T{ {\log (n/\alpha)} }}+  \dfrac{\sqrt{d} \log(\log(n)/\delta)}{\varepsilon },  \\\label{eq-reg-lp-strongly,p<2}
    &\textbf{Strongly Convex:}\quad   \text{Reg}(T)\lesssim  { {\log T \log (n/\alpha)} }+  \dfrac{d\log^3  T  \log(\log(T)/\delta)}{\varepsilon^2 }.
\end{align}

For $2<p\le \infty$, we plug the constants in Lemma~\ref{lem-regularity-q<2} to Corollary~\ref{cor-regret-lp-general}, Corollary~\ref{cor-regret-lp-strongly} and formula \ref{eq-noise-variance}\\
\textbf{Regret:}
\begin{align}\label{eq-reg-lp-convex,p>2}
   & \textbf{Convex:} \hspace{2.2cm}  \text{Reg}(T) \lesssim d^{1/2-1/p} \sqrt{ T{ {\log (n/\alpha)} }}+  \dfrac{{d}^{1-1/p} \log(\log(n)/\delta)}{\varepsilon },  \\\label{eq-reg-lp-strongly,p>2}
    &\textbf{Strongly Convex:}\quad   \text{Reg}(T)\lesssim d^{1-2/p}  { {\log T \log (n/\alpha)} }+  \dfrac{d^{2-2/p}\log^3  T  \log(\log(T)/\delta)}{\varepsilon^2 }.
\end{align}

\noindent Finally, we can also derive the regret guarantee when $p = 1$ as in prior sections: \begin{align}
   & \textbf{Convex:} \hspace{2.2cm}  \text{Reg}(T) \lesssim \log d \sqrt{T}/\varepsilon,  \\\label{eq-reg-lp-strongly,p=1}
    &\textbf{Strongly Convex:}\quad   \text{Reg}(T)\lesssim \big(\log d\cdot  \log^2{T}/\varepsilon \big)^2.
\end{align}

In conclusions, our algorithm improves the DP online general convex optimization, i.e., \cite{guha2013nearly}, to a privacy-free rate under the stochastic and smooth setting.

\section{DP High Dimensional Generalized Linear Bandits}\label{sec:bandits}
In this section, we consider the generalized contextual bandits with stochastic contexts, where a decision is made upon each new data \cite{li2017provably}.
Our proposed private Frank-Wolfe algorithm is promising to derive a satisfying estimator for smart decisions under a wide range of reward structures while providing sufficient privacy protection in this setting due to the streaming and continual release feasibility.

However, we face some non-stationarity incurred by the decision process, which leads to a highly non-trivial difficulty when applying the recursive gradient for variance reduction. 
For the fluency of the presentation, we first formulate the contextual bandits model and further explain the difficulty and our novel contributions in-depth.

\subsection{Introduction to Generalized Linear Bandit Problem}

Consider the following generalized linear bandit problem. 
At each time $t$, with individual-specific context $X_t$ sampled from some distribution $\mathcal{P}$ on $\mathcal{X}$, the decision maker can take an action $a_t$ from a finite set (arms) of size $K$ to receive a reward depending on the context $X_{t}$ and the chosen arm $a_t$ through its parameter $\theta_{a_t}^{*}$ via a generalized linear model (GLM): $r_t = \zeta (X_t^{\top} \theta_{a_t}^{*}) + \epsilon_t$, where $\zeta(\cdot)$ is an inverse link function.

We further assume that the context $X_t$, the underlying parameters $\{\theta^{*}_i\}_{i\in [K]}$ and the reward $r_t$ are all bounded.
We assume the noise $\epsilon_t$ is sub-Gaussian \cite{wainwright2019high} and conditional mean zero, i.e., $\mathcal{F}_{t}=\sigma(X_{1:t},r_{1:t-1})$ and $\E[\epsilon_t\lvert \mathcal{F}_{t}]=0$.
We use the standard notion of pseudo regret, i.e., the difference between expected rewards obtained by the algorithm and the best achievable expected rewards, across the time:
\begin{align*}
    \text{Regret}(T) = \sum_{t=1}^T \zeta(X_t^{\top}\theta_{a_t^{*}}^{*} ) - \zeta(X_t^{\top}\theta_{a_t}^{*} ),
\end{align*}
where $a_t^{*}=\arg\max_{i\in [K]}X_t^{\top}\theta_{i}^{*}$.

It is non-trivial to introduce the privacy guarantee in the design of the bandit algorithms. 
The standard notion of DP under continual observation would enforce to select almost the same action for different contexts and incur $\Omega(T)$ regret \cite{shariff2018differentially}.
Here we utilize the more relaxed notion of \emph{Joint Differential Privacy} under continuous observation \cite{shariff2018differentially}.

\begin{definition}[$(\varepsilon, \delta)$-Jointly Differential Privacy (JDP)]
A randomized action policy $\mathcal{A} = (\mathcal{A}_t )_{t=1}^T$ is said to be $(\varepsilon,\delta)$-jointly differentially private under continual observations if for any $t$, any pair of sequences $\dataset$ and $\dataset'$ differing in the $t$ entry and any sequences of action ranges from time $t+1$ to the end $\mathcal{E}_{>t}$, it holds for $\mathcal{A}_{>t}(\dataset): = (A_s(\dataset))_{s>t})$ that $\mathbb{P}[\mathcal{A}_{>t}(\dataset) \in \mathcal{E}_{>t}] \leq e^{\varepsilon} \mathbb{P}[\mathcal{A}_{>t}(\dataset')\in \mathcal{E}_{>t}] +\delta$.
\end{definition}    

We present some standard assumptions in contextual bandits, and similar assumptions can be found in \citet{Goldenshluger2013bandit,bastani2020greedy, bastani2020online}.

\begin{assumption}[Optimal Arm Set] We have a partition $[K] = K_{\text{sup}}\cup K_{\text{opt}}$, so that for every arm $i\in K_{\text{sup}}$, 
	\begin{align*}
	    P(i = \arg\max_{j\in [K]}X^{\top} \theta_j^* ) =0.
	\end{align*}
Moreover, we suppose there exists a $\subopt>0$ such that\begin{enumerate}
	    \item $\max_{i\in [K]} X^{\top}\theta_i^{*}-\subopt > X^{\top}\theta_{j}^{*} \quad  \forall j \in K_{sup}, X\in \mathcal{X}$. 
	    \item For $U_{i}: = \{X\lvert X^{\top}\theta_i^{*}-\subopt> \max_{j\neq i} X^{\top}\theta_j^{*}\}$ we have $P(X\in U_i)>u$ for some $u>0$.
	\end{enumerate} 
\end{assumption}

\begin{assumption}[Eigenvalue]
We assume that $\E[XX^{\top} | X\in U_i]\succeq \lambda I_d $, for some $\lambda>0$, $\forall i\in [K]$. 
\end{assumption}

\begin{assumption}[Margin Condition] 
\label{as:margin-condition}
There exists a constant $\ell$ so that for the sets
\begin{align*}
	\Gamma_i: = \{ \theta: \lVert \theta-\theta_i^*\rVert_1 \leq \ell \}, \forall i \in K_{opt},
\end{align*}  
and given $\theta_i\in \Gamma_i$, $\forall i\in K_{opt}$, we have,
\begin{align*}
	P(X^{\top} \theta_{i_t} - \max_{j\in K_{opt},j\neq i_t} X^{\top} \theta_{j}\leq h )\leq \margin h,
\end{align*}
where $i_t\coloneqq \arg\max_{i\in K_{opt}} X_t^{\top} \theta_i$ for some $\margin>0$.
\end{assumption}

Next we impose the standard regularity assumption on the reverse link function \cite{li2017provably, ren2020batched, chen2020privacy} which includes widely-used linear model and logistic regression.    
\begin{assumption}
There exist $\strongly$ and $\smooth$ such that  
    $0< \strongly \leq \zeta'(z)\leq \smooth$ for any $|z|\leq C$, where $C$ is some given constant.
\end{assumption}

\subsection{Private High Dimensional Bandit Algorithm}

Based on the previous assumptions, we design differentially private high-dimensional GLM bandits (Algorithm~\ref{algo:High-dim-bandit}).
Our algorithm follows the similar procedure of \citet{bastani2020online} to use two sets of estimators: the forced-sampling estimators $\{\theta_{t_0,j}\}_{j\in [K]}$ constructed using i.i.d.\ samples to select a  pre-selected set of arms; and the all-sample estimators $\{\theta_{t,j}\}_{t>t_0, j\in [K]}$ to greedily choose the "best" arm in the pre-selected set.
Another ingredient of our algorithm is the so-called synthetic update, i.e., adding the noisy all-zero contexts and zero rewards to the collected samples for the unselected arm.
This ingredient is similar to \citet{han2021generalized} while they focus on local differential privacy.
\begingroup
\allowdisplaybreaks
\begin{algorithm}[ht]
\caption{Private High Dimensional Bandit (DP-HDB)}
\label{algo:High-dim-bandit}
\begin{algorithmic}[1]
		\STATE {\bfseries Input: }{time horizon $T$, warm up period length $t_0$, privacy parameter $(\varepsilon,\delta)$, initial parameters $\theta_{0,i},i\in [K]$}\\
		\STATE{\bfseries Initialize}{ $\mathcal{I}_i  = \emptyset$ for $i\in [K]$ }
		\FOR{ $i= 0$ {\bfseries to}  $K-1$}
        \FOR{ $t = 1$ {\bfseries to}  $t_0$  }
        \STATE Observe the context $X_{it_0+t}$.
        \STATE Pull arm $i$ and receive $r_{it_0+t}$.
        \STATE Add $(X_{it_0+t},r_{it_0+t})$ to $\mathcal{I}_{i+1}$
        \STATE Update $\theta_{t,i+1}$ via running the $t$-th step of Algorithm~\ref{alg-dp-sco-l1} over $\mathcal{I}_{i+1}$ .
        \ENDFOR		
		\ENDFOR
		\FOR{ $t \geq  Kt_{0} + 1$}
		\STATE Observe the context $X_t$.
		
        \STATE Compute the set of pre-selected arms: {\vspace{-0.8em}\small{
        \begin{align*}
            \hat{K}_t = \{i\in [K]: \zeta( X_t^{\top}\theta_{t_0,i})>\max_{j\in [K]}\zeta( X_t^{\top}\theta_{t_0,j}) - \frac{h_{sub}}{2}\}
            \end{align*}
        }}
        \STATE  \vspace{-1.7em} Compute the greedy action
        \vspace{-0.8em}
        $$ a_t = \arg\max_{a\in \hat{K}_t} \zeta( X_t^{\top}\theta_{t,i})$$	
        \STATE \vspace{-1.8em} Select $a_t$-th arm and receive  $r_t$.
        \STATE Add $(X_t,r_t)$ to $\mathcal{I}_{a_t}$. Add $(\bm 0, \zeta(0))$ to $\mathcal{I}_{i}$ for $i\neq a_t.$ 
        \STATE  Update $\theta_{t,i}$ via running the $t$-th step of Algorithm~\ref{alg-dp-sco-l1} over $\mathcal{I}_i$ for all $i\in [K]$.
        \ENDFOR
		\end{algorithmic}
\end{algorithm}
\endgroup
For our synthetic update, we have the following privacy guarantee and the proof is deferred to Appendix~\ref{pf:private-bandit}.
\begin{theorem}[Privacy Guarantee]
\label{thm:privacy-bandits}
Algorithm~\ref{algo:High-dim-bandit} is $(\varepsilon, \delta)$-JDP.
\end{theorem}

Although it is natural to run Algorithm~\ref{alg-dp-sco-l1} for estimators 
for for arm $i\in [K]$, we are in fact facing various loss functions, say $F_t(\theta_t)\coloneqq \E[\nabla f_t(\theta_{t,i};x_{t,a_t}, y_t)\lvert \mF_{t-1}]$, at each time $t$.
While all of the loss functions share the same minimizers $\theta^{*}_{i}$, $\Delta_t=d_t-\nabla F_t(\theta_{t,i})$ in Algorithm~\ref{alg-dp-sco-l1} is not mean zero and thus the recursive gradient is not an unbiased estimator for the population gradient.
As in the SCO setting, to show that the norm of the gradient estimation error $\Delta_t$ converges to zero sufficiently fast, we reformulate $\Delta_t$ as the sum of a sequence $\{\zeta_{t,\tau}\}_{\tau=1}^{t}$.
Our SCO results enjoy the i.i.d.\ nature of the data and thus $\{\zeta_{t,\tau}\}_{\tau=1}^{t}$ is a martingale difference sequence which can be controlled by an Azuma-Hoeffding-type concentration inequality.
In the bandits setting, after the forced-sampling period, the sample distribution for each arm evolves by time, and thus the sequence is no longer conditional mean zero.
To overcome the difficulty, we develop a novel lemma on bridging the gradient error to the total variance difference of distributions between each time step, which is the key to our success in deriving the nontrivial regret bound in this setting.

\begin{lemma} 
\label{lm:context-gradient-error-total-variation}
For each arm $i\in K_{\text{opt}}$, suppose that the greedy action begins to be picked at $t_0$, then  for any $t>t_0$ we have with probability at least $1-\alpha$,
{
\begingroup
\allowdisplaybreaks
\begin{align*}
&\lVert \Delta_t\rVert_\infty  \lesssim \sqrt{\log((d+ T)/\alpha)}\bigg(  \dfrac{(MD+\beta)}{\sqrt{t}} +\dfrac{\beta\diam \Lx }{t} \Big( \SCerror(\frac{\alpha}{d+t_0}) \sqrt{t_0}  + \margin\sum_{\tau = t_0+1}^t \lVert \theta_{\tau - 1,i}-\theta_{i}^*\rVert_1    \Big)\bigg),
\end{align*}
\endgroup
}
where $\SCerror(\alpha)=O(\log(dT/\alpha))$ is specified in the complete version (Lemma~\ref{lm:estimaton-error-bandit}).
\end{lemma}
Such lemma provides a guideline on tuning the warm-up stage length of the algorithm. In particular, it implies that polylog($T$) length of warm-up is
sufficient to get a $\tilde{O}(\frac{1}{\sqrt{t}})$-decayed gradient estimation error for each arm $i\in K_{\text{opt}}$ if the previous estimators converge to the underlying one at sufficiently fast rates.
Such a low gradient estimation error is sufficient for the fast parameter convergence in the consequent time steps.

As far as we know, this is the first attempt to directly apply variance reduction in a non-stationary environment, which is sharply contrast to the previous solutions.
In reinforcement learning (RL), as pointed out by \cite{papini2018stochastic}, variance reduction can potentially improve much the sample efficiency since the collection of the samples requires the agent to interact with the environment, which could be costly.
However, the sampling trajectories is generated by an RL algorithm. 
Thus the direct usage of the variance reduction also suffer from the changing distribution of the collected sample once their RL algorithm improves based on previous experience.
This also applies to the bandits setting which shares the similar spirit in the data collection process.
In overcome this, previous work \citep{sutton2016emphatic, papini2018stochastic, xu2019sample}, mainly employ importance sampling to correct the distribution shift and construct an unbiased estimator for the policy gradient with respect to the snapshot policy.
However, importance sampling is prone to high variance, e.g., \cite{thomas2015high}. 

We prove the desired convergence rate of the estimation error by induction in Section~\ref{sec:Proof_of_bandit_estimation}, and here we present the corresponding theorem.
\begin{theorem}[Estimation Error]
\label{lm:induction-text}
For the full-sample estimator $\theta_{t,i}$, when $t>t_0=O(\frac{\log(dT/\alpha)\log(T)}{\varepsilon^2})$, for every arm $i\in K_{\text{opt}} $, we have with probability at least $1-\alpha$,
\begin{align*}
    \lambda\strongly u \lVert \theta_{t,i} - \theta^{*}_i \rVert_1^2 \le F_t(\theta_{t,i}) - F_t(\theta^{*}_i)\le \frac{\Cpara(\alpha)}{t},
\end{align*}
for some constant $t_0$ and $\Cpara(\alpha)=O( \frac{\log^2(dT/\alpha)\log(T)}{\varepsilon^2})$ specified in Section~\ref{sec:Proof_of_bandit_estimation}. 
\end{theorem}
Now we are ready to present our regret bound by converting the estimation error to regret, whose formal proof is given in Section~\ref{sec:proof-regret}.

\begin{theorem}[Regret bound]
\label{thm:regret}
    With probability at least $1-\alpha$, Algorithm~\ref{algo:High-dim-bandit} achieves the following regret bound
    \begin{align*}
        \text{Regret}(T)&\le t_0 + \Lx^2 \smooth^2 \Cpara(\alpha/(4\lvert K_{opt}\rvert)) \log(T) + 2\Lx \smooth \sqrt{\Cpara(\alpha/(4\lvert K_{opt}\rvert))\log(T)\log(4/\alpha)} \\
        & = O\left(\dfrac{\log^2(dT/\alpha)\log^2 T}{\varepsilon^2}\right).
    \end{align*}
\end{theorem}

\begin{remark}
This regret has a sublinear growth rate, and it is the first regret bound for DP high-dimensional generalized linear bandits. In particular, 
the upper bound above has only a poly-logarithmic growth concerning dimension $d$, as desired in high dimensional scenarios. Compared with the regret bound $O(\log^2 (dT))$ without DP in \citet{bastani2020online}, our upper bound contains an extra $O(\log^2T)$ factor, which is due to our simplified proof to shed light on the main idea. 
We leave the refinement as future directions. 
\end{remark}

\section{Experiments}\label{sec:experiments}
In this section, we present experiment results to demonstrate the efficacy and efficiency of our algorithm.

\subsection{Generation of Generalized Gaussian Noise}
Firstly we provide an algorithm to generate the generalized Gaussian noise in Lemma~\ref{lemma-general-gaussian}, which will be used by DP-TOFW, DP-POFW and NoisySFW (Algorithm 3 in \cite{bassily2021non}) in the following experiment. When $2<p\leq \infty$(i.e. $1\leq q<2$), the corresponding Generalized Gaussian Noise is a re-scaled standard Gaussian noise under $\ell_2$ norm. We focus on the case $1< p<2$, in which the Generalized Gaussian Noise follows the p.d.f. defined in Lemma~\ref{lemma-general-gaussian} with $\lVert \cdot \rVert_+ = \lVert \cdot \rVert_{q^+},q^+ =   \min\{q-1,\log d - 1\}.$ 

\begin{algorithm}[ht]
\caption{Generation of the $\ell_{q^+}$ Gaussian Noise}
\label{algo:generalized-gaussian}
\begin{algorithmic}[1]
		\STATE {\bfseries Input:} {dimension $d$, $q_{+}$, noise level $\sigma_+$} 
		\STATE {Generate $r^2\sim Gamma(d/2, 2\sigma_+)$}
		\STATE {Generate $d$ independent random real scalars $\epsilon_i\sim \overline{G}(1/q_{+},q_{+})$ where $ \overline{G}(1/q_{+},q_{+})$ is the generalized normal distribution}
		\STATE {Construct the vector $x$ of component $x_i = s_i\cdot \epsilon_i$ and $\{s_i\}_{i\in[d]}$ are independent random signs}
		\STATE {\bfseries Output:} {$y=r\frac{x}{\lVert x\rVert_{\ell_{q^+}}}$}
		\end{algorithmic}
\end{algorithm}

\begin{lemma}\label{lemma-generalied-gaussian-distribution}
The output in Algorithm~\ref{algo:generalized-gaussian} follows the generalized Gaussian distribution $\mathcal{G}_{\lVert \cdot \rVert_{+}}(0, \sigma_+^2)$ with $\lVert \cdot \rVert_+ = \lVert \cdot \rVert_{q^+}$.
\end{lemma}
\begin{proof}
By Lemma~\ref{lemma-gamma}, we know that if $Z\sim \mathcal{G}_{\lVert \cdot \rVert_{+}}(0, \sigma_{+}^2)$ in $d$-dimensional space, then $\lVert Z\rVert_{+}^2$ follows Gamma distribution $\Gamma(d/2, 2\sigma_+)$. 
Since the generalized Gaussian distribution is $\ell_{\kappa_{+}}$ radially symmetric, Lemma 3.2 in \cite{calafiore1998uniform} proves that conditional on $\lVert Z\rVert_+^2=r^2$ for any $r>0$, $Z$ is uniformly distributed on $\ell_{\kappa_{+}}$ spherical with radius $r$.
Thus the remaining part is to sample uniformly from the sphere with radius $\lVert Z\rVert_+$ in $\ell_{\kappa_{+}}$ space, which is achieved by modifying Algorithm 4.1 in \cite{calafiore1998uniform} (step 3-5 in Algorithm~\ref{algo:generalized-gaussian}).
\end{proof}

\subsection{Experimental Setting}
In this section, we consider the linear regression setting,
\begin{equation*}
    y = X^\top \theta + \epsilon,
\end{equation*}
where the design matrix $X\in\mathbb{R}^{d\times n}$, true parameter $\theta\in \mathbb{R}^d$, output $y\in \mathbb{R}^n$, and $\epsilon \sim ~ N(0, \nu^2I_n)$ is a noise vector. We define the loss function as $\mathcal{L}(\hat{\theta}, X) = \frac{1}{n} \sum_{i=1}^n (y_i - \langle x_i, \hat{\theta} \rangle)^2$ for any given estimation $\hat{\theta}$, where $y_i$ is the $i$-th entry of $y$ and $x_i$ is the $i$-th column of $X$. Therefore, the excess risk will be $F(\hat{\theta}) = \mathbb{E} [\mathcal{L}(\hat{\theta})]$ where the expectation is taken with respect to the randomness in $X$ and $\epsilon$. Here we will use the loss function over a separate testing set as an empirical estimation of the excess population risk, which we denote as $\mathcal{L}(\hat{\theta}, X_{\text{test}})$ . And we further introduce suboptimality as $\text{SubOpt}=\frac{\mathcal{L}(\hat\theta, X_{\text{test}}) - \mathcal{L}(\theta, X_{\text{test}})}{\mathcal{L}(\theta_0, X_{\text{test}}) - \mathcal{L}(\theta, X_{\text{test}})}$. Here $\theta_0$ is zero vector, serving as the initialization of all algorithms.
All experiments are finished on a server with 256 AMD EPYC 7H12 64-Core Processor CPUs.
The code to reproduce our experimental results is shared in our \href{https://github.com/liangzp/DP-Streaming-SCO}{Github Repo}.

\subsection{Comparison with DP-SCO Algorithms}
To demonstrate the efficacy and efficiency of our algorithm in $1\le p\le \infty$ regimes, we choose $p=1.5$ and $p=\infty$ as our geometries. We compare our DP-TOFW with NoisySFW (Algorithm 3 in \citet{bassily2021non}), LocalMD (Algorithm 6 in \citet{asi2021private}) when $p=1.5$ and with NoisySGD (Algorithm 2 in \citet{bassily2020stability}) when $p=\infty$. We generate $T$ samples i.i.d. from a normal distribution with mean zero and standard deviation 0.05, and then normalize them by their $q$-norm to ensure each sample maintain unit $q$-norm.
We also generate the true underlying parameter $\theta$
by setting all its entries to be sampled from a normal distribution with mean zero and standard deviation 0.05 and then normalized it by its $p$-norm. The size of the testing set is 10000.

For all the experiment, we set the radius of constrain set $\mathcal{C}$ as 2 and guarantee $(1, 1/T)$-DP. To comprehensively demonstrate the performance of our algorithm, we conduct our experiment with $T=1000, 2000, 5000$ and $10000$ with dimension $d=5, 10$ and $20$. To achieve the best performance for each algorithm, we will scale their default learning rate by a grid of scaling factors. In Figure~\ref{fig-lr-scale}, we show the SubOpt of several algorithms under different learning rate scalings. As we can see, comparing with NoisySFW in $p=1.5$ and NoisySGD in $p=\infty$, DP-TOFW is robust against learning rate scaling. In Table~\ref{tab:exp}, we show the risk, SubOpt and wall-clock time for all algorithms with their best learning rate scaling under different $T$ and $d$ combinations. All the results are based on 10 independent runs with different random seeds. 
As we can see, our proposed significantly outperforms NoisySFW in terms of risk while our DP-TOFW achieves comparable risk with NoisySGD but with much less computational cost.

One thing we need to mention here is that LocalMD does not converge regardless of the learning rate scaling. We suspect that this is due to the large constants before their Bregman divergence, and the standard deviation of their Gaussian noise. In Figure~\ref{fig-time-subopt}, we visualize the SubOpt against wall-clock time of NoisySFW and DP-TOFW with their best learning scaling under $p=1.5$. We notice that NoisySFW converges faster than DP-TOFW because it has a  smaller number of total iteration ($O(\sqrt{n})$) in centralized setting, while DP-TOFW needs to receive the data one by one and triggers the tree mechanism upon each data arrival ($O(n)$ times).

\subsection{Comparison with DP-Bandit Algorithms}
Moreover, we conduct the experiment on our bandits applications.
We compare our DP-HDB with the Linear UCB via Additive Gaussian Noise algorithm (DPUCB) in \cite{shariff2018differentially}. We choose the time horizon $T=10000$, the dimension $d=50$, number of arms $K=2$ and privacy epsilon $\varepsilon=1$.
The other parameters are set as recommended.
The comparison between two bandit algorithms' cumulative regret is demonstrated in Figure~\ref{fig:bandit-exp}.
Our proposed algorithm significantly outperforms DPUCB in the cumulative regret.

\begin{figure}[t]
\centering
\includegraphics[width=8cm]{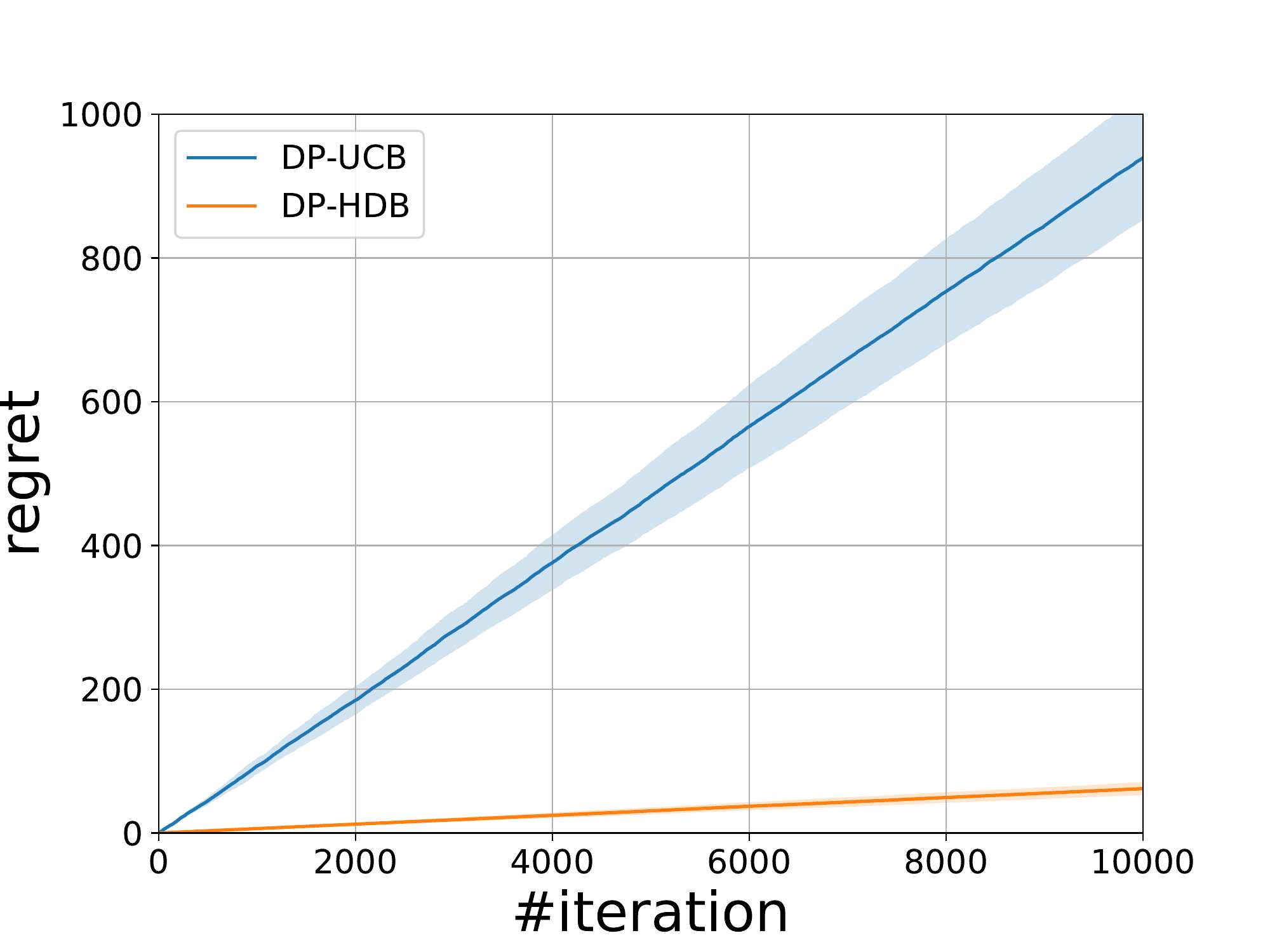}
\caption{Comparison between cumulative regret of DP-UCB and our DP-HDB algorithm}
\label{fig:bandit-exp}
\end{figure}

\begin{table}
\centering
\caption{Experiment Results. We do not distinguish our DP-TOFW (Algorithm~\ref{alg-dp-sco-lp}) and DP-POFW (Algorithm~\ref{alg-dp-sco-l1}) and denote them as OFW in this table as they belong to our unified online Frank-Wolfe framework}
\label{tab:exp}
\resizebox{.8\textwidth}{!}{%
\begin{tabular}{lllllll}
\toprule
      &    &     &     &                  Risk &                 SubOpt &                 Time \\
T & d & p & algo &                       &                        &                      \\
\midrule
\multirow{12}{*}{1000} & \multirow{4}{*}{5} & \multirow{2}{*}{1.5} & NoisySFW &    0.0885$\pm$0.00907 &       0.522$\pm$0.0565 &   0.0189$\pm$0.00216 \\
      &    &     & OFW &   0.00536$\pm$0.00155 &     0.0172$\pm$0.00987 &   0.0929$\pm$0.00157 \\
\cline{3-7}
      &    & \multirow{2}{*}{inf} & NoisySGD &     0.0259$\pm$0.0176 &      0.0802$\pm$0.0613 &       30.2$\pm$0.805 \\
      &    &     & OFW &     0.0357$\pm$0.0216 &       0.112$\pm$0.0727 &   0.0765$\pm$0.00345 \\
\cline{2-7}
\cline{3-7}
      & \multirow{4}{*}{10} & \multirow{2}{*}{1.5} & NoisySFW &     0.0771$\pm$0.0107 &       0.953$\pm$0.0947 &  0.0216$\pm$0.000514 \\
      &    &     & OFW &    0.0183$\pm$0.00332 &       0.201$\pm$0.0483 &     0.105$\pm$0.0124 \\
\cline{3-7}
      &    & \multirow{2}{*}{inf} & NoisySGD &     0.0525$\pm$0.0149 &        0.636$\pm$0.186 &       31.2$\pm$0.516 \\
      &    &     & OFW &     0.0915$\pm$0.0209 &        0.582$\pm$0.157 &    0.0852$\pm$0.0109 \\
\cline{2-7}
\cline{3-7}
      & \multirow{4}{*}{20} & \multirow{2}{*}{1.5} & NoisySFW &   0.0414$\pm$0.000302 &        1.05$\pm$0.0108 &    0.0217$\pm$0.0027 \\
      &    &     & OFW &    0.0307$\pm$0.00344 &        0.775$\pm$0.128 &    0.108$\pm$0.00901 \\
\cline{3-7}
      &    & \multirow{2}{*}{inf} & NoisySGD &    0.0202$\pm$0.00262 &        0.955$\pm$0.111 &       31.4$\pm$0.945 \\
      &    &     & OFW &    0.0766$\pm$0.00464 &       0.982$\pm$0.0768 &    0.0825$\pm$0.0114 \\
\cline{1-7}
\cline{2-7}
\cline{3-7}
\multirow{12}{*}{2000} & \multirow{4}{*}{5} & \multirow{2}{*}{1.5} & NoisySFW &    0.0746$\pm$0.00644 &        0.44$\pm$0.0287 &  0.0376$\pm$0.000492 \\
      &    &     & OFW &  0.00285$\pm$0.000197 &    0.00235$\pm$0.00106 &     0.222$\pm$0.0184 \\
\cline{3-7}
      &    & \multirow{2}{*}{inf} & NoisySGD &    0.0127$\pm$0.00371 &      0.0344$\pm$0.0133 &        119$\pm$0.962 \\
      &    &     & OFW &    0.0152$\pm$0.00585 &        0.0432$\pm$0.02 &    0.156$\pm$0.00851 \\
\cline{2-7}
\cline{3-7}
      & \multirow{4}{*}{10} & \multirow{2}{*}{1.5} & NoisySFW &    0.0701$\pm$0.00493 &       0.887$\pm$0.0691 &   0.0366$\pm$0.00502 \\
      &    &     & OFW &   0.00704$\pm$0.00247 &      0.0595$\pm$0.0321 &    0.189$\pm$0.00384 \\
\cline{3-7}
      &    & \multirow{2}{*}{inf} & NoisySGD &     0.0382$\pm$0.0133 &        0.484$\pm$0.178 &         124$\pm$2.36 \\
      &    &     & OFW &     0.0582$\pm$0.0113 &       0.364$\pm$0.0795 &    0.159$\pm$0.00928 \\
\cline{2-7}
\cline{3-7}
      & \multirow{4}{*}{20} & \multirow{2}{*}{1.5} & NoisySFW &    0.0376$\pm$0.00262 &       0.957$\pm$0.0673 &   0.0397$\pm$0.00427 \\
      &    &     & OFW &     0.018$\pm$0.00374 &        0.406$\pm$0.106 &      0.2$\pm$0.00776 \\
\cline{3-7}
      &    & \multirow{2}{*}{inf} & NoisySGD &    0.0209$\pm$0.00148 &       0.947$\pm$0.0672 &        125$\pm$0.345 \\
      &    &     & OFW &     0.067$\pm$0.00799 &         0.82$\pm$0.114 &     0.164$\pm$0.0098 \\
\cline{1-7}
\cline{2-7}
\cline{3-7}
\multirow{12}{*}{5000} & \multirow{4}{*}{5} & \multirow{2}{*}{1.5} & NoisySFW &     0.0587$\pm$0.0212 &         0.35$\pm$0.135 &    0.0979$\pm$0.0143 \\
      &    &     & OFW &  0.00258$\pm$1.58e-05 &   0.000702$\pm$0.00051 &    0.469$\pm$0.00995 \\
\cline{3-7}
      &    & \multirow{2}{*}{inf} & NoisySGD &  0.00538$\pm$0.000982 &    0.00989$\pm$0.00378 &         742$\pm$3.35 \\
      &    &     & OFW &  0.00667$\pm$0.000392 &     0.0145$\pm$0.00153 &     0.397$\pm$0.0291 \\
\cline{2-7}
\cline{3-7}
      & \multirow{4}{*}{10} & \multirow{2}{*}{1.5} & NoisySFW &    0.0659$\pm$0.00926 &        0.808$\pm$0.115 &   0.0867$\pm$0.00326 \\
      &    &     & OFW &  0.00376$\pm$0.000479 &      0.0163$\pm$0.0053 &     0.477$\pm$0.0189 \\
\cline{3-7}
      &    & \multirow{2}{*}{inf} & NoisySGD &    0.0201$\pm$0.00175 &       0.115$\pm$0.0118 &          747$\pm$2.4 \\
      &    &     & OFW &     0.022$\pm$0.00347 &       0.125$\pm$0.0204 &      0.38$\pm$0.0117 \\
\cline{2-7}
\cline{3-7}
      & \multirow{4}{*}{20} & \multirow{2}{*}{1.5} & NoisySFW &   0.0401$\pm$0.000848 &           1$\pm$0.0264 &    0.0834$\pm$0.0121 \\
      &    &     & OFW &   0.00962$\pm$0.00209 &       0.185$\pm$0.0558 &     0.527$\pm$0.0121 \\
\cline{3-7}
      &    & \multirow{2}{*}{inf} & NoisySGD &     0.049$\pm$0.00585 &       0.602$\pm$0.0571 &         755$\pm$2.89 \\
      &    &     & OFW &    0.0535$\pm$0.00736 &        0.637$\pm$0.105 &      0.413$\pm$0.004 \\
\cline{1-7}
\cline{2-7}
\cline{3-7}
\multirow{11}{*}{10000} & \multirow{3}{*}{5} & \multirow{2}{*}{1.5} & NoisySFW &      0.0607$\pm$0.027 &        0.351$\pm$0.161 &    0.163$\pm$0.00869 \\
      &    &     & OFW &  0.00255$\pm$4.72e-05 &  0.000318$\pm$0.000179 &      1.04$\pm$0.0295 \\
\cline{3-7}
      &    & inf & OFW &  0.00337$\pm$0.000336 &    0.00293$\pm$0.00106 &      0.76$\pm$0.0157 \\
\cline{2-7}
      & \multirow{4}{*}{10} & \multirow{2}{*}{1.5} & NoisySFW &    0.0646$\pm$0.00522 &       0.789$\pm$0.0744 &     0.175$\pm$0.0253 \\
      &    &     & OFW &   0.00282$\pm$0.00025 &    0.00465$\pm$0.00184 &      1.07$\pm$0.0987 \\
\cline{3-7}
      &    & \multirow{2}{*}{inf} & NoisySGD &   0.0103$\pm$0.000748 &     0.0505$\pm$0.00514 &     3.03e+03$\pm$8.8 \\
      &    &     & OFW &   0.00976$\pm$0.00259 &      0.0467$\pm$0.0159 &      0.803$\pm$0.013 \\
\cline{2-7}
\cline{3-7}
      & \multirow{4}{*}{20} & \multirow{2}{*}{1.5} & NoisySFW &    0.0391$\pm$0.00178 &       0.937$\pm$0.0303 &     0.153$\pm$0.0194 \\
      &    &     & OFW &  0.00487$\pm$0.000584 &      0.0592$\pm$0.0155 &       1.04$\pm$0.102 \\
\cline{3-7}
      &    & \multirow{2}{*}{inf} & NoisySGD &    0.0409$\pm$0.00362 &       0.482$\pm$0.0453 &    3.12e+03$\pm$27.6 \\
      &    &     & OFW &    0.0316$\pm$0.00192 &       0.363$\pm$0.0283 &     0.844$\pm$0.0464 \\
\bottomrule
\end{tabular}}
\end{table}

\begin{figure}[htbp]
     \centering
     \subfigure[NoisySFW under $p=1.5$]{\includegraphics[width=0.45\textwidth]{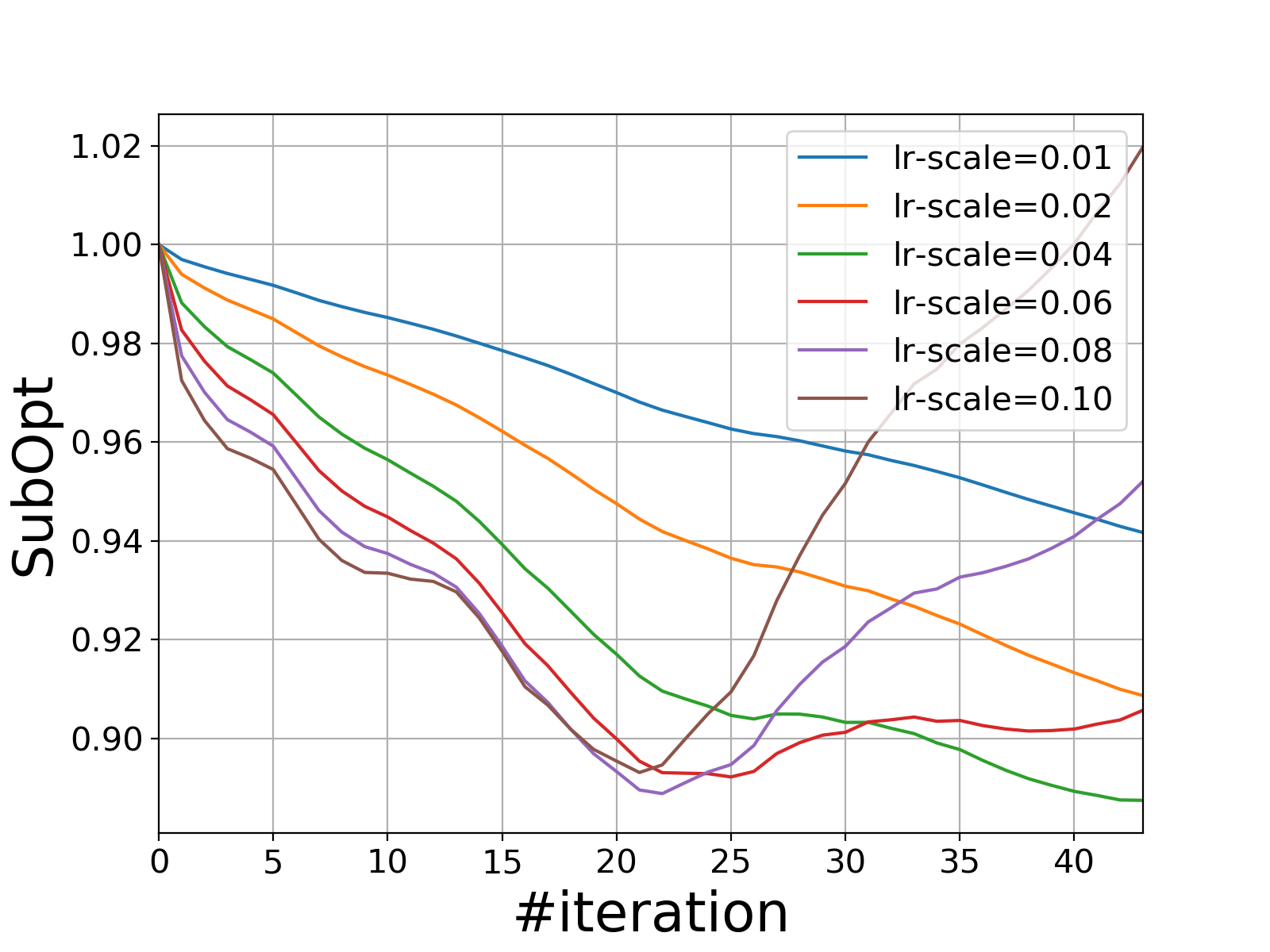}}
     \subfigure[DP-TOFW under $p=1.5$]{\includegraphics[width=0.45\textwidth]{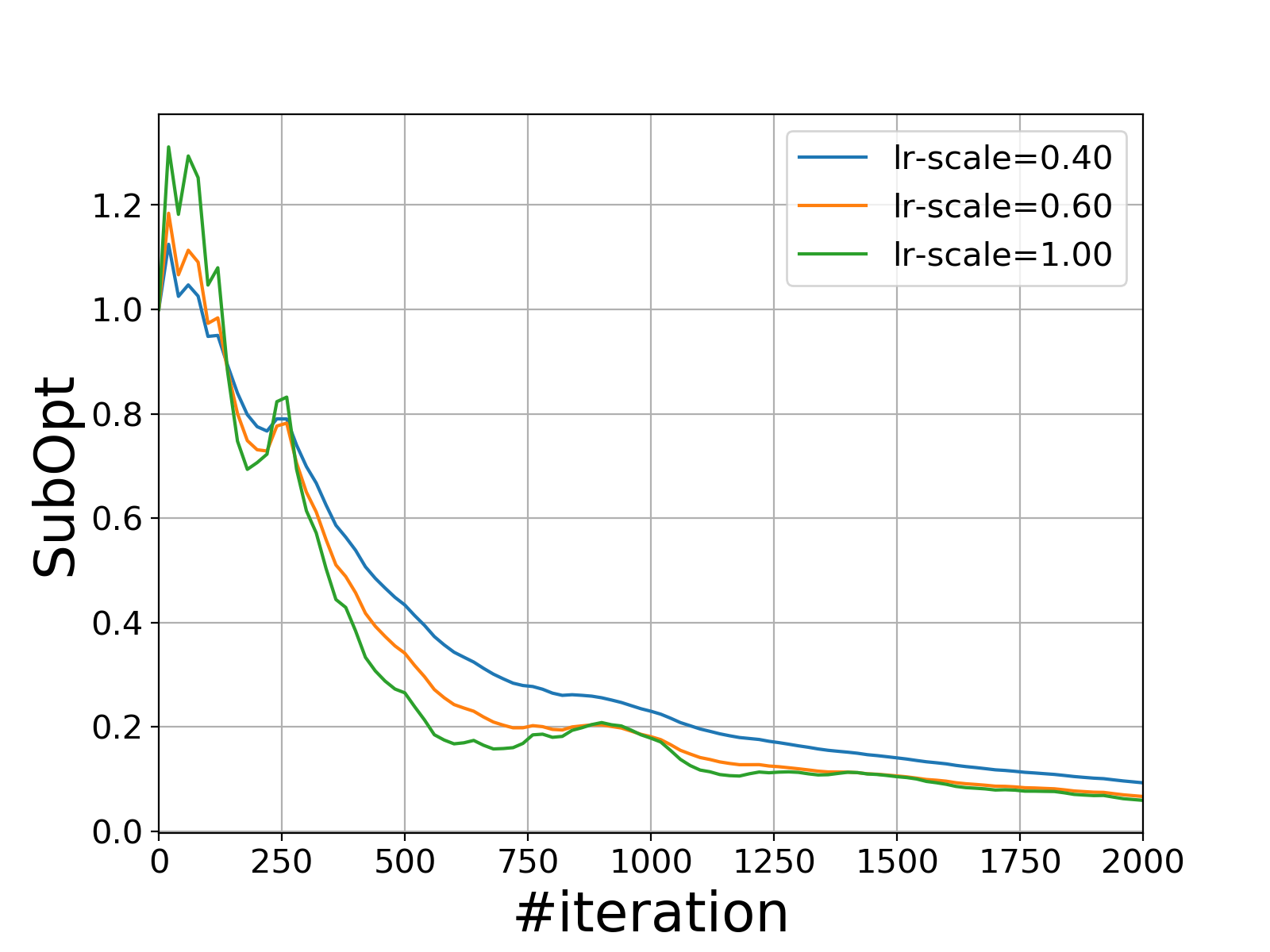}} 
     \subfigure[NoisySGD under $p=\infty$]{\includegraphics[width=0.45\textwidth]{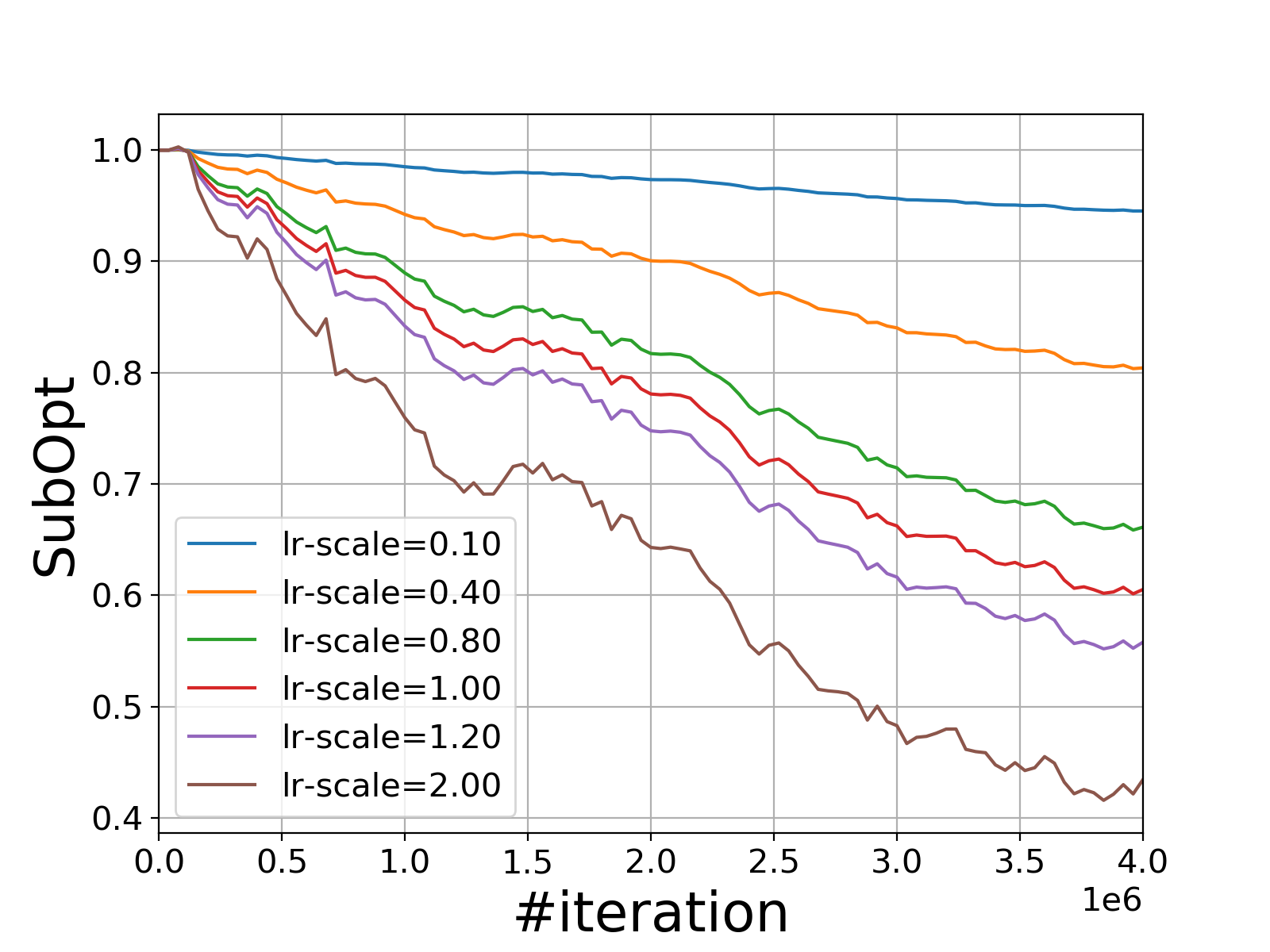}}
     \subfigure[DP-TOFW under $p=\infty$]{\includegraphics[width=0.45\textwidth]{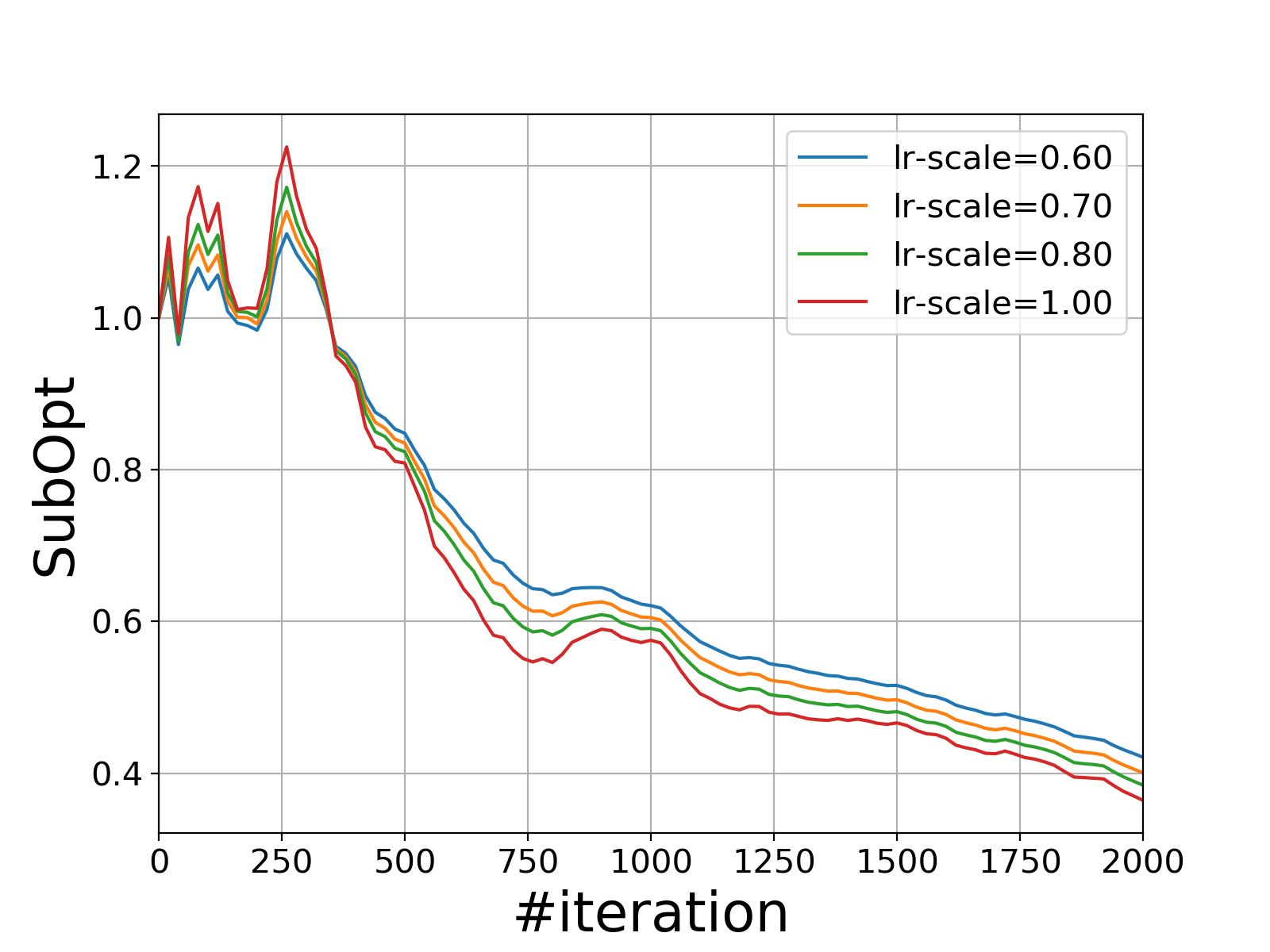}}

     \caption{Training curves of four algorithms with different learning rate scalings. Here we set $T=2000$, $d=10$. And they are all under $(1, 1/T)$-DP.}\label{fig-lr-scale}
\end{figure}

\begin{figure}[htbp]
     \centering
     \subfigure[$T=1000$ and  $d=5$]{\includegraphics[width=0.45\textwidth]{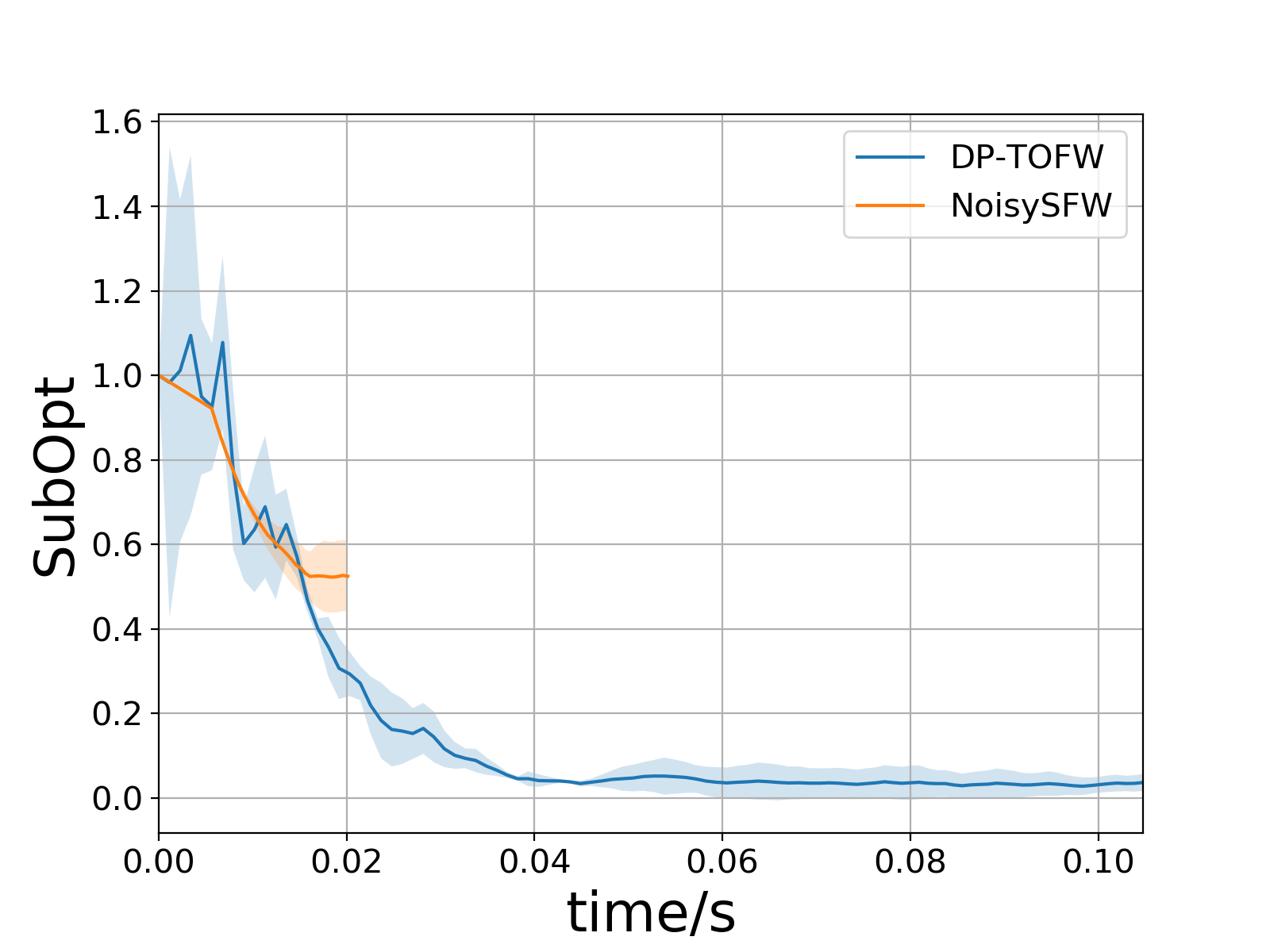}}
     \subfigure[$T=2000$ and $d=10$]{\includegraphics[width=0.45\textwidth]{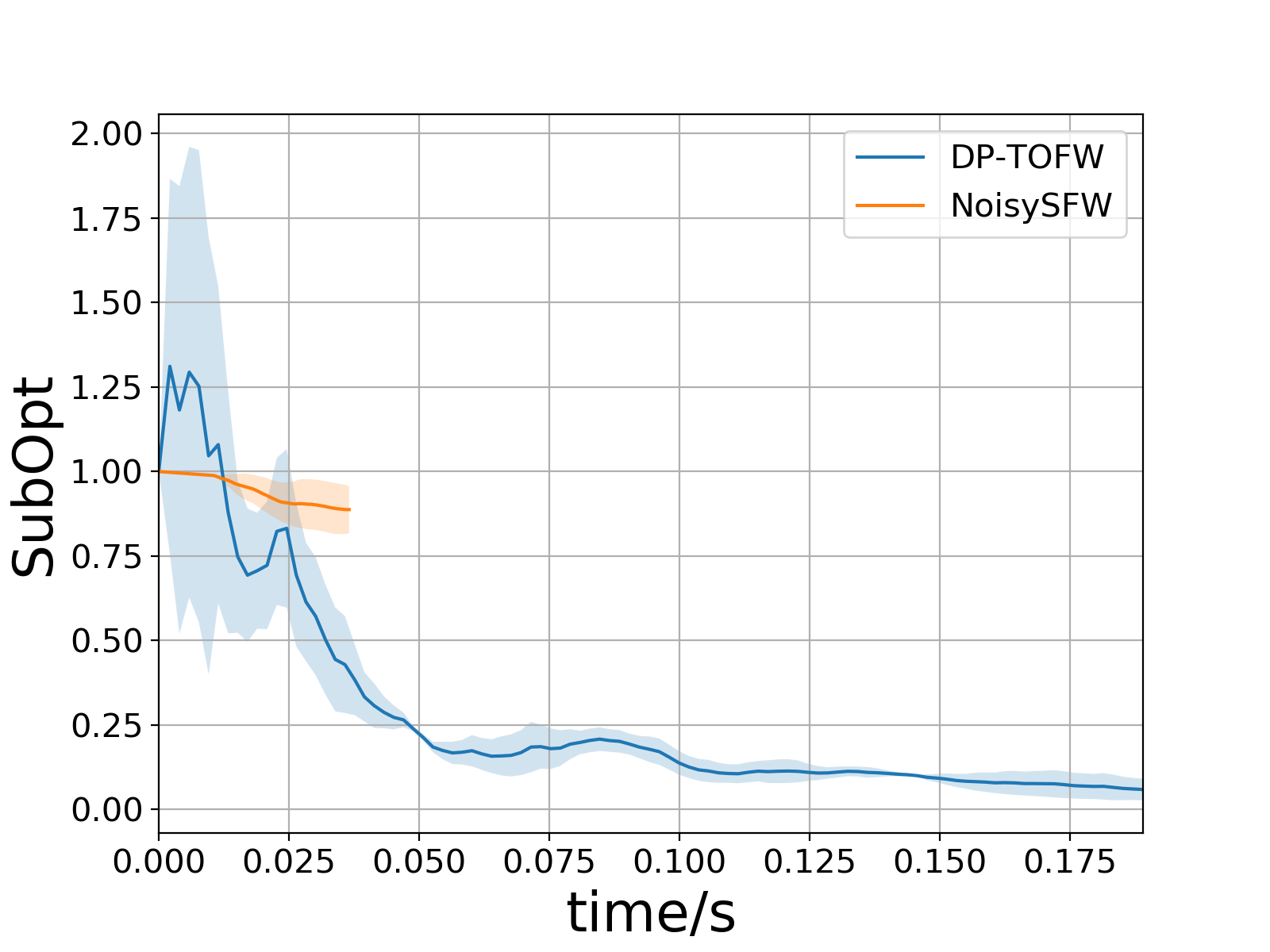}}

     \caption{Comparison of SubOpt and wall-clock time between DP-TOFW and NoisySFW. The shadows indicate $\pm 1\times$std across different random seeds. Here we select the best learning rate scaling for both algorithms, according to the results in Figure \ref{fig-lr-scale}.} \label{fig-time-subopt}
\end{figure}

\section{Conclusions}
In this paper, we present a new framework for the online convex optimization in $\ell_p$ geometry and high dimensional decision making with differential privacy guarantee. 
Our framework can continually release the solutions in a fully-online update manner while still maintain 
privacy protection for the whole time horizon.
Besides the privacy guarantee, our algorithm achieves in linear time the optimal rates when $1<p\leq 2$ and the state-of-the-art rates that matches the non-private lower bound when $2<p\leq \infty$.
The flexibility to extend to $p=1$ case and the novel exploitation of the recursive gradient estimator in our algorithm also allow us to design the first high dimensional bandits algorithm satisfying DP requirements with sub-linear regret. The efficacy of the proposed algorithms are demonstrated by comparative experiments with various popular DP-SCO and DP-Bandit algorithms.

\appendix
\section{ Proofs of Section~\ref{sec:dp-sco}}
 
\begin{algorithm}[ht]
\caption{Private Tree based aggregation protocol \cite{guha2013nearly}}
\label{alg-tree-based}
\begin{algorithmic}[1]
    \STATE {\bfseries Input:} $\langle z_1,z_2, .., z_n \in \mathbb{R}^d \rangle$ (in an online sequence), noise level $\sigma_+(q,\varepsilon,\delta)$
    
    \STATE {\bfseries Initialization}: Define a binary tree $A$ of size $2^{\lceil \log_2  n \rceil  + 1} - 1$ with leaves $z_1, z_2,..., z_n$.
    
    \STATE {\bfseries Online Phase:} At each iteration $t\in[n]$, execute Steps \ref{step-start} till \ref{step-partial-sum-end}
    
    \STATE Accept $z_t$ from the data stream. \label{step-start} 
    
    \STATE Let $\textit{path}=\{z_t \rightarrow \dots \rightarrow \text{root} \}$ be the path from $z_t$ to the root.
    
    \STATE {\bfseries Tree update:} Step \ref{tree-update-start} till \ref{tree-update-end}
    
    \STATE $\Lambda \leftarrow $ First node in \textit{path} that is left child in the tree. Let $\textit{path}_\Lambda = \{z_t \rightarrow \dots \rightarrow \Lambda \}$. \label{tree-update-start}

    \FOR {$\alpha$ {\bfseries in} \textit{path}}
    \STATE $\alpha \leftarrow \alpha + z_t$
    \STATE {\bfseries If} $\alpha \in \text{path}_\Lambda$, \textbf{then} $\alpha \leftarrow \alpha + \textbf{n}$ where $\textbf{n} \sim \mathcal{G}_{\|\cdot\|_{q,+}}(0,\sigma_+^2)$.
    
    \ENDFOR \label{tree-update-end}
    
    \STATE {\bfseries Output Private Partial Sum:} Step \ref{step-partial-sum-start} till \ref{step-partial-sum-end}
    
    \STATE Initial Vector $v\in\mathbb{R}^d$ to zero. Let $b\leftarrow \lceil \log_2 n \rceil +1 $-bit binary representation of $t$. \label{step-partial-sum-start}
    
    \FOR {{\bfseries all} $i$ {\bfseries in} $[\lceil \log_2 n \rceil + 1]$}
    
    \IF {bit $b_i=1$}
    
    \IF {$i$-th node in \textit{path} (denoted by $\textit{\text{path}}(i)$) is a left child in $A$} 
    \STATE $v\leftarrow v+ \textit{\text{path}}(i)$,
    \ELSE 
    \STATE $v\leftarrow v+ \text{left sibling}\big(\textit{\text{path}}(i)\big)$.
    \ENDIF
    \ENDIF
    \ENDFOR
    
    \RETURN The noisy partial sum $v$. \label{step-partial-sum-end}
    
\end{algorithmic}
\end{algorithm}

\subsection{Proof of Theorem~\ref{thm-privacy-lp} }\label{proof-privacy-lp}

\begin{proof}[Proof of Theorem~\ref{thm-privacy-lp}]

We expend $d_t$ as follow
\begin{equation}\label{eq-expand-dt}
    \begin{aligned}
        d_t 
        &= \nabla f(\theta_t, x_t) + (1-\rho_t) (d_{t-1} - \nabla f (\theta_{t-1}, x_t)) \\
        &= \sum_{i=1}^t \bigg( \prod_{k=i+1}^t (1-\rho_k) \nabla f(\theta_i, x_i) - \prod_{k=i}^t (1-\rho_k) \nabla f(\theta_{i-1}, x_i)\bigg) \\
        &= \frac{1}{t+1}\sum_{i=1}^t \Big( (i+1) \nabla f(\theta_i, x_i) - i \nabla f(\theta_{i-1}, x_i)\Big),
    \end{aligned}
\end{equation}
where the last inequality is due to the fact that $\rho_t = \frac{1}{t+1}$.
If we consider the tree based mechanism in Algorithm  \ref{alg-tree-based}, each sample $x_i$ is involved in at most $\lceil \log_2 n\rceil + 1$ nodes in the tree. And all partial summations can also be determined by at most $\lceil \log_2 n \rceil$ nodes. The privacy analysis of the partial sum now reduces to the privacy analysis of the tree. 

Suppose adjacent datasets $\dataset$ and $\dataset'$ differ by sample $x_i$ and $x_i'$, then for any sets $B = (B_1, B_2, ..., B_{2^{\lceil\log_2 n \rceil +1}-1})$ corresponding to the post-order traversal of the binary tree, it suffices to prove that 
\begin{equation*}
\begin{aligned}
    & \mathbb{P}(A_1(\dataset )\in B_1, ..., A_{2^{\lceil \log_2 n \rceil+1}-1}(\dataset) \in B_{2^{\lceil \log_2 n\rceil +1}-1}) \\
    & \quad \leq e^\varepsilon \mathbb{P}(A_1(\dataset')\in B_1, ..., A_{2^{\lceil \log_2 n \rceil +1}-1}(\dataset') \in B_{2^{\lceil \log_2 n \rceil+1} -1}) + \delta.
\end{aligned}
\end{equation*}
Here $2^{\lceil{\log_2 n\rceil}+1}-1$ is the maximum number of nodes (including root and leaves) in a tree with $\lceil{\log_2 n\rceil}+1$ levels. For node $A_m$ including $x_i$, suppose that it stores the summation $\sum_{j=k}^l \big((j+1)\nabla f(\theta_j, x_j) - j \nabla f(\theta_{j-1}, x_{j})\big)$, we have then conditioned on $A_1(\dataset)=A_1(\dataset'),...,A_{m-1}(\dataset)=A_{m-1}(\dataset')$, $\theta_j(\dataset)=\theta_j(\dataset')=\theta_j, \forall j\leq l$. Thus the difference between $x_i$ and $x_i'$ will cause the difference between
\begin{equation*}
    (i+1)\nabla f(\theta_i, x_i) - i \nabla f(\theta_{i-1}, x_i) \quad \text{and} \quad (i+1)\nabla f(\theta_i, x_i') - i \nabla f(\theta_{i-1}, x_{i}').
\end{equation*}
which has $\ell_q$ sensitivity $2(\smooth \diam + \lip)$ because
\begin{equation*}
    \begin{aligned}
         &\quad \| \big((i+1)\nabla f(\theta_i, x_i) - i \nabla f(\theta_{i-1}, x_{i})\big) - \big((i+1)\nabla f(\theta_i, x_i') - i \nabla f(\theta_{i-1}, x_{i}')\big) \|_q \\
        &\leq 2i\smooth\|\theta_i - \theta_{i-1}\|_p + \|\nabla f(\theta_i, x_i) - \nabla f(\theta_i,x_i') \|_q  \\
        &\leq 2(\smooth \diam+ \lip).
    \end{aligned}
\end{equation*}
According to the above sensitivity, and using the fact that $\|\cdot\|_{q,+}$ is $\kappa_{q,+}$-smooth, we can now apply the generalized Gaussian in Lemma~\ref{lemma-general-gaussian}. We add noise $\mathcal{G}_{\|\cdot\|_+}(0,8(\lceil\log_2 n\rceil +1)^2 \kappa_q \log((\lceil \log_2 n\rceil + 1)/\delta)(\smooth \diam +\lip)^2/\varepsilon^2)$ independently to each node to ensure that each node is $(\varepsilon/(\lceil \log_2 n \rceil +1), \delta/(\lceil \log_2 n \rceil +1))$-differentially private. 

We recall that each sample $x_i$ is involved in at most $\lceil \log_2 n\rceil+1$ nodes in the tree. We denote the path from $x_i$ to the root of the tree as $\textit{path}_i$, where $|\textit{path}_i|\leq \lceil \log_2 n \rceil+1$. And here we use $p$ to denote the density of $(A_1(\dataset), ..., A_{2^{\lceil \log_2 n \rceil+1}-1}(\dataset))$ and $p'$ for its counterpart regarding dataset $\dataset'$. Then for any $B = (B_1, B_2, ..., B_{2^{\lceil\log_2 n \rceil +1}-1})$, we have
\begin{equation*}
    \begin{aligned}
        &\quad \mathbb{P}(A_1(\dataset)\in B_1, ..., A_{2^{\lceil \log_2 n\rceil+1}-1}(\dataset) = B_{2^{\lceil \log_2 n\rceil+1}-1}(\dataset)) \\
        &= \int_{B_1 \times, ..., \times B_{2^{\lceil \log_2 n\rceil+1}-1}} p(a_1, ..., a_{2^{\lceil \log_2 n\rceil+1}-1}) d a_1 ...da_{2^{\lceil \log_2 n\rceil+1}-1} \\
        &= \int \prod_{m\in\textit{path}_i}p(a_m| a_1, ...,a_{m-1}) \cdot \prod_{m\notin\textit{path}_i}p(a_m| a_1, ...,a_{m-1})  d a_1 ...da_{2^{\lceil \log_2 n\rceil+1}-1}.
    \end{aligned}
\end{equation*}
Notice that for any $m\notin \textit{path}_i$, $p(a_m|a_1,...,a_{m-1}) = p'(a_m|a_1,...,a_{m-1})$. For $m\in \textit{path}_i$,
\begin{equation*}
    \begin{aligned}
        \int_{B_m} p(a_m|a_1,...,a_{m-1}) da_m 
        &= \int_{B_m} p(a_m|a_1,...,a_{m-1}) - 1\land e^{\varepsilon/(\lceil \log_2 n \rceil+1)} p'(a_m|a_1,...,a_{m-1}) d a_m +\dots  \\
        &\quad + \int_{B_m} 1\land e^{\varepsilon/(\lceil \log_2 n \rceil+1)} p'(a_m|a_1,...,a_{m-1})  d a_m \\
        &\leq \delta/(\lceil \log_2 n \rceil+1) + \int_{B_m} 1\land e^{\varepsilon/(\lceil \log_2 n \rceil+1)} p'(a_m|a_1,...,a_{m-1})  d a_m.
    \end{aligned}
\end{equation*}
Applying the above inequality to any node in $\textit{path}_i$, we have
\begin{equation*}
\begin{aligned}
    &\quad \int \prod_{m\in\textit{path}_i}p(a_m| a_1, ...,a_{m-1}) \cdot \prod_{m\notin\textit{path}_i}p(a_m| a_1, ...,a_{m-1})  d a_1 ...da_{2^{\lceil \log_2  n\rceil+1}-1} \\
    &\leq e^\varepsilon \int \prod_{m\in\textit{path}_i}p'(a_m| a_1, ...,a_{m-1}) \cdot \prod_{m\notin\textit{path}_i}p'(a_m| a_1, ...,a_{m-1})  d a_1 ...da_{2^{\lceil \log_2 n\rceil+1}-1} +\delta \\
    &= e^\varepsilon \mathbb{P}(A_1(\dataset')\in B_1, ..., A_{2^{\lceil \log_2 n\rceil+1}-1}(\dataset') = B_{2^{\lceil \log_2 n\rceil+1}-1}(\dataset')) +\delta,
\end{aligned}
\end{equation*}
which concludes the proof.
\end{proof}

\subsection{Proof of Lemma~\ref{lemma-gamma}}\label{proof-lemma-gamma}

\begin{proof}[Proof of Lemma~\ref{lemma-gamma}]
Since each $Z_{j}$ are i.i.d. $  \mathcal{G}_{\lVert \cdot \rVert_+}(0,\sigma_+^2 )$, we have 
\begin{align*}
     \mathbb{P}(\lVert Z_j\rVert_{+}^2 >\lambda)  & = C(\sigma_+,d)\text{Area}\{ \lVert x\rVert_{+} = 1\}\int_{ r^2 > \lambda }r^{d-1} \exp( - \dfrac{r^2}{2\sigma_+^2} ) dr\\
    & =\dfrac{1}{2} C(\sigma_+,d)\text{Area}\{ \lVert x\rVert_{+} = 1\} \int_{ r>{\lambda} } r^{d/2-1} \exp(- \frac{r}{2\sigma_+^2})   dr.
\end{align*}
By 
\begin{align*}
    C(\sigma_+,d)\text{Area}\{ \lVert x\rVert_{+} = 1\} &= \dfrac{1}{(2\sigma_+^2)^{d/2}\cdot \Gamma(d/2)/2},
\end{align*}
we know that the tail of $\|Z_j\|_{+}^2$ is exactly the tail of $\Gamma(d/2,2\sigma_+^2)$ at $\lambda$, which means $\|Z_j\|_{+}^2$ follows $\Gamma(d/2,2\sigma_+^2)$. Thus $\lVert Z_{j}\rVert_{+}^2 - \E[\lVert Z_{j}\rVert_{+}^2]$ is subGamma$(2\sigma_+^4 d, 2\sigma_+^2 )$ , then the standard tail bound of sub-Gamma distribution implies 
\begin{align}\label{eq-gamma-tail}
    P(\lVert Z_j\rVert_{+}^2 > \E[\lVert Z_j\rVert_{+}^2] + 2\sqrt{\sigma_+^4 d\lambda  }+    2\sigma_+^2\lambda   ) \leq  \exp(-\lambda) 
\end{align}
\end{proof}

\subsection{Proof of Proposition~\ref{prop-lp-variance-reduction}}\label{proof-prop-lp-grad-error}

\begin{proposition}[Azuma-Hoeffding inequality in regular space]\label{prop: regular hoeffding}
Given the $\kappa$-smooth norm $\lVert\cdot\rVert$ and a vector-valued martingale difference sequence $ \mathbf{d}_t $ with respect to $\{\mathcal{F}_t\}_t$, we have if \begin{align}
    \mathbb{E}[\exp(\lVert \mathbf{d}_t\rVert^2/\sigma_t^2 ) |\mathcal{F}_{t-1}]\leq \exp(1) ,\quad \forall t,
\end{align} 
then \begin{align*}
    \mathbb{P}\bigg( \Big\lVert \sum_{i=1}^t \mathbf{d}_i\Big\rVert \geq (\sqrt{2e\kappa}+\sqrt{2}\lambda ) \Big(\sum_{i = 1}^t \sigma_i^2\Big)^{1/2}\bigg)\leq 2\exp(- \lambda^2/64 ).
\end{align*}
\end{proposition}

We provide the a detailed version of Proposition~\ref{prop-lp-variance-reduction} in the following proposition.
\begin{proposition}\label{prop-grad-error-lp-full}
We denote $\Delta_t = d_t -\nabla F(\theta_t)$. Assume Assumption~\ref{assumption-smooth} and \ref{assumption-grad-bound}, for $t\in[n]$, we have that with probability at least $1-\alpha$, Algorithm~\ref{alg-dp-sco-lp} will satisfies
\begin{align*}
    \lVert \Delta_t\rVert_{q} &\leq (\sqrt{2e \kappa_q}+8\sqrt{4\log(2/\alpha) })\frac{2(\smooth\diam + \gradvar  ) }{\sqrt{t+1}} + \dots\\
    &\quad +  \lceil\log_2 n\rceil \frac{\sigma_+}{t+1} \big( d+2\sqrt{d\log(2\lceil\log_2 n\rceil/\alpha)}+2d\log(2\lceil\log_2 n\rceil/\alpha)   \big)^{1/2}.
\end{align*}
\end{proposition}
\begin{proof}
We first reformulate $\Delta_t = d_t - \nabla F(\theta_t)$ as the sum of a martingale difference sequence. We denote $M_t$ the set of node indices used when reporting $d_t$ and $Z$ the noise in the tree based mechanism in Algorithm~\ref{alg-tree-based} . For $t\geq 1$, we have
\begin{equation}
    \begin{aligned}
        \Delta_t  & =\dfrac{1}{1+t} \sum_{j\in M_{t} }  Z_j   +  \nabla f(\theta_t, x_t) + (1-\rho_t) (d_{t-1} - \nabla f(\theta_{t-1}, x_{t})) - \nabla F(\theta_t) \\
        &=\dfrac{1}{1+t}\sum_{j\in M_{t} }  Z_j+ (1-\rho_t) \Delta_{t-1} + \rho_t (\nabla f(\theta_t, x_t) - \nabla F(\theta_t)) + \dots \\
        &\quad + (1-\rho_t) \big(\nabla f(\theta_t, x_t) - \nabla f(\theta_{t-1}, x_t) - (\nabla F(\theta_t) - \nabla F(\theta_{t-1}))\big) \\
        &= \dfrac{1}{1+t}\sum_{j\in M_{t} }  Z_j+ \prod_{k=2}^t (1-\rho_k) \Delta_1
        + \sum_{\tau=2}^t \bigg( \rho_\tau \prod_{k=\tau+1}^t (1-\rho_k) \big( \nabla f(\theta_\tau, x_\tau) - \nabla F(\theta_\tau) \big)  +\dots \\
        &\quad + \prod_{k=\tau}^t (1-\rho_k)\big(\nabla f(\theta_\tau, x_\tau) - \nabla f(\theta_{\tau-1}, x_\tau) - (\nabla F(\theta_\tau) - \nabla F(\theta_{\tau-1})\big)\bigg) \\
        & \triangleq \frac{1}{t+1}\sum_{j\in M_{t} }  Z_j+  \zeta_{t, 1} + \sum_{\tau=2}^t \zeta_{t, \tau}
    \end{aligned}
\end{equation}
Recall that $\Delta_1 = \nabla f(\theta_1, x_1) - \nabla F(\theta_1)$. And we observe that $\mathbb{E}[\zeta_{t, \tau}| \mathcal{F}_{\tau-1}]=0$ where $\mathcal{F}_\tau$ is the $\sigma$-field generated by $\{x_1, x_2, ..., x_{\tau-1} \}$. Therefore, $\{\zeta_{t, \tau}\}_{\tau=1}^t$ is a martingale difference sequence. 
In what follows, we derive upper bounds of $\|\zeta_{t, \tau}\|_q$.  We start by observing that for any $\tau=1,2,...,t$, 
\begin{equation}
    \begin{aligned}
        \prod_{k=\tau}^t (1-\rho_k) 
        &= \prod_{k=\tau}^t(1- \frac{1}{k+1}) 
        = \prod_{k=\tau}^t \frac{k}{k+1} =\frac{\tau}{t+1}.
    \end{aligned}
\end{equation}
We can bound $\|\zeta_{t,1}\|_q$:
\begin{equation*}
    \begin{aligned}
        \|\zeta_{t, 1} \|_q \leq \frac{1}{t+1} \| \nabla f(\theta_1, x_1) - \nabla F(\theta_1)\|_q \leq \frac{\gradvar }{t+1} \triangleq c_{t,1},
    \end{aligned}
\end{equation*}
where the second inequality follows from Assumption~\ref{assumption-grad-bound}. For $\tau >1$, \begin{equation}\label{eq-gradient-error-lp-eta-term}
    \begin{aligned}
       \| \zeta_{t, \tau}\|_q  
       &\leq   \prod_{k=\tau}^t (1-\rho_k)\big(\| \nabla f(\theta_\tau, x_\tau) - \nabla f(\theta_{\tau-1}, x_\tau)\|_q + \|\nabla F(\theta_\tau) - \nabla F(\theta_{\tau-1})\|_q \big) + \dots \\
       &\quad + \rho_\tau \prod_{k=\tau+1}^t (1-\rho_k) \| \nabla f(\theta_\tau, x_\tau) - \nabla F(\theta_\tau)\|_q  \\
       &\leq 2\smooth \|\theta_\tau - \theta_{\tau-1} \|_p \prod_{k=\tau}^t (1-\rho_k) + \gradvar \rho_\tau \prod_{k=\tau+1}^t (1-\rho_k) \\
       &= 2\smooth \eta_{\tau-1} \| v_{\tau-1} - \theta_{\tau-1} \|_p \prod_{k=\tau}^{t} (1-\rho_k) + \gradvar \rho_\tau \prod_{k=\tau+1}^t (1-\rho_k) \\
       &\leq\frac{2(\smooth \diam +\gradvar)}{t+1} \triangleq c_{t,\tau},
    \end{aligned}
\end{equation}
where the second inequality follows from Assumption~\ref{assumption-smooth} and \ref{assumption-grad-bound}, and the last inequality is due to $\eta_\tau = \rho_\tau$ and the definition of $\diam$. 
Now according to Proposition~\ref{prop: regular hoeffding}, we have
\begin{equation}\label{eq-mds-single-value-lp}
    \mathbb{P} \bigg(\Big\|\Delta_{t}  - \dfrac{1}{1+t} \sum_{j\in M_{t} }  Z_j   \Big\|_{q} \geq  (\sqrt{2e \kappa_q}+\sqrt{2}\lambda)\Big(\sum_{\tau = 1}^t c_{t,\tau}^2 \Big)^{1/2}   \bigg)\leq 2\exp(- \lambda^2/64 ),
\end{equation}
We can bound $\sum_{\tau=1}^t c^2_{t,\tau}$ as
\begin{equation*}
    \begin{aligned}
        \sum_{\tau=1}^t c^2_{t,\tau} = c^2_{t,1} + \sum_{\tau=2}^t c^2_{t,\tau} = \bigg(\frac{G}{t+1}\bigg)^2 + \sum_{\tau=2}^t  \bigg(\frac{2\smooth \diam + \gradvar}{t+1} \bigg)^2 \leq \sum_{\tau=1}^t  \bigg(\frac{2\smooth \diam + 2\gradvar}{t+1} \bigg)^2 \leq \frac{4(\smooth \diam +\gradvar)^2}{t+1}.
    \end{aligned}
\end{equation*}
Plugging the above bound into Eq. (\ref{eq-mds-single-value-lp}) and setting
\begin{equation*}
    \lambda = 8\sqrt{\log( 2/\smallprob)} ,
\end{equation*}
we have with probability at least $1-\smallprob$,
\begin{equation*}
    \Big\|\Delta_{t}   - \frac{1}{t+1}\sum_{j\in M_{t} } Z_j \Big\|_{q} \leq (\sqrt{2e \kappa_q}+8\sqrt{2\log(2/\smallprob) })\dfrac{2(\smooth \diam +\gradvar)}{\sqrt{t+1}}.
\end{equation*}
According to Lemma~\ref{lemma-gamma}, we know that $\|Z_j\|_{q,+}^2$ follows Gamma distribution $\Gamma(d/2,2\sigma_+)$. 
Selecting $\lambda = \log(\lceil\log_2 n \rceil/\smallprobtail)$, by $\E[\lVert Z_j\rVert_{q,+}^2] = \sigma_+^2 d$ and Eq.~(\ref{eq-gamma-tail}), we get with probability at least $1-\smallprobtail/\lceil\log_2 n\rceil$, \begin{align*}
    \lVert Z_j\rVert_{q,+}^2 &\leq \sigma_+^2 d + 2 \sigma_+^2 \sqrt{d\log(\lceil\log_2 n\rceil/\smallprobtail) }+ 2\sigma_+^2 d\log(\lceil \log_2 n \rceil/\smallprobtail).
\end{align*}
Thus with probability at least $1-\smallprobtail ,$ we have \begin{align}
    \max_{j\in M_t}  \lVert Z_j\rVert_{q,+}^2 &\leq \sigma_+^2 d + 2 \sigma_+^2 \sqrt{d\log(\lceil\log_2 n\rceil/\smallprobtail)}+ 2\sigma_+^2 d\log(\lceil\log_2 n \rceil/\smallprobtail), 
\end{align}
here we use the fact that $|M_t|\leq \lceil \log_2 n\rceil$. 
Thus with probability at least $1-\smallprobtail$,  
\begin{align*}
     \bigg\lVert \sum_{j\in M_t} Z_{j} \bigg\rVert_{q,+} 
     &\leq \lceil \log_2 n \rceil  \max_{j\in M_t}  \lVert Z_j\rVert_{q,+}\\
     &\leq \lceil\log_2 n \rceil \sigma_+ \big( d+2\sqrt{d\log(\lceil\log_2 n \rceil/\smallprobtail)}+2d\log(\lceil\log_2  n \rceil/\smallprobtail)   \big)^{1/2}.
\end{align*}
According to the norm equivalent property in Definition~\ref{lemma-kappa-regularity}, we have 
\begin{equation*}
    \bigg\lVert \sum_{j\in M_t} Z_{j} \bigg\rVert_{q} \leq \bigg\lVert \sum_{j\in M_t} Z_{j} \bigg\rVert_{q,+}.
\end{equation*}
As a result, by setting $\smallprob=\smallprobtail=\frac{\alpha}{2}$, we have with probability at least $1-\alpha$,  
\begin{align*}
    \lVert \Delta_t\rVert_{q} &\leq (\sqrt{2e \kappa_q}+8\sqrt{4\log(2/\alpha) })\dfrac{2(\smooth\diam + \gradvar  ) }{\sqrt{t+1}} + \dots \\
    &\quad +  \frac{\lceil\log_2 n\rceil \sigma_+ \big( d+2\sqrt{d\log(2\lceil\log_2 n\rceil/\alpha)}+2d\log(2\lceil\log_2 n\rceil/\alpha)   \big)^{1/2}}{t+1}.
\end{align*}
\end{proof}

\subsection{Proof of Theorem~\ref{thm-convergence-lp-general}}\label{proof-lp-general}

We provide a detailed version of Theorem~\ref{thm-convergence-lp-general} in the following Theorem.

\begin{theorem}\label{thm-detailed-lp-convex}
Consider Algorithm~\ref{alg-dp-sco-lp} with convex function $F$, Assumption~\ref{assumption-smooth}, \ref{assumption-grad-bound} and \ref{assumption-lip}, for $t\in[n]$, we have with probability at least $1-\alpha$, 
\begin{equation}
\begin{aligned}
    F(\theta_t)-F(\theta^*)\leq & \frac{2(\sqrt{2e \kappa_q}+8\sqrt{4\log(2n/\alpha) })\diam(\smooth\diam + \gradvar)}{\sqrt{t}} +\frac{(\log t+1)\beta D^2}{2t} + \dots \\
    &+\frac{\log t}{ t}\cdot \diam \lceil\log_2 n\rceil \sigma_+ \big( d+2\sqrt{d\log(2n\lceil\log_2 n\rceil/\alpha)}+2d\log(2n\lceil\log_2 n\rceil/\alpha)   \big)^{1/2}.
\end{aligned}
\end{equation}
\end{theorem}

\begin{proof}
We start from $\beta$-smoothness:
\begin{equation*}
\begin{aligned}
    F(\theta_{t+1}) &\leq F(\theta_t) + \langle \nabla F(\theta_t), \theta_{t+1} - \theta_t \rangle + 
    \frac{\smooth}{2}\|\theta_{t+1} - \theta_t\|_p^2 \\
    &\leq F(\theta_t) + \eta_t \langle \nabla F(\theta_t), v_t - \theta_t \rangle + 
    \frac{\eta_t^2 \smooth \diam^2}{2}.
\end{aligned}
\end{equation*}
We subtract $F(\theta^*)$ from both sides, and denote $h_t=F(\theta_t) - F(\theta^*)$. We have
\begin{equation*}
\begin{aligned}
    h_{t+1} 
    &\leq h_t + \eta_t \langle \nabla F(\theta_t), v_t - \theta_t \rangle + 
    \frac{\eta_t^2 \smooth\diam^2}{2} \\
    &= h_t + \eta_t \langle \nabla F(\theta_t) - d_t, v_t - \theta_t \rangle + \eta_t \langle d_t, v_t - \theta_t \rangle + 
    \frac{\eta_t^2 \smooth \diam^2}{2} \\
    &\leq h_t + \eta_t \langle \nabla F(\theta_t) - d_t, v_t - \theta_t \rangle + \eta_t \langle d_t, \theta^* - \theta_t \rangle + 
    \frac{\eta_t^2 \smooth \diam^2}{2} \\
    &= h_t + \eta_t \langle  d_t- \nabla F(\theta_t), \theta^* - v_t \rangle + \eta_t \langle \nabla F(\theta_t), \theta^* - \theta_t \rangle + 
    \frac{\eta_t^2 \smooth \diam^2}{2} \\
    &\leq h_t + \eta_t \diam \|d_t - \nabla F(\theta_t)\|_q + \eta_t \langle \nabla F(\theta_t), \theta^* - \theta_t \rangle + 
    \frac{\eta_t^2 \smooth \diam^2}{2} \\
    &\leq (1-\eta_t) h_t + \eta_t \diam \|d_t - \nabla F(\theta_t)\|_q +\frac{\eta_t^2 \smooth \diam^2}{2}. 
\end{aligned}
\end{equation*}
where the second inequality is due to definition of $v_t$. According to Proposition~\ref{prop-grad-error-lp-full}, with probability at least $1-t\alpha'$, we have
\begin{equation*}
    \begin{aligned}
        h_{t+1} 
        &\leq (1-\eta_t) h_t + \frac{\smooth \diam^2}{2(t+1)^2} + \dots \\  
        &\quad +\frac{1}{t+1} \diam\Big( \dfrac{2(\sqrt{2e \kappa_q}+8\sqrt{4\log(2/\alpha') })(\smooth\diam + \gradvar  ) }{\sqrt{t+1}}+ \dots \\
        &\qquad \qquad \qquad + \frac{\lceil\log_2 n\rceil \sigma_+ \big( d+2\sqrt{d\log(2\lceil\log_2 n\rceil/\alpha')}+2d\log(2\lceil\log_2 n\rceil/\alpha')   \big)^{1/2}}{t+1}\Big) \\
        &\leq (1-\eta_t) h_t 
        + \frac{1}{(t+1)^{3/2}}\underbrace{2(\sqrt{2e \kappa_q}+8\sqrt{4\log(2/\alpha') })\diam(\smooth\diam + \gradvar)}_{C_1} + \dots \\
        &\quad +  \frac{1}{(t+1)^2}  \Big(\underbrace{\diam \lceil\log_2 n\rceil \sigma_+ \big( d+2\sqrt{d\log(2\lceil\log_2 n\rceil/\alpha')}+2d\log(2\lceil\log_2 n\rceil/\alpha')   \big)^{1/2} + \smooth \diam^2/2}_{C_2}\Big).
    \end{aligned}
\end{equation*}
Then we have
\begin{equation*}
    \begin{aligned}
        h_{t+1}
        &= (1-\eta_t)h_t + \frac{C_1}{(t+1)^{3/2}} + \frac{C_2}{(t+1)^2} \\
        &=h_1\prod_{\tau=1}^{t} (1-\eta_\tau) + \sum_{k=1}^t\bigg( \dfrac{C_1}{(k+1)^{3/2}} +\dfrac{C_2}{(k+1)^2}   \bigg)\prod_{\tau = k+1}^t (1-\eta_\tau) \\
        &= \frac{1}{t+1}h_1 + \sum_{k=1}^t\bigg( \dfrac{C_1}{(k+1)^{3/2}} +\dfrac{C_2}{(k+1)^2}   \bigg)\prod_{\tau = k+1}^t (1-\eta_\tau)\\
        &= \frac{1}{t+1}h_1 + \frac{1}{t+1}\sum_{k=1}^t\bigg( \dfrac{C_1}{(k+1)^{1/2}} +\dfrac{C_2}{(k+1)}   \bigg)\\
        &\leq \frac{1}{t+1}h_1 + \frac{C_1}{\sqrt{t+1}} + \frac{C_2 \log t}{t+1}.
    \end{aligned}
\end{equation*}
Now  setting $\alpha'=\frac{\alpha}{n}$, and recalling that $h_1 \leq \frac{\smooth \diam^2}{2}$ according to $\smooth$-smoothness lead to the desired result.
\end{proof}

\subsection{Proof of Theorem~\ref{thm-convergence-lp-strongly}}\label{proof-lp-strongly}

We firstly introduce the following lemma.

\begin{lemma}[Lemma 6 in \cite{lafond2015online}]\label{lemma-fw-strongly-convex}
Assume Assumption~\ref{assumption-boundary}, and the population loss function $F$ satisfies Assumption~\ref{assumption-strongly-convex} and \ref{assumption-smooth}, then 
\begin{equation*}
    \Big(\max_{\theta\in \mathcal{C}} \langle \nabla F(\theta_t), \theta_t - \theta \rangle  \Big)^2 \geq 2\strongly\boundary^2 h_t \quad \text{and} \quad \smooth \diam^2 \geq \boundary^2 \strongly.
\end{equation*}
where $h_t = F(\theta_t) - F(\theta^*)$.
\end{lemma}

We provide a detailed version of Theorem~\ref{thm-convergence-lp-strongly} in the following Theorem.

\begin{theorem}
Consider Algorithm~\ref{alg-dp-sco-lp} with Assumption~\ref{assumption-strongly-convex}, \ref{assumption-smooth}, \ref{assumption-grad-bound}, \ref{assumption-lip} and \ref{assumption-boundary}, for $t \in [n]$, we have with probability at least $1-\alpha$, 
\begin{equation*}
    \begin{aligned}
        &F(\theta_t) - F(\theta^*) \\
        &\quad \leq \frac{18}{\boundary^2\strongly}\frac{4\diam^2(\sqrt{2e \kappa_q}+8\sqrt{4\log(2n/\alpha) })^2(\smooth\diam + \gradvar)^2}{t+1} + \dots \\
        &\quad + \frac{18}{\boundary^2\strongly} \frac{\Big(\diam \lceil\log_2 n\rceil \sigma_+ \big( d+2\sqrt{d\log(2n\lceil\log_2 n\rceil/\alpha)}+2d\log(2T\lceil\log_2 n\rceil/\alpha) \big)^{1/2} + \smooth \diam^2/2\Big)^2\log n}{(t+1)^2}.
    \end{aligned}
\end{equation*}
\end{theorem}

\begin{proof}
We denote $h_t=F(\theta_t) - F(\theta^*)$, and $\tilde{\theta}_t\coloneqq \arg\max_{\theta\in \mathcal{C}} (\langle \nabla F(\theta_t), \theta_t - \theta \rangle)^2$ in Lemma~\ref{lemma-fw-strongly-convex}. We start from $\beta$-smoothness:
\begin{equation}\label{eq-fw-strongly-convex-lp}
\begin{aligned}
    h_{t+1} 
    &\leq h_t + \eta_t \langle \nabla F(\theta_t) , v_t - \theta_t \rangle + \frac{\eta_t^2\smooth\diam^2}{2} \\
    &= h_t + \eta_t \langle d_t , v_t - \theta_t \rangle - \eta_t \langle d_t - \nabla F(\theta_t), v_t-\theta_t \rangle+ \frac{\eta_t^2\smooth\diam^2}{2} \\
    &\leq h_t + \eta_t \langle d_t , \tilde{\theta}_t - \theta_t \rangle - \eta_t \langle d_t - \nabla F(\theta_t), v_t-\theta_t \rangle+ \frac{\eta_t^2\smooth\diam^2}{2}\\
    &= h_t + \eta_t \langle \nabla F(\theta_t) , \tilde{\theta}_t - \theta_t \rangle + \eta_t \langle d_t - \nabla F(\theta_t), \tilde{\theta}_t - v_t \rangle+ \frac{\eta_t^2\smooth\diam^2}{2} \\
    &\leq h_t + \eta_t \| d_t - \nabla F(\theta_t)\|_q \diam -\eta_t \boundary\sqrt{2\strongly h_t} + \frac{\eta_t^2\smooth\diam^2}{2}
\end{aligned}
\end{equation}
where the first inequality is due to the definition of $v$ and the last inequality comes from Lemma~\ref{lemma-fw-strongly-convex}. According to Proposition~\ref{prop-grad-error-lp-full}, with probability at least $1-t\alpha'$, we have 
\begin{equation}\label{eq-strongly-lp-before-induction}
    \begin{aligned}
        h_{t+1} 
        &\leq \sqrt{h_t} (\sqrt{h_t} - \eta_t \boundary\sqrt{2\strongly} ) + \frac{1}{(t+1)^{3/2}} \underbrace{2\diam(\sqrt{2e \kappa_q}+8\sqrt{4\log(2/\alpha') })(\smooth\diam + \gradvar)}_{C_1} + \dots \\
        &\quad + \frac{1}{(t+1)^2} \Big(\underbrace{\diam \lceil\log_2 n\rceil \sigma_+ \big( d+2\sqrt{d\log(2\lceil\log_2 n\rceil/\alpha')}+2d\log(2\lceil\log_2 n\rceil/\alpha') \big)^{1/2} + \smooth \diam^2/2}_{C_2}\Big)\\
        &= \sqrt{h_t} (\sqrt{h_t} - \eta_t \boundary\sqrt{2\strongly} ) + \frac{1}{(t+1)^{3/2}}C_1 + \frac{1}{(t+1)^2}C_2,
    \end{aligned}
\end{equation}
Now the claim holds by induction. We assume that 
\begin{equation*}
\begin{aligned}
    h_t &\leq \frac{1}{t+1}\cdot  \frac{18C_1^2}{\boundary^2\strongly} + \frac{1}{(t+1)^2}\cdot  \frac{18 C_2^2 \log^2 n}{\boundary^2\strongly} \triangleq \frac{1}{t+1}A + \frac{1}{(t+1)^2} B.
\end{aligned}
\end{equation*}
For $t=1$, according to Eq.~(\ref{eq-strongly-lp-before-induction}), we have 
\begin{equation*}
    h_2 \leq h_1 + \frac{C_1}{2\sqrt{2}} + \frac{C_2}{4} \leq \frac{9C_1^2}{\boundary^2\strongly} + \frac{9C_2^2}{2\boundary^2\strongly},
\end{equation*}
where the second inequality comes from Lemma~\ref{lemma-fw-strongly-convex} that $\smooth \diam^2 \geq \boundary^2\strongly$.

For $t\geq 1$. There are two cases.

\textbf{Case 1.} $\sqrt{h_t} - \eta_t \boundary\sqrt{2\strongly} \leq 0:$

Since $\eta=\frac{1}{t+1}$, Eq. (\ref{eq-strongly-lp-before-induction}) yields,
\begin{equation*}
    \begin{aligned}
        h_{t+1} &\leq \frac{1}{(t+1)^{3/2}}C_1 + \frac{1}{(t+1)^2}C_2 \leq \frac{C_1^2}{\boundary^2\strongly(t+1)^{3/2}} + \frac{C_1^2}{\boundary^2\strongly(t+1)^2} \\
        & \leq \frac{1}{t+1}\cdot  \frac{18C_1^2}{\boundary^2\strongly} + \frac{1}{(t+1)^2}\cdot  \frac{18 C_2^2 \log^2 n}{\boundary^2\strongly}.
    \end{aligned}
\end{equation*}
where the second inequality comes from Lemma~\ref{lemma-fw-strongly-convex} that $\smooth \diam^2 \geq \boundary^2\strongly$.

\textbf{Case 2.} $\sqrt{h_t} - \eta_t \boundary\sqrt{2\strongly} > 0:$

According to Eq.~(\ref{eq-strongly-lp-before-induction}) and the assumption that $h_t\leq \frac{A}{t+1}+\frac{B}{(t+1)^2}$, we have
\begin{equation}\label{eq-lp-strongly-case2}
\begin{aligned}
    &h_{t+1} - \frac{A}{t+2} - \frac{B}{(t+2)^2} \\
    &\quad \leq A\bigg(\frac{1}{t+1} - \frac{1}{t+2} \bigg) +B\bigg(\frac{1}{(t+1)^2} - \frac{1}{(t+2)^2} \bigg) + \dots  \\
    &\qquad + \frac{C_1}{(t+1)^{3/2}} + \frac{C_2}{(t+1)^2} - \frac{\boundary}{t+1}\sqrt{2\strongly \Big(\frac{A}{t+1}+\frac{B}{(t+1)^2}\Big)} \\
    &\quad= \frac{A}{(t+1)^2} + \frac{2B}{(t+1)^3} + \frac{C_1}{(t+1)^{3/2}} + \frac{C_2}{(t+1)^2} - \frac{\boundary}{2(t+1)^{3/2}}\sqrt{2\strongly A} - \frac{\boundary}{2(t+1)^2}\sqrt{2\strongly B} \\
    &\quad \leq \frac{A}{(t+1)^2} + \frac{2B}{(t+1)^3} + \frac{C_1}{(t+1)^{3/2}} + \frac{C_2}{(t+1)^2} - \frac{3C_1}{(t+1)^{3/2}} - \frac{3C_2 \log n}{(t+1)^2}\\
    &\quad = \frac{A}{(t+1)^2} + \frac{2B}{(t+1)^3} - \frac{2C_1}{(t+1)^{3/2}} - \frac{2C_2 \log n}{(t+1)^2} \\
    &\quad \leq \frac{2}{(t+1)^{3/2}} \bigg( \frac{A}{(t+1)^{1/2}} + \frac{B}{(t+1)^{3/2}} - C_1 - \frac{C_2 \log n}{(t+1)^{1/2}} \bigg)
\end{aligned}
\end{equation}
Define 
\begin{equation*}
    t_0 \coloneqq \inf \bigg\{ t\geq1: \frac{A}{(t+1)^{1/2}} + \frac{B}{(t+1)^{3/2}} - C_1 - \frac{C_2 \log n}{(t+1)^{1/2}} \leq 0 \bigg\}.
\end{equation*}
According to the definition of $A$ and $B$, $t_0$ exists. For those $t\geq t_0$, the RHS of Eq.~\eqref{eq-lp-strongly-case2} is negative, then the proof is done. For those $t\geq t_0$, we have
\begin{equation*}
       C_1 + \frac{C_2 \log n}{(t+1)^{1/2}} \leq \frac{A}{(t+1)^{1/2}} + \frac{B}{(t+1)^{3/2}}, 
\end{equation*}
which is equivalent to 
\begin{equation*}
       \frac{C_1}{(t+1)^{1/2}} + \frac{C_2 \log n}{t+1} \leq \frac{A}{t+1} + \frac{B}{(t+1)^2}.
\end{equation*}
To finish the proof, it suffices to prove that
\begin{equation*}
    h_{t} \leq \frac{C_1}{(t+1)^{1/2}} + \frac{C_2 \log n}{t+1},
\end{equation*}
which is demonstrated in Theorem~\ref{thm-detailed-lp-convex}. Now we conclude the proof by setting $\alpha'=\alpha/n$.
\end{proof}

\subsection{Proof of Theorem~\ref{thm-privacy-l1}}\label{proof-privacy-l1}
\begin{proof}
Consider two adjacent datasets $\dataset$ and $\dataset'$, and their corresponding $d_t$ and $d_t'$. We denote the sensitivity of $\langle d_t, v\rangle$ as $s_t$, namely $s_t \coloneqq \max_{v\in \mathcal{C}}\max_{\dataset \simeq \dataset'} |\langle d_t - d_t', v \rangle|$. Then
\begin{equation*}
    s_t \leq \max_{\dataset \simeq \dataset'}\diam \|d_t  - d_t'\|_\infty.
\end{equation*}
Now we upper bound the sensitivity of $\|d_t - d_t'\|_\infty$. According to Eq.~(\ref{eq-expand-dt}), we know that 
\begin{align*}
	 d_t = \frac{1}{t+1}\sum_{i=1}^t \Big( (i+1) \nabla f(\theta_i, x_i) - i \nabla f(\theta_{i-1}, x_i)\Big).
\end{align*}
If adjacent datasets $\dataset$ and $\dataset'$ differ in data point $x_i$ and $x_i'$, then 
\begin{equation*}
\begin{aligned}
    &\|d_t - d_t' \|_\infty \\ 
    & = \frac{1}{t+1}\bigg\|\Big((i+1) \nabla f(\theta_i, x_i) - i \nabla f(\theta_{i-1}, x_i)\Big) - \Big((i+1) \nabla f(\theta_i, x_i') - i \nabla f(\theta_{i-1}, x_i')\Big) \bigg\|_\infty \\
    & = \frac{1}{t+1}\bigg\| i\Big(\nabla f(\theta_i, x_i) - \nabla f(\theta_{i-1}, x_i)\Big) - i\Big( \nabla f(\theta_i, x_i') -\nabla f(\theta_{i-1}, x_i')\Big) + \Big( \nabla f(\theta_i, x_i) -\nabla f(\theta_{i-1}, x_i') \Big)\bigg\|_\infty\\
    &\leq \frac{2}{t+1} (i\smooth\|\theta_i - \theta_{i-1} \|_1 + \lip) \leq \frac{2}{t+1} (\smooth\|v_t - \theta_{i-1} \|_1 + \lip) \\
    & \leq \frac{2}{t+1}(\smooth \diam+\lip).
\end{aligned}
\end{equation*}
where the first inequality is due to $\smooth$-smoothness and $\lip$-Lipschitz of $F$. Now we have
\begin{equation*}
    s_t \leq \frac{2\diam(\smooth \diam + \lip)}{t+1}.
\end{equation*}
We denote the selected $v_t$ in each iteration as random variable $A_t$. For any $v_1, v_2, ..., v_n \in \mathcal{C}$, we have 
\begin{equation*}
    \begin{aligned}
        &\log \frac{\mathbb{P}(A_1 = v_1, A_2=v_2, ..., A_n = v_T|\dataset)}{\mathbb{P}(A_1'=v_1, A_2'=v_2, ..., A_n'=v_n|\dataset')} \\ 
        &\quad = \sum_{t=1}^n \log \frac{\mathbb{P}(A_t=v_t|A_{t-1}=v_{t-1}, ..., A_1 = v_1, \dataset)}{\mathbb{P}(A_t'=v_t|A_{t-1}'=v_{t-1}, ..., A_1' = v_1, \dataset')} 
        \coloneqq \sum_{t=1}^n c_t(v_t,..., v_1). 
    \end{aligned}
\end{equation*}
For each $c_t$, since we condition on $A_1=v_1, A_2=v_2,...,A_{t-1}=v_{t-1}$, the randomness of $A_t$ totally comes from the noise $\textbf{n}_v^t \sim \text{Lap}\Big(\frac{2 s_t \sqrt{t \cdot \log n \cdot \log(1/\delta)}}{\varepsilon}\Big)$. 
According to the Report Noisy Max Mechanism in Claim 3.9 in \cite{dwork2014algorithmic}, we have 
\begin{equation*}
    |c_t| \leq \frac{\varepsilon}{2\sqrt{t \cdot \log n\cdot \log(1/\delta)}} \coloneqq \varepsilon_t.
\end{equation*}
Then according to Lemma 3.18 in \cite{dwork2014algorithmic}, we have 
\begin{equation*}
    \mathbb{E}[c_t|v_1, v_2,..., v_{t-1}] \leq \varepsilon_t (e^{\varepsilon_t}-1). 
\end{equation*}
Now, according to Azuma-Hoeffding's inequality, we have 
\begin{equation*}
    \begin{aligned}
        \mathbb{P}\bigg(\sum_{t=1}^n c_t \geq \sum_{t=1}^n \varepsilon_t (e^{\varepsilon_t}-1)+ \sqrt{2\log(1/\delta)}\sqrt{ \sum_{t=1}^n \varepsilon_t^2}\bigg) \leq \delta.
    \end{aligned}
\end{equation*}
So we can get $(\varepsilon', \delta)$-DP, where
\begin{equation*}
\begin{aligned}
    \varepsilon' &= \sum_{t=1}^n \varepsilon_t^2+ \sqrt{2\log(1/\delta)}\sqrt{ \sum_{t=1}^n \varepsilon_t^2} \leq  \varepsilon,
\end{aligned}
\end{equation*}
which concludes the proof.
\end{proof}

\subsection{Proof of Theorem~\ref{thm-convergence-l1-general}}\label{proof-l1-general}

Firstly, we would like to introduce a proposition and a lemma.

\begin{proposition}(Theorem 3.5 in \cite{pinelis1994optimum})\label{prop-vector-value-mds}
Let $\zeta_1, \zeta_2,..., \zeta_t\in \mathbb{R}^d$ be a vector-valued martingale difference sequence w.r.t. a filtration $\{\mathcal{F}_t\}$, i.e. for each $\tau\in 1,2,...,t$, we have $\mathbb{E}[\zeta_\tau| \mathcal{F}_{\tau-1}]=0$. Suppose that $\|\zeta_\tau\|_2\leq c_{\tau}$ almost surely. Then, $\forall t\geq 1$,
\begin{equation*}
    \mathbb{P} \bigg(\bigg\|\sum_{\tau=1}^t \zeta_\tau \bigg\|_2 \geq \lambda \bigg) \leq 4\exp \bigg( -\frac{\lambda^2}{4\sum_{\tau=1}^t c_\tau^2} \bigg).
\end{equation*}
\end{proposition}

\begin{lemma}\label{lemma-fw-gradient-error-l1}
Assume Assumption \ref{assumption-smooth} and  \ref{assumption-grad-bound}, for $t\in[n]$, we have that with probability at least $1-\smallprob$, Algorithm~\ref{alg-dp-sco-l1} will statisfies
\begin{equation}
    \|\Delta_t \|_\infty \coloneqq \| d_t- \nabla F(\theta_t) \|_\infty \leq\frac{4(\smooth \diam+\gradvar)\sqrt{\log(4d/\smallprob)}}{\sqrt{t+1}}.
\end{equation}
\end{lemma}

\begin{proof}[Proof of Lemma~\ref{lemma-fw-gradient-error-l1}]
This proof is similar to the proof of Lemma 1 in \cite{xie2020efficient}, except that we consider the $\|\cdot\|_1$ norm and its dual norm $\|\cdot\|_\infty$, and apply the Proposition~\ref{prop-vector-value-mds} in a different way. Reformulating $\Delta_t = d_t - \nabla F(\theta_t)$ as the sum of a martingale difference sequence. For $t\geq 1$, we have
\begin{equation}
    \begin{aligned}
        \Delta_t 
        &= \nabla f(\theta_t, x_t) + (1-\rho_t) (d_{t-1} - \nabla f(\theta_{t-1}, x_{t})) - \nabla F(\theta_t) \\
        &= (1-\rho_t) \epsilon_{t-1} + \rho_t (\nabla f(\theta_t, x_t) - \nabla F(\theta_t)) + \dots  \\
        &\quad + (1-\rho_t) \big(\nabla f(\theta_t, x_t) - \nabla f(\theta_{t-1}, x_t) - (\nabla F(\theta_t) - \nabla F(\theta_{t-1}))\big) \\
        &=  \prod_{k=2}^t (1-\rho_k) \epsilon_1
        + \sum_{\tau=2}^t \bigg( \rho_\tau \prod_{k=\tau+1}^t (1-\rho_k) \big( \nabla f(\theta_\tau, x_\tau) - \nabla F(\theta_\tau) \big) +\dots \\
        &\quad + \prod_{k=\tau}^t (1-\rho_k)\big(\nabla f(\theta_\tau, x_\tau) - \nabla f(\theta_{\tau-1}, x_\tau) - (\nabla F(\theta_\tau) - \nabla F(\theta_{\tau-1})\big)\bigg) \triangleq \zeta_{t, 1} + \sum_{\tau=2}^t \zeta_{t, \tau}
    \end{aligned}
\end{equation}
Recall that $\Delta_1 = \nabla f(\theta_1, x_1) - \nabla F(\theta_1)$. And we observe that $\mathbb{E}[\zeta_{t, \tau}| \mathcal{F}_{\tau-1}]=0$ where $\mathcal{F}_\tau$ is the $\sigma$-field generated by $\{x_1, x_2, ..., x_{\tau-1} \}$. Therefore, $\{\zeta_{t, \tau}\}_{\tau=1}^t$ is a martingale difference sequence. In what follows, we derive upper bounds of $\|\zeta_{t, \tau}\|_\infty$ . We start by observing that for any $\tau=1,2,...,t$,
\begin{equation}
    \begin{aligned}
        \prod_{k=\tau}^t (1-\rho_k) 
        &= \prod_{k=\tau}^t(1- \frac{1}{k+1}) 
        = \prod_{k=\tau}^t \frac{k}{k+1} =\frac{\tau}{t+1}
    \end{aligned}
\end{equation}
We can bound $\|\zeta_{t,1}\|_\infty$ as follows:
\begin{equation*}
    \begin{aligned}
        \|\zeta_{t, 1} \|_\infty \leq \frac{1}{t+1} \| \nabla f(\theta_1, x_1) - \nabla F(\theta_1)\|_\infty \leq \frac{\gradvar }{t+1} \coloneqq c_{t,1},
    \end{aligned}
\end{equation*}
where the first inequality is due to Assumption~\ref{assumption-grad-bound}. For $\tau >1$, 
\begin{equation*}
    \begin{aligned}
       \| \zeta_{t, \tau}\|_\infty  
       &\leq   \prod_{k=\tau}^t (1-\rho_k)\big(\| \nabla f(\theta_\tau, x_\tau) - \nabla f(\theta_{\tau-1}, x_\tau)\|_\infty + \|\nabla F(\theta_\tau) - \nabla F(\theta_{\tau-1})\|_\infty \big) + \dots  \\
       &\quad + \rho_\tau \prod_{k=\tau+1}^t (1-\rho_k) \| \nabla f(\theta_\tau, x_\tau) - \nabla F(\theta_\tau)\|_\infty  \\
       &\leq 2\smooth \|\theta_\tau - \theta_{\tau-1} \|_1 \prod_{k=\tau}^t (1-\rho_k) + \gradvar \rho_\tau \prod_{k=\tau+1}^t (1-\rho_k) \\
       &= 2\smooth \eta_{\tau-1} \| v_{\tau-1} - \theta_{\tau-1} \|_1 \prod_{k=\tau}^{t} (1-\rho_k) + \gradvar \rho_\tau \prod_{k=\tau+1}^t (1-\rho_k) \\
       &\leq \frac{2 \smooth \diam + G}{t+1} \coloneqq c_{t,\tau}.
    \end{aligned}
\end{equation*}
where the second inequality follows from Assumption~\ref{assumption-smooth} and \ref{assumption-grad-bound}, and the last inequality is due to $\eta_\tau = \rho_\tau = \frac{1}{\tau+1}$ and the definition of $\diam$. 
Now we denote the $i$-th element of $\Delta_t$ as $\Delta_{t,i}$ for $i\in 1,2,...,d$. According to Proposition~\ref{prop-vector-value-mds}, we have
\begin{equation}\label{eq-mds-single-value}
    \mathbb{P} \big(|\Delta_{t,i}| \geq \lambda \big)\leq 4\exp \bigg(-\frac{\lambda^2}{4\sum_{\tau=1}^t c^2_{t,\tau}} \bigg).
\end{equation}
We can bound $\sum_{\tau=1}^t c^2_{t,\tau}$ as
\begin{equation*}
    \begin{aligned}
        \sum_{\tau=1}^t c^2_{t,\tau} = c^2_{t,1} + \sum_{\tau=2}^t c^2_{t,\tau} = \bigg(\frac{G}{t+1}\bigg)^2 + \sum_{\tau=2}^t  \bigg(\frac{2\smooth \diam + \gradvar}{t+1} \bigg)^2 \leq \sum_{\tau=1}^t  \bigg(\frac{2\smooth \diam + 2\gradvar}{t+1} \bigg)^2 \leq \frac{4(\smooth \diam +\gradvar)^2}{t+1}.
    \end{aligned}
\end{equation*}
Plugging in the above bound and and setting $\lambda = \frac{4(\smooth \diam+\gradvar)\sqrt{\log(4d/\smallprob)}}{\sqrt{t+1}}$, for some $\smallprob\in(0,1)$, we have with probability $1-\smallprob/d$,
\begin{equation*}
    |\Delta_{t,i}| \leq \frac{4(\smooth \diam+\gradvar)\sqrt{\log(4d/\smallprob)}}{\sqrt{t+1}}
\end{equation*}
Then 
\begin{equation*}
    \mathbb{P}(\|\Delta_t\|_\infty\leq \lambda) = 1 - \mathbb{P}(\|\Delta_t\|_\infty> \lambda) \geq 1 - \sum_{i=1}^d \mathbb{P}(|\Delta_{t,i}|\geq \lambda) = 1-\smallprob,
\end{equation*}
where the first inequality comes from the union bound. In other word, with probability at least $1-\smallprob$, we have
\begin{equation*}
    \| \Delta_t \|_\infty \leq \frac{4(\smooth \diam+\gradvar)\sqrt{\log(4d/\smallprob)}}{\sqrt{t+1}}.
\end{equation*}
\end{proof}

Now we are ready to prove Theorem~\ref{thm-convergence-l1-general}.

\begin{proof}[Proof of Theorem~\ref{thm-convergence-l1-general}]

We denote $h_t=F(\theta_t) - F(\theta^*)$, and $\tilde{v}_t\coloneqq \arg\min_{v\in \mathcal{C}}(d_t, v)$. We start from $\beta$-smoothness:
\begin{equation}\label{eq-fw-general-convex}
\begin{aligned}
    h_{t+1} 
    &\leq h_t + \eta_t \langle \nabla F(\theta_t) , v_t - \theta_t \rangle + \frac{\eta_t^2\smooth\diam^2}{2} \\
    &= h_t + \eta_t \langle d_t , v_t - \theta_t \rangle - \eta_t \langle d_t - \nabla F(\theta_t), v_t-\theta_t \rangle+ \frac{\eta_t^2\smooth\diam^2}{2} \\
    &= h_t + \eta_t \langle d_t , \tilde{v}_t - \theta_t \rangle - \eta_t \langle d_t - \nabla F(\theta_t), v_t-\theta_t \rangle+ \frac{\eta_t^2\smooth\diam^2}{2} + \eta_t \langle d_t, v_t - \tilde{v}_t\rangle \\
    &\leq h_t + \eta_t \langle d_t , \theta^* - \theta_t \rangle - \eta_t \langle d_t - \nabla F(\theta_t), v_t-\theta_t \rangle+ \frac{\eta_t^2\smooth\diam^2}{2} + \eta_t \langle d_t, v_t - \tilde{v}_t\rangle\\
    &= h_t + \eta_t \langle \nabla F(\theta_t) , \theta^* - \theta_t \rangle + \eta_t \langle d_t - \nabla F(\theta_t), \theta^* - v_t \rangle+ \frac{\eta_t^2\smooth\diam^2}{2} + \eta_t \langle d_t, v_t - \tilde{v}_t\rangle\\
    &\leq (1 - \eta_t) h_t + \eta_t \diam \| d_t - \nabla F(\theta_t)\|_\infty + \frac{\eta_t^2\smooth\diam^2}{2} + \eta_t \langle d_t, v_t - \tilde{v}_t\rangle.
\end{aligned}
\end{equation}
To upper bound $\eta_t\langle d_t, {v}_t - \tilde{v}_t  \rangle$, notice that  \begin{align}\label{eq-fw-noise-error}
    \langle d_t, {v}_t - \tilde{v}_t  \rangle &= \min_{v\in \mathcal{C}} \big(\langle v, d_t\rangle +\textbf{n}_v^t\big)- \min_{v\in \mathcal{C}} \langle v, d_t\rangle
    \leq 2\max_{v = 1,\dots,2d} |\textbf{n}_v^t|
\end{align}
with $\textbf{n}_v^t\overset{i.i.d.}{\sim}$ Laplace$(0,\dfrac{4\diam (\smooth \diam+\lip)\sqrt{\log n \cdot \log (1/\delta)}}{\sqrt{t} \varepsilon})$ , we have by integrating the tail density
\begin{align*}
   \mathbb{P}(\max_{v}\lvert \textbf{n}_{v}^t\rvert > \lambda)&\leq \sum_{v=1}^{2d} P(\lvert \textbf{n}_v^t\rvert > \lambda )
   \leq 2d\exp\Big(- \dfrac{\sqrt{t} \varepsilon \lambda_t}{4\diam(\smooth \diam+\lip)\sqrt{\log n \cdot \log (1/\delta)}}\Big).
\end{align*}
selecting $\lambda_t = \dfrac{4\diam(\smooth \diam+\lip)\sqrt{\log n \cdot \log (1/\delta)}}{\sqrt{t}\varepsilon } \cdot\log(2d/\smallprobtail) $
we get then with probability at least $1-\smallprobtail$,
\begin{align}\label{eq-max-laplace-noise}
    \max_v |\textbf{n}_v^t|\leq \dfrac{4\diam(\smooth \diam+\lip)\sqrt{\log n \cdot \log (1/\delta)}}{\sqrt{t}\varepsilon } \cdot\log(2d/\smallprobtail).
\end{align}

According to Eq. (\ref{eq-fw-general-convex}), (\ref{eq-fw-noise-error}), (\ref{eq-max-laplace-noise}) and Lemma~\ref{lemma-fw-gradient-error-l1}, at iteration $t$, we have with probability at least $1-t(\smallprob + \smallprobtail)$,  
\begin{equation}\label{eq-fw-before-induction}
\begin{aligned}
     h_{t+1}&\leq (1-\eta_t)h_t + \frac{\eta_t^2 \smooth \diam^2}{2}+\dots \\
     &\quad +  \frac{\eta_t}{\sqrt{t+1}} \underbrace{\bigg(8\diam(\smooth \diam+\gradvar)\sqrt{\log(4d/\smallprob)} +\dfrac{16\diam (\smooth \diam+\lip)\log(2d/\smallprobtail)\sqrt{\log n \cdot \log (1/\delta)} }{\varepsilon } \bigg)}_{A}\\
     & = (1-\eta_t)h_t + \frac{\smooth \diam^2}{2(t+1)^2} + \frac{A}{(t+1)^{3/2}}.
\end{aligned}
\end{equation}
Now we prove $h_t\leq \frac{3}{\sqrt{t+1}}(\smooth \diam^2 + A)$ by induction. For $t=1$, we have
\begin{equation*}
    h_2 \leq \frac{1}{2} \big(F(\theta_1) - F(\theta^*)\big)+  \frac{\smooth \diam^2}{8} + \frac{A}{3^{3/2}} \leq \frac{3}{\sqrt{2}} (\smooth \diam^2 +A), 
\end{equation*}
where the last inequality is due to $F(\theta_1) - F(\theta^*) \leq \frac{\smooth \diam^2}{2}$ by the smoothness of $F$. Now we suppose $h_t\leq \frac{3}{\sqrt{t+1}}(\smooth \diam^2 + A)$ for $t\geq 1$. For $t+1$, according to Eq. (\ref{eq-fw-before-induction}), we have
\begin{equation*}
    \begin{aligned}
        h_{t+1} - \frac{3}{\sqrt{t+2}} (\smooth \diam^2 + A) 
        &\leq 3 (\smooth \diam^2 +A) \bigg( \frac{1}{\sqrt{t+1}} - \frac{1}{\sqrt{t+2}} \bigg) - \frac{2(\smooth \diam^2 +A)}{(t+1)^{3/2}} \\
        &\leq \frac{3(\smooth\diam^2 +A)}{2(t+1)^{3/2}} - \frac{2(\smooth \diam^2 +A)}{(t+1)^{3/2}} \leq 0,
    \end{aligned}
\end{equation*}
where the second inequality is due to $\frac{1}{(t+1)^{1/2}} - \frac{1}{(t+2)^{1/2}} \leq \frac{1}{2(t+1)^{3/2}}$. And now we conclude the proof by setting $\smallprob=\smallprobtail=\frac{\alpha}{2n}$.
 
\end{proof}

\subsection{Proof of Theorem~\ref{thm-convergence-l1-strongly}}\label{proof-l1-strongly}

\begin{proof}[Proof of Theorem~\ref{thm-convergence-l1-strongly}]

We define $\tilde{v}_t\coloneqq \arg\min_{v\in \mathcal{C}}(d_t, v)$ and $\tilde{\theta}_t\coloneqq \arg\max_{\theta\in \mathcal{C}} (\langle \nabla F(\theta_t), \theta_t - \theta \rangle)^2$ in Lemma~\ref{lemma-fw-strongly-convex}. And we denote that $h_t = F(\theta_t) - F(\theta^*)$. According to $\smooth$-smoothness, we have
\begin{equation}\label{eq-strongly-convex-smooth}
\begin{aligned}
    h_{t+1} &\leq F(\theta_t)+\eta_t \langle \nabla F(\theta_t),v_t-\theta_t\rangle + \frac{\eta_t^2\smooth\diam^2}{2} \\
    & = h_t+\eta_t \langle \nabla F(\theta_t)-d_t,v_t-\theta_t\rangle +\eta_t \langle d_t , \tilde{v}_t-\theta_t\rangle + \dfrac{\eta_t^2\smooth\diam^2}{2}  + \eta_t\langle d_t, {v}_t - \tilde{v}_t  \rangle \\
    & \leq h_t +\eta_t \lVert\nabla F(\theta_t)-d_t\rVert_\infty\diam +\eta_t \langle d_t,\tilde{\theta}_t-\theta_t\rangle+ \dfrac{\eta_t^2\smooth\diam^2}{2} + \eta_t\langle d_t, {v}_t - \tilde{v}_t  \rangle  \\
    &\leq h_t +2\eta_t \lVert\nabla F(\theta_t)-d_t\rVert_\infty\diam +\eta_t\langle \nabla F(\theta_t), \tilde{\theta}_t-\theta_t\rangle+ \dfrac{\eta_t^2\smooth\diam^2}{2}+ \eta_t\langle d_t, {v}_t - \tilde{v}_t  \rangle \\
    &\leq h_t +2\eta_t \lVert\nabla F(\theta_t)-d_t\rVert_\infty\diam -\eta_t\boundary\sqrt{2\strongly h_t } + \dfrac{\eta_t^2\smooth\diam^2}{2}+ \eta_t\langle d_t, {v}_t - \tilde{v}_t  \rangle,
\end{aligned}
\end{equation}
where the last inequality follows from Lemma~\ref{lemma-fw-strongly-convex}. According to Eq.  (\ref{eq-fw-noise-error}), (\ref{eq-max-laplace-noise}) and (\ref{eq-strongly-convex-smooth}), Lemma~\ref{lemma-fw-gradient-error-l1}, at iteration $t$, we have with probability at least $1-t(\smallprob + \smallprobtail)$,
\begin{equation}\label{eq-fw-before-induction-strongly}
    \begin{aligned}
     h_{t+1}&\leq \sqrt{h_t}(\sqrt{h_t} - \eta_t \boundary \sqrt{2\strongly}) + \frac{\eta_t^2 \smooth \diam^2}{2} +\dots \\
     &\quad +  \frac{\eta_t}{\sqrt{t+1}} \underbrace{\bigg(8\diam(\smooth \diam+\gradvar)\sqrt{\log(4d/\smallprob)} +\dfrac{16\diam (\smooth \diam+\lip)\log(2d/\smallprobtail)\sqrt{\log n \cdot \log (1/\delta)} }{\varepsilon } \bigg)}_{A}
    \end{aligned}
\end{equation}
Now the claim holds by induction. For simplicity, we denote 
\begin{equation*}
    B \coloneqq \frac{9(\smooth \diam^2 + A)^2}{\boundary^2\strongly}.
\end{equation*}
Firstly, for $t=1$, according to Eq. (\ref{eq-fw-before-induction-strongly}) we have 
\begin{equation*}
    h_2 \leq F(\theta_1) - F(\theta^*) + \frac{\smooth \diam^2}{8} + \frac{A}{3^{3/2}} \leq \frac{B}{2}.
\end{equation*}
where the last inequality is due to Lemma~\ref{lemma-fw-strongly-convex} and the fact that $F(\theta_1) - F(\theta^*) \leq \frac{\smooth \diam^2}{2}$.
Suppose that $h_t\leq \frac{B}{t+1}$ for some $t\geq 1$. There are two cases.
\begin{case}
$\sqrt{h_t} - \eta_t \boundary \sqrt{2\strongly} \leq 0$:
\end{case}
Then since $\eta_t=\frac{1}{t+1}$, Eq. (\ref{eq-fw-before-induction-strongly}) yields
\begin{equation*}
\begin{aligned}
    h_{t+1} 
    &\leq \frac{\smooth \diam^2}{2(t+1)^2} + \frac{A}{(t+1)^{3/2}} \leq \frac{\smooth D^2 + A}{t+1} \leq \frac{2}{t+2}\frac{(\smooth \diam^2+A)^2}{\boundary^2 \strongly} \leq \frac{B}{t+2}.
\end{aligned}
\end{equation*}
where the third inequality is due to Lemma~\ref{lemma-fw-strongly-convex} and the last inequality is from the definition of $B$.

\begin{case}
$\sqrt{h_t} - \eta_t \boundary \sqrt{2\strongly} > 0$:
\end{case}
According to Eq. (\ref{eq-fw-before-induction-strongly}) and the assumption that $h_t\leq \frac{B}{t+1}$, we have
\begin{equation}\label{eq-fw-case-2}
    \begin{aligned}
        h_{t+1} - \frac{B}{t+2} &\leq B\bigg(\frac{1}{t+1} - \frac{1}{t+2}\bigg) + \frac{\smooth \diam^2}{2(t+1)^2} + \frac{A}{(t+1)^{3/2}} - \frac{\boundary \sqrt{2 \strongly B}}{(t+1)^{3/2}}  \\
        &\leq \frac{1}{(t+1)^{3/2}} \bigg(\frac{B}{(t+1)^{1/2}} + \smooth \diam^2 + A - \boundary \sqrt{2\strongly B} \bigg)\\
        &\leq \frac{1}{(t+1)^{3/2}} \bigg(\frac{B}{(t+1)^{1/2}} - 3(\smooth \diam^2 + A) \bigg),
    \end{aligned}
\end{equation}
where the last inequality comes from the definition of $B$. Define
\begin{equation*}
    t_0 \coloneqq \inf \{ t\geq 1: \frac{B}{(t+1)^{1/2}}  - 3(\smooth \diam^2 + A) \leq 0\}.
\end{equation*}
According to the definition of $B$, $t_0$ exists. For any $t\geq t_0$, the RHS of Eq. (\ref{eq-fw-case-2}) is negative, and the proof is done. For those $t< t_0$, we have 
\begin{equation*}
    3(\smooth \diam^2 + A) \leq \frac{B}{(t+1)^{1/2}},
\end{equation*}
which is equivalent to
\begin{equation*}
    \frac{3(\smooth \diam^2 + A)}{(t+1)^{1/2}} \leq \frac{B}{t+1}.
\end{equation*}
To conclude the proof, it suffices to show that the following inequality holds,
\begin{equation}\label{eq-l1-strongly-second-induction}
    h_{t} \leq \frac{3(\smooth \diam^2 + A)}{(t+1)^{1/2}}.
\end{equation}
which is demonstrated in Theorem \ref{thm-convergence-l1-general}. Finally, we conclude the proof be setting $\smallprob=\smallprobtail=\frac{\alpha}{2n}$.
\end{proof}

\section{Proofs of Section~\ref{sec:bandits}}
In this section we establish the privacy protection for our Algorithm~\ref{algo:High-dim-bandit} and the convergence result for the forced-sample estimators and full-sample estimators.
We prove the convergence of estimators for any given arm $i$ and use $\theta_t$ to represent $\theta_{t,i}$ and $\theta^{*}$ to represent $\theta^{*}_{i}
$ for notation simplicity.

\subsection{Proof of Theorem~\ref{thm:privacy-bandits}}
\label{pf:private-bandit}
\begin{proof}
By post-processing property, we only need to guarantee that the sequence $(\theta_1,\dots,\theta_T)$ is $(\varepsilon,\delta)$ differentially private. In fact, we have for each sequence $\{\nu_{i+1},\dots,\nu_T\}\subset \mathcal{C} $. Suppose condition on $a_{i}(\mathcal{D}) = j_{i}$ and  $a_{i}(\mathcal{D}') = j_i'$, we have then \begin{align*}
    &\dfrac{ P(\theta_{i+1} = \nu_{i+1},\dots,\theta_T = \nu_T\lvert \mathcal{D})  }{ P(\theta_{i+1} = \nu_{i+1},\dots,\theta_T = \nu_T \lvert \mathcal{D}' )  } = \dfrac{P(\theta_{i+1} = \nu_{i+1}\lvert \mathcal{D})}{P(\theta_{i+1} = \nu_{i+1} \lvert \mathcal{D}')} \\
    &= \dfrac{P(\theta_{i+1,1} = \nu_{i+1,1}\dots,\theta_{i+1,K} = \nu_{i+1,K} \lvert \mathcal{D})}{P(\theta_{i+1,1} = \nu_{i+1,1}\dots,\theta_{i+1,K} = \nu_{i+1,K}\lvert \mathcal{D}')}\\
    & = \dfrac{P(\theta_{i+1,j_i} = \nu_{i+1,j_i},\theta_{i+1,j_i'} = \nu_{i+1,j_i'} \lvert \mathcal{D})}{P(\theta_{i+1,j_i'} = \nu_{i+1,j_i'},\theta_{i+1,j_i} = \nu_{i+1,j_i}\lvert \mathcal{D}')}
\end{align*}
Now by the synthetic update method, we have the above ratio is smaller or equal than $\varepsilon$ with probability at least $1-\delta$, which implies the $(\varepsilon,\delta)$-differential privacy guarantee of $(\theta_1,\dots,\theta_T)$.
\end{proof}

\begin{lemma}
\label{lm:candidate-set}
As long as $t_0$ is selected so that with probability at least $1-\alpha$,  \begin{align*}
    \dfrac{\SCerror(\alpha)}{\sqrt{t_0}} \leq \min\{\dfrac{h_{sub}}{4\smooth \Lx}, \ell\}.
\end{align*}
Then we have with probability at least $1-\alpha$, the following claim holds for all $t_0\leq t\leq T$ and $i\in K_{\text{opt}},
$ \begin{align*}
    &\hat{U}_t \cap U^c = \emptyset,\\
    &u\lambda   \leq \E[X_{t,i}X_{t,i}^\top ]\\
    & a_t^* \in \hat{K}_t.
\end{align*}
In particular, that implies the GLM loss\begin{align*}
F_t(\theta):=    \E[ f(\theta;X_{t,i},r_{t,i}) \lvert \mathcal{F}_{t-1} ]
\end{align*}
is $\strongly \lambda u $-strongly convex in $\ell_1$-geometry.
\end{lemma}
\begin{proof}
    Firstly, notice that as long as $\sup_{i\in [K]}\lVert \theta_{t_0,i } - \theta^*_{i} \rVert_1< \dfrac{h_{sub}}{4\beta M} $, we have for each $t$, denote $i_t: = \text{argmax}_{i\in [K]} \zeta(X_t^\top\theta_{t_0, i} ),i_t^*: = \text{argmax}_{i\in [K]} \zeta(X_t^\top\theta_{i}^* ),$  then \begin{align*}
        P( \hat{K}_t\cap K_{sub} \neq \emptyset ) &= P(\exists j\in K_{sub}\text{ s.t. }  \zeta(X_t^\top\theta_{t_0,j} ) > \zeta(X_t^\top\theta_{t_0,i_t} ) - h_{sub}/2   )\\
        &\leq P(\exists j\in K_{sub}c\text{ s.t. }  \zeta(X_t^\top\theta_{t_0,j} ) > \zeta(X_t^\top\theta_{t_0,i_t^*} ) - h_{sub}/2   )\\
        & \leq P(\exists j\in K_{sub}\text{ s.t. }  \zeta(X_t^\top\theta_{j}^* ) + \dfrac{h_{sub}  }{4} > \zeta(X_t^\top\theta_{i_t^*}^* ) - \dfrac{3 h_{sub}}{4}   )\\
        & = 0.
    \end{align*}
    Thus the first claim holds.\\
    To prove the second claim, notice that for every $t\geq K*t_0$, we have condition on the  $\sup_{i\in [K]}\lVert \theta_{t_0,i} - \theta_i^* \rVert_1< \dfrac{h_{sub}}{4\beta M} $, for every $i\in K_{opt}$, \begin{align*}
        P(a_t = i) & \geq P(\hat{K}_t =\{i\} )\\
                   & \geq P(\hat{K}_t =\{i\}, X_t \in U_i )\\
                   & = P( \zeta(X_t^\top\theta_{Kt_0,i})> \max_{j\neq i}\zeta(X_t^\top\theta_{Kt_0, j})+h_{sub}/2,\zeta(X_t^\top\theta_{i}^*)> \max_{j\neq i}\zeta(X_t^\top\theta_{j}^*)+h_{sub})\\
                   & \geq P( \sup_{i\in [K]}\lVert \theta_{s_0,i} - \theta_i^* \rVert_1< \dfrac{h_{sub}}{4\beta M}   ,\zeta(X_t^\top\theta_{i}^*)> \max_{j\neq i}\zeta(X_t^\top\theta_{j}^*)+h_{sub})\\
                   & = P(X_t\in U_j) \geq u.
    \end{align*}
    Thus we have \begin{align*}
        \lambda u \leq  P(X_t\in U_i)\cdot \E[X_{t}X_{t}^\top\lvert X_t\in U_i ] \leq  \E[X_{t,i}X_{t,i}^\top].
    \end{align*}
    To prove the third claim, note that for any $i_t=\operatorname{argmax}_{i\in [K]} X_t^{\top}\theta_{t_0, i}$
    \begin{align*}
        \zeta(X_t^{\top}\theta_{t_0, i_t}) - \zeta(X_t^{\top}\theta_{t_0, i_t^{*}})&= (\zeta(X_t^{\top}\theta_{t_0, i_t}) - \zeta(X_t^{\top}\theta_{t_0, i_t}^{*})) + (\zeta(X_t^{\top}\theta_{t_0, i_t}^{*}) - \zeta(X_t^{\top}\theta_{t_0, i_t^{*}}^{*}))+ \ldots \\
        & + \zeta(X_t^{\top}\theta_{t_0, i_t^{*}}^{*}) - \zeta(X_t^{\top}\theta_{t_0, i_t^{*}}))\\
        &\le \frac{\subopt}{2} + \zeta(X_t^{\top}\theta_{t_0, i_t}^{*}) - \zeta(X_t^{\top}\theta_{t_0, i_t^{*}}^{*})\\
        &\le \frac{\subopt}{2}.
    \end{align*}
    Thus $i_t^{*}\in \hat{K}_i$.
\end{proof}

\begin{lemma}\label{lemma-tv-norm} 
Consider the arm $i$ with $i\in K_{\text{opt}}$. Suppose the action $a_\tau$  $(\text{and }a_{\tau-1})$ depend only on $\theta_{\tau}(\text{and }\theta_{\tau-1})$, then we have for $P_i(\cdot\lvert \theta)$ the distribution of $X_{t,i}: = X_t\bm{1}\{a_t = i \}$ condition on $\theta$ (in particular, such distribution is independent of $t$), 
\begin{align*}
	\lVert \E[  \nabla F_{\tau-1}(\theta_{\tau-1}) - \nabla F_{\tau}(\theta_{\tau-1}) \lvert\mathcal{F}_{\tau -1} ]\rVert_\infty \leq 2 \smooth \lVert \theta_{\tau - 1}-\theta^*\rVert_1 \Lx \cdot \E_{\theta_\tau} \big[ \lVert P_i(\cdot\lvert \theta_{\tau-1}) - P_i(\cdot\lvert \theta_{\tau}) \rVert_{TV}\lvert \mathcal{F}_{\tau -1} \big]
\end{align*}
Moreover, when both $a_{\tau-1}$ and $a_{\tau}$ are greedy actions, we have \begin{align*}
	\lVert P_i(\cdot\lvert \theta_{\tau-1}) - P_i(\cdot\lvert \theta_{\tau}) \rVert_{TV}\leq 2 \smooth \eta_{\tau-1}D.
\end{align*}
\end{lemma}
\begin{proof}
	Denote $P_i(\cdot \lvert \theta )$ as the distribution of $X_{t,i} $ under greedy action condition on $\theta$, and $\E^{\tau-1}[\cdot]$ the expectation condition on $\mathcal{F}_{\tau -1}$, then notice that $\nabla f(\theta; \bm{0},r) = \bm 0 $ for every $\theta,r$,  we have
\begin{align*}
	&\lVert \E[D_\tau\lvert \mathcal{F}_{\tau -1} ] \rVert_\infty \\
	= &\big\lVert \E^{\tau-1}_{X_{\tau-1,i}}[ \E_{r}[ \nabla f(\theta_{\tau-1,i};X,r)\lvert X ] ] - \E^{\tau-1}_{X_{\tau,i}}[ \E_{r}[ \nabla f(\theta_{\tau-1,i};X,r)\lvert X ] ] \big\rVert_\infty \\
	= & \lVert \int_{\mathcal{X}} \big(\zeta(x^\top\theta_{\tau -1,i }) -\zeta(x^\top\theta^*_i)\big) xdP_i (x\lvert \theta_{\tau - 1 } )  - \int_{\Theta} \int_{\mathcal{X}} \big( \big(\zeta(x^\top\theta_{\tau -1,i }) -\zeta(x^\top\theta^*_i)\big) x dP_i( x\lvert \theta_{\tau} ) dP(\theta_{\tau })\rVert_\infty         \\
	\leq & \int_{\Theta}\bigg[\int_{\mathcal X } \lvert (\zeta(x^\top\theta_{\tau -1 ,i }) -\zeta(x^\top\theta^*_i)\rvert\cdot  \lVert x\rVert_\infty \cdot \lvert p_i(x\lvert \theta_{\tau,i}) - p_i(x\lvert \theta_{\tau -1,,i}) \rvert d\nu \bigg]dP(\theta_{\tau})\\
\leq & 2\smooth \lVert \theta_{\tau-1,i} - \theta^*_i\rVert\Lx  \cdot \E^{\tau-1}_{\theta_\tau}\big[ \lVert P_i(\cdot\lvert \theta_{\tau-1}) - P_i(\cdot\lvert \theta_{\tau}) \rVert_{TV}\big]	 
\end{align*}
Thus the first part is proved. On the other hand, notice that \begin{align*}
	  \lVert \theta_{\tau}- \theta_{\tau -1}\rVert_1 = \eta_{\tau -1} \lVert v_{\tau - 1} - \theta_{\tau-1}\rVert_1 \leq \eta_{\tau - 1}\diam
\end{align*}
we get then \begin{align*}
	 \lVert P_i(\cdot\lvert \theta_{\tau-1}) - P_i(\cdot\lvert \theta_{\tau}) \rVert_{TV} &  = \dfrac{1}{2} \int_{\mathcal{X}} \lvert p_i(x\lvert \theta_{\tau -1}) - p_i(x\lvert \theta_{\tau} )\rvert d\nu \\
	 &\leq \dfrac{1}{2} \big(\int_{S}+\int_{S^c}\big)\cdot \big\lvert p_i(x\lvert \theta_{\tau -1}) - p_i(x\lvert \theta_{\tau} )\big\rvert d\nu 
\end{align*}
with $S: = \big\{x\in \mathcal{X}: \bm{1} \{a(x\lvert \theta_{\tau-1}) = i\} = \bm{1}\{ a(x\lvert \theta_{\tau}) = i\} \big\}$ and $a(x\lvert \theta)\in[K]$ is the greedy action given  context $x$ and estimator $\theta$, in particular we have $\bm 0 \in S$. Clearly we have the distribution of $X\bm{1}\{a(X\lvert \theta_{\tau} )  = i\} $ equals to the distribution of $X\bm{1}\{a(X\lvert \theta_{\tau-1} )  = i\} $  on $S$, thus \begin{align*}
	 \dfrac{1}{2} \big(\int_{S}+\int_{S^c}\big)\cdot \big\lvert p_i(x\lvert \theta_{\tau -1}) - p_i(x\lvert \theta_{\tau} )\big\rvert d\nu  & = \dfrac{1}{2}\int_{S^c} \big\lvert p_i(x\lvert \theta_{\tau -1}) - p_i(x\lvert \theta_{\tau} )\big\rvert d\nu.
\end{align*}
On the other hand, if we denote $p(z)$ the distribution of $X$, then by $\bm 0 \notin S^c$, \begin{align*}
	\int_{S^c} \big\lvert p_i(x\lvert \theta_{\tau -1} ) -  p_i(x\lvert \theta_{\tau } )\big\rvert  d\nu &\leq 2 \int_{\mathcal{X}} \int_{S^c} p_i(x\lvert \theta_{\tau -1},z)p(z) d\nu dz  \\
	&= 2\int_{S^c} p(z)\int_{S^c}  p_i(x\lvert \theta_{\tau -1},z) d\nu dz  \\
	&\leq 2 \int_{S^c} p(z) dz\\
	& = 2 P(X\in S^c ).
\end{align*}
And by assumption~\ref{as:margin-condition}
\begin{align*}
	P(X\in S^c) =& P\big(\bm{1} \{a(X\lvert \theta_{\tau}) = i\}\neq_d \bm{1} \{a(X\lvert \theta_{\tau -1}) = i\}   \big)\\
	 =& P\big( a(X\lvert \theta_{\tau}) = i, a(X\lvert \theta_{\tau-1}) \neq  i  \big) + P\big( a(X\lvert \theta_{\tau}) \neq i, a(X\lvert \theta_{\tau-1}) =  i  \big)\\
	 \leq&  P( \max_{j\neq i} X^\top(\theta_{\tau,i}- \theta_{\tau,j})>0, \max_{j\neq i} X^\top(\theta_{\tau-1,i}- \theta_{\tau-1,j})\le 0 )\\
	 &+ P( \max_{j\neq i} X^\top(\theta_{\tau,i}- \theta_{\tau,j})>0, \max_{j\neq i} X^\top(\theta_{\tau-1,i}- \theta_{\tau-1,j})\le 0 )\\
	  \leq &  2 P(\max_{j\neq i}X^\top(\theta_{\tau-1,i} - \theta_{\tau-1,j})< \eta_{\tau -1}D)\\
	  \leq & 2 \margin \eta_{\tau - 1}D.	   \end{align*}
	   We get \begin{align*}
	   	\lVert P_i(\cdot\lvert \theta_{\tau-1}) - P_i(\cdot\lvert \theta_{\tau}) \rVert_{TV}\leq 2 \margin \eta_{\tau-1}D.
	   \end{align*}
\end{proof}

\noindent Next we provide a complete version of Lemma~\ref{lm:context-gradient-error-total-variation}.
\begin{lemma} 
\label{lm:estimaton-error-bandit}
For each arm $i\in K_{\text{opt}}$, suppose that the greedy action begins to be picked at time $t_0$, then for any $t>t_0$ we have with probability at least $1-\alpha$,  
\begin{align*}
\Delta_t  &\lesssim \dfrac{\beta}{t}\big(M D +\SCerror(\alpha/2(d+t_0))  (M+\sqrt{\log((d+T)/\alpha)} )\sqrt{t_0} + \diam \Lx \margin \sum_{\tau = t_0+1}^t \lVert \theta_{\tau - 1,i}-\theta_{i}^*\rVert_1  \big) + \ldots \\
& \quad + \dfrac{(MD  +\beta )\sqrt{\log((d+ T)/\alpha)}}{\sqrt{t}},
\end{align*}
where
\begin{align*}
        \Delta_t & = \lVert d_t - \nabla F_t(\theta_t)\rVert_\infty,\quad 
    \SCerror(\alpha)  = \sqrt{\frac{9(\smooth \diam^2 + A(\alpha))^2}{u \strongly \lambda}},
\end{align*}
and $A$ is stated in Theorem~\ref{thm-convergence-l1-general}.
\end{lemma}
\begin{proof}
	
For each arm $i$, denote $X_{t,i}: = X_{t}\bm{1}\{a_t = i\},r_{t,i}: = r_t \bm{1}\{a_t = i\}$. 
Moreover we introduce $f_t(\theta) \coloneqq f(\theta;X_{t,a_t},y_t)$, $F_t(\theta)\coloneqq \E[\nabla f_t(\theta)\lvert \mF_{t-1}]$ and  $\Delta_t \coloneqq d_t - \nabla F_t(\theta_t)$. Then 

\begin{equation}
\label{eq:eps}
\begin{aligned}
		\Delta_t &=  d_t- \nabla F_t(\theta_t) \\
		&=  \nabla f_t(\theta_t)+ (1-\rho_t)( {d}_{t-1}-\nabla f_t(\theta_{t-1}) )  -\nabla F_t(\theta_t)\\
		& =  (1-\rho_t)\epsilon_{t-1} + \rho_t( \nabla f_t(\theta_t)- \nabla F_t(\theta_t )  )\\
		&\quad + (1-\rho_t) \big( \nabla f_t(\theta_t) - \nabla f_t(\theta_{t-1}) - (\nabla F_t(\theta_t)- \nabla F_{t-1}(\theta_{t-1}))   \big) \\
		&\leq  \prod_{k=2}^t (1-\rho_k) \epsilon_1  + \sum_{\tau = 2}^t\prod_{k=\tau}^t (1-\rho_k)\underbrace{\big( \nabla f_{\tau}(\theta_\tau) - \nabla  f_{\tau}(\theta_{\tau-1}) - (\nabla F_\tau (\theta_\tau)-\nabla F_{\tau-1}(\theta_{\tau-1}))  \big)}_{D_\tau}\\
		&\quad + \sum_{\tau = 2}^t
\rho_\tau \prod_{k=\tau+1}^t	(1-\rho_k) \big(\nabla f_{\tau} (\theta_\tau) - \nabla F_\tau  (\theta_\tau)\big)\\
&= \prod_{k=2}^t(1-\rho_k) \epsilon_1 + \underbrace{\sum_{\tau = 2}^t \prod_{k=\tau}^t (1-\rho_k) \big( D_\tau - \E[D_\tau\lvert \mathcal{F}_{\tau - 1} ] \big)}_{\text{I}}  \\
	 &\quad + \underbrace{\sum_{\tau = 2}^t
\rho_\tau \prod_{k=\tau+1}^t	(1-\rho_k) \big(\nabla f_{\tau} (\theta_\tau) - \nabla F_\tau  (\theta_\tau)\big)}_{\text{II}}+ \underbrace{{\sum_{\tau = 2}^t \E[\prod_{k=\tau}^t (1-\rho_k) D_\tau\lvert \mathcal{F}_{\tau-1}]}}_{\text{III}}.
\end{aligned}
\end{equation}

First we bound III.
Now by the theorem~\ref{thm-convergence-l1-strongly}, we have with probability at least $1-(t_0 - 1)\smallprob$, \begin{align*}
 \lVert	 \theta_{s,i}- \theta_i^*\rVert_1 \leq \dfrac{\SCerror((t_0 - 1)\smallprob)}{\sqrt{s}},\quad  \forall s\leq t_0 - 1.
\end{align*} 
Thus by the Lemma~\ref{lemma-tv-norm} and the fact that \begin{align*}
	\E_{\theta_{t_0}} \big[ \lVert P_i(\cdot\lvert \theta_{t_0-1}) - P_i(\cdot\lvert \theta_{t_0)} \rVert_{TV}\lvert \mathcal{F}_{t_0 -1} \big]\leq 1,
\end{align*}
we have with probability at least $1- (t_0-1) \smallprob$,
 \begin{align*}
	\sum_{\tau = 2}^t \prod_{k=\tau }^t(1-\rho_k)\E[ D_{\tau}\lvert \mathcal{F}_{\tau - 1}] & = \dfrac{t_0+1}{t+1} 2\beta \lVert \theta_{t_0-1} - \theta^*\rVert_1  \Lx \ +  \sum_{\tau = t_0 + 1}^t \dfrac{\tau +1}{t+1} 2\beta \lVert \theta_{\tau-1} - \theta^*\rVert_1  \Lx \cdot 2\margin \eta_{\tau - 1} \diam  \\
	&\leq \dfrac{\beta M}{t+1}\bigg( {2 \SCerror((t_0 - 1)\smallprob)}\cdot \sqrt{t_0+1} + \sum_{\tau = t_0 + 1}^t ({\tau +1}) 2 \lVert \theta_{\tau-1} - \theta^*\rVert_1   \cdot 2\margin \eta_{\tau - 1} \diam  \bigg)  
\end{align*}

Now to bound $\Delta_t$, it sufficient to bound I$+$II, i.e.,\begin{align*}
	\sum_{\tau = 2}^t \prod_{k=\tau}^t (1-\rho_k) \big( D_\tau - \E[D_\tau\lvert \mathcal{F}_{\tau - 1} ] \big)+ \sum_{\tau = 2}^t
\rho_\tau \prod_{k=\tau+1}^t	(1-\rho_k) \big(\nabla f _{\tau}(\theta_\tau) - \nabla F_\tau  (\theta_\tau)\big).
\end{align*}
For the second summation, we have by the same argument as in proof of Lemma~\ref{lemma-fw-gradient-error-l1}, with probability at least $1-\smallprobtail$  \begin{align*}
	\lVert  \sum_{\tau = 2}^t
\rho_\tau \prod_{k=\tau+1}^t	(1-\rho_k) \big(\nabla f_{\tau} (\theta_\tau) - \nabla F_\tau  (\theta_\tau)\big)\rVert_\infty \leq \dfrac{4\Lx \diam \sqrt{\log (4d/\alpha_2)} }{\sqrt{t+1}} 
\end{align*}
For the first summation,
notice that \begin{align*}
    \lVert D_{\tau}\rVert_\infty &\leq 2\Lx  \diam\eta_\tau + \lVert \nabla F_\tau(\theta_{\tau-1}) - \nabla F_{\tau - 1}(\theta_{\tau - 1}) \rVert_\infty\\
    & = 2\Lx  \diam \eta_\tau + \E \big[ \lVert \E[\nabla F_\tau(\theta_{\tau-1})\lvert \mathcal{F}_{\tau - 1}] - \nabla F_{\tau - 1}(\theta_{\tau - 1}) \rVert_\infty\big]\\
    &\leq 2 MD\eta_\tau + 2 \beta \lVert \theta_{\tau -1,i}-\theta_i^*\rVert_1 M \cdot \E[\lVert P_i(\cdot\lvert \theta_{\tau-1}) - P_i(\cdot\lvert \theta_{\tau}) \rVert_{TV}]
\end{align*}
Now by \begin{align*}
    \E[\lVert P_i(\cdot\lvert \theta_{\tau-1}) - P_i(\cdot\lvert \theta_{\tau}) \rVert_{TV}]\leq \left\{\begin{matrix} 
    & 0  & \tau < t_0 \\
    & 1 & \tau = t_0, \\
    & D\margin \eta_\tau  & \tau>t_0.
    \end{matrix}\right.
\end{align*}
And $\lVert \theta_{\tau -1,i}-\theta_i^*\rVert_1\leq \diam $ 
we have by setting $M_\tau = \tau\big( D_{\tau}-\E[D_\tau \lvert \mathcal{F}_{i-1}]\big)$, then \begin{align*}
    \lVert M_\tau \rVert_\infty &\leq \left\{\begin{matrix} 
    & 2MD & \tau < t_0 \\
    & 2M(D+\beta \SCerror(\smallprob) \sqrt{t_0}) & \tau = t_0,\text{ with probability at least $1-\smallprob$} \\
    & 2MD(1+\beta)  & \tau>t_0.
    \end{matrix}\right.
\end{align*}
And thus apply Azuma-Hoeffding's inequality to each components $M_{\tau,i}$ with $M_{t_0}$ replaced by $M_{t_0}': = M_{t_0}\bm{1}\{  M_{t_0}\leq  2M(D+\beta \SCerror(\smallprob) \sqrt{t_0})  \}$, we have with probability at least $1-d\smallprob,$  \begin{align*}
  & \lvert M_{t_0,i}' +  \sum_{\tau \neq  t_0} M_{\tau,i} \rvert \\
  &\leq \bigg( 2M\sqrt{ t_0 D^2 + (D+\beta \SCerror(\smallprob) \sqrt{t_0})^2 + (t-t_0) D^2(1+\beta)^2   } + \E[M_{t_0,i}'\lvert \mathcal{F}_{t_0-1} ] \bigg) \cdot\sqrt{\log(1/\alpha_1)}\\
  &\leq 4M\big[ D(1+\smooth)\sqrt{t}+ (D+\beta\SCerror(\smallprob) \sqrt{t_0})\big]\cdot\sqrt{\log(1/\alpha_1)}
\end{align*}
normalizing the summation by $t$ and notice that $M_{t_0}'\neq M_{t_0}$ with probability at most $\alpha_1$, we have then with probability at least $1-d\alpha_1$ \begin{align*}
    \sum_{\tau = 2}^t \prod_{k=\tau}^t (1-\rho_k) \big( D_\tau - \E[D_\tau\lvert \mathcal{F}_{\tau - 1} ] \big) \leq \big( \dfrac{2MD(1+\beta) }{\sqrt{t}} + \dfrac{D+\beta \SCerror(\smallprob) \sqrt{t_0}}{t } \big) \cdot\sqrt{\log(1/\alpha_1)}, \forall i \in [d]
\end{align*}
Now combining all bounds and set $\smallprob = \alpha(d+t_0)/2,\smallprobtail = \alpha/2$, we get  with probability at least $1-\alpha$,\begin{align*}
   & \lVert \Delta_t \rVert\\
   &\leq  \dfrac{\Lx \beta \diam }{t+1} + \big( \dfrac{2MD(1+\beta) }{\sqrt{t}} + \dfrac{D+\SCerror(\alpha(d+t_0)/2) \beta \sqrt{t_0}}{t } \big) \cdot\sqrt{\log(2(d+t_0)/\alpha)} + \dfrac{4\Lx \diam \sqrt{\log (8d/\alpha)} }{\sqrt{t+1}} \\
   &\quad +\dfrac{\beta M}{t+1}\bigg( {3 \SCerror((t_0-1)\alpha/(2(d+t_0)))}\cdot \sqrt{t_0+1} + \sum_{\tau = t_0 + 1}^t ({\tau +1}) 2 \lVert \theta_{\tau-1,i} - \theta_{i}^*\rVert_1   \cdot 2\margin \eta_{\tau - 1} \diam  \bigg) \\
   &\lesssim \dfrac{\beta}{t}\big(M D +\SCerror((d+t_0)\alpha/2)  (M+\sqrt{\log((d+T)/\alpha)} )\sqrt{t_0} + \diam \Lx \margin \sum_{\tau = t_0+1}^t \lVert \theta_{\tau - 1}-\theta^*\rVert_1  \big) \\
   &\quad + \dfrac{(MD  +\beta )\sqrt{\log((d+ T)/\alpha)}}{\sqrt{t}}
\end{align*}
as claimed.
\end{proof}

\subsection{Proof of Theorem~\ref{lm:induction-text}}
\label{sec:Proof_of_bandit_estimation}
\begin{proof}
Suppose at time $t$, we have with probability at least $1-\alpha$ that \begin{align*}
	 h_\tau  \leq  \dfrac{\Cpara(\alpha)}{\tau}, \text{(thus }  \lVert \theta_\tau - \theta^*\rVert_1\leq \sqrt{\dfrac{\Cpara(\alpha)}{ u\lambda \strongly \tau}}),\quad \forall \tau \leq t-1,
\end{align*}
then condition on such event,  from Lemma~\ref{lm:estimaton-error-bandit} and the same argument in \eqref{eq-max-laplace-noise}, we have for $h_t\coloneqq F_t(\theta_t) - F_t(\theta^*)$, with probability at least $1-\alpha_1-\alpha_2$,  
\begin{align*}
h_{t}&\leq h_{t-1} +2\eta_t \lVert\nabla F_t(\theta_t)-d_t\rVert_\infty\diam -\eta_t\boundary\sqrt{2\strongly h_{t-1}} + \dfrac{\eta_t^2L\diam^2}{2}+ \eta_t\langle d_t, {v}_t - \tilde{v}_t \rangle\\
&\leq h_{t-1} - \eta_t\boundary \sqrt{2\mu h_{t-1}}+ 2\eta_{t}\dfrac{\diam \beta}{t}\big(M D +\SCerror(\alpha_1/2(d+t_0))  (M+\sqrt{\log((d+T)/\alpha_1)} )\sqrt{t_0}) + \ldots \\
&\quad + \frac{2 \eta_t \diam \Lx \margin \sqrt{t}\beta\Cpara^{1/2}}{\sqrt{\mu}t} + \dfrac{2 \eta_t D(MD  +\beta )\sqrt{\log((d+ T)/\alpha_1)}}{\sqrt{t}} + \frac{\eta_t^2 LD^2}{2} + \ldots \\
&\quad + \eta_t \dfrac{4\diam(\smooth \diam+\lip)\sqrt{\log T \cdot \log (1/\delta)}}{\sqrt{t}\varepsilon } \cdot 4\log(2d/\alpha_2)\\
&\coloneqq h_{t-1} - \eta_t\boundary \sqrt{2\mu h_{t-1}}+\frac{G_1\Cpara^{1/2}}{t^{3/2}} + \frac{G_2}{t^{3/2}\varepsilon} + \frac{G_3}{t^{3/2}} + \frac{G_4}{t^2}
,
\end{align*}
where 
\begin{align*}  
    G_1 &\coloneqq \frac{2 \beta \diam^2 \Lx \margin}{\sqrt{\mu}},\\
    G_2 &\coloneqq \dfrac{4\diam(\smooth \diam+\lip)\sqrt{\log T \cdot \log (1/\delta)}}{\varepsilon } \cdot 4\log(2d/\alpha_2),\\
    G_3 &\coloneqq D(MD  +\beta )\sqrt{\log((d+ T)/\alpha_1)},\\
    G_4 &\coloneqq D\beta \big(M D +\SCerror(\alpha_1/2(d+t_0))     (M+\sqrt{\log((d+T)/\alpha_1)} )\sqrt{t_0}\big) + \frac{L\diam^2}{2}.
\end{align*}
For notation simplicity we abbreviate $\Cpara$ for $\Cpara(\alpha)$ below.

\textbf{Case1: }$ h_{t-1}  - \eta_t\boundary \sqrt{2\mu h_{t-1}}\leq 0:$  i.e. 
\begin{align*}
 h_t&\le \frac{G_1\Cpara^{1/2}}{t^{3/2}} + \frac{G_2}{t^{3/2}\varepsilon} + \frac{G_3}{t^{3/2}} + \frac{G_4}{t^2}.
\end{align*} 
To ensure the induction, we need  \begin{align*}
	& \frac{G_1\Cpara^{1/2}}{t^{3/2}} + \frac{G_2}{t^{3/2}\varepsilon} + \frac{G_3}{t^{3/2}} + \frac{G_4}{t^2}\le \frac{\Cpara}{t}\\
    \iff & \frac{G_1\Cpara^{1/2}}{t^{1/2}} + \frac{G_2}{t^{1/2}\varepsilon} + \frac{G_3}{t^{1/2}} + \frac{G_4}{t}\le \Cpara
\end{align*}
As long as $t\ge \max\{G_1^2, \frac{G_2^2}{\log(dT/\alpha)}, G_3^2, \frac{G_4}{\log(dT/\alpha)}\}\Rightarrow t\ge \dfrac{\log(dT/\alpha)\log(T)}{\varepsilon^2}$, we can choose $\Cpara = \max\{4, \frac{9\log^2(dT/\alpha)}{\varepsilon^2}\}$ to satisfy the above inequality.\\
\textbf{Case2: $h_{t-1}-\eta_t\boundary \sqrt{2\mu h_{t-1}}> 0$ :} \begin{align*}
	 h_{t}  &= h_{t-1} - \eta_t\boundary \sqrt{2\mu h_{t-1}} + \frac{G_1\Cpara^{1/2}}{t^{3/2}} + \frac{G_2}{t^{3/2}\varepsilon} + \frac{G_3}{t^{3/2}} + \frac{G_4}{t^2}\\
	   &\leq \dfrac{\Cpara}{t} -\eta_{t} \gamma \frac{\sqrt{2\mu}\Cpara^{1/2}}{\sqrt{t}} + \frac{G_1\Cpara^{1/2}}{t^{3/2}} + \frac{G_2}{t^{3/2}\varepsilon} + \frac{G_3}{t^{3/2}} + \frac{G_4}{t^2},
\end{align*} 
i.e. \begin{align*}
	 h_{t+1} - \dfrac{\Cpara }{{ t+1}} &\leq   \big(\dfrac{\Cpara}{t} - \dfrac{\Cpara}{t+1}\big) -\eta_{t} \gamma \frac{\sqrt{2\mu}\Cpara^{1/2}}{\sqrt{t}} + \frac{G_1\Cpara^{1/2}}{t^{3/2}} + \frac{G_2}{t^{3/2}\varepsilon} + \frac{G_3}{t^{3/2}} + \frac{G_4}{t^2},
\end{align*}
i.e. we need to choose $\Cpara$ so that RHS is not greater than zero, \begin{align*}
\text{ i.e. }	& (\gamma\sqrt{2\mu}-G_1)\Cpara^{1/2}-\frac{\Cpara}{t^{1/2}}\ge \frac{G_2}{t^{1/2}\varepsilon} + \frac{G_3}{t^{1/2}} + \frac{G_4}{t}.
\end{align*}

Choosing $\boundary \geq  \frac{3G_1}{2\sqrt{2}\strongly}$, and  as long as \begin{align*}
    \Cpara &= \tilde{C}\frac{\log^2(dT/\alpha)}{\varepsilon^2} \text{ for some $\tilde{C}$ independent of $d, T$ and $\alpha$,} \\
    t&\ge \max\{\frac{G_2^2}{\log(dT/\alpha)}, G_3^2, \frac{G_4}{\log(dT/\alpha)}, \frac{\Cpara}{9\log(dT/\alpha)}\}\Rightarrow t\gtrsim \dfrac{\log(dT/\alpha)\log(T)}{\varepsilon^2},
\end{align*}
the claim holds. 

Finally we need to ensure the induction holds when $t=t_0$. Note that when $t\le t_0$, the convergence result is given by Theorem~\ref{thm-convergence-l1-strongly}.
Thus to ensure the induction holds we also need $\Cpara\ge \lambda \smooth \strongly (\SCerror(\alpha))^2$.
In conclusion, we choose $\alpha_1=\alpha_2=\frac{\alpha}{2T}$, $$\Cpara = \max\{4, \frac{9\log(dT)}{\varepsilon^2}, \frac{36\log^2(dT)}{(\varepsilon G_1)^2}, \lambda \smooth \strongly(\SCerror(\alpha))^2\}=O( \frac{\log^2(dT/\alpha)\log(T)}{\varepsilon^2}),$$ 
and 
$$t_0= \max\{G_1^2, \frac{G_2^2}{\log(dT/\alpha)}, G_3^2, \frac{G_4}{\log(dT/\alpha)}, \frac{\Cpara}{9\log(dT/\alpha)}\}=O( \frac{\log(dT/\alpha)\log(T)}{\varepsilon^2}),$$ 
to ensure the induction holds.
\end{proof}

\subsection{Proof of Theorem~\ref{thm:regret}}
\label{sec:proof-regret}
\begin{proof}
    We define event 
    \begin{align*}
    E_{t,1}&\coloneqq \big\{\hat{U}_{t} \cap U^c=\emptyset, a_{t}^{*} \in \hat{U}_{t} \big\},\\
    E_{t,2}&\coloneqq \big\{\sup_{i\in K_{opt}}\lVert \theta_{t,i} - \theta^{*}_i\rVert_1 \le \frac{\sqrt{\Cpara(\alpha/(4\lvert K_{opt}\rvert)}}{2\sqrt{t}}\big\},\\
    E_{t}&\coloneqq \big\{E_{t,1}, E_{t,2}, \Delta_t\le \frac{\Lx \smooth \sqrt{\Cpara(\alpha/(4\lvert K_{opt}\rvert))}}{\sqrt{t}} \big\}.
    \end{align*}
    and we use $\Cpara$ for $\Cpara(\alpha/(4\lvert K_{opt}\rvert))$ for notation simplicity in the following. 
    Using the similar argument as in Theorem~\ref{lm:induction-text}, we can verify that condition on the event $\sup_{i\in K_{opt}}\lVert \theta_{t,i} - \theta^{*}_i\rVert_1 \le \frac{\sqrt{\Cpara}}{2\sqrt{t}}$ and $\Delta_t\le \frac{\Lx \smooth \sqrt{\Cpara}}{\sqrt{t}}$, we must have $a_t = a_t^{*}$. Thus with probability at least $1-\frac{3\alpha}{4}$, the event $E_t \cup E_{t,1}^{c}\cup E_{t,2}^{c}$ holds, thus,
    \begingroup
    \allowdisplaybreaks
    \begin{align*}
        \text{Regret}(T) &= \big(\sum_{t\le t_0} + \sum_{t> t_0}\big) \big (\zeta(X_{t}^{\top} \theta^{*}_{t}) -\zeta(X_{t}^{\top} \theta^{*}_{a_t})\big ) \\
        &\le \sum_{t\le t_0} (\zeta(X_{t}^{\top} \theta^{*}_{t}) -\zeta(X_{t}^{\top} \theta^{*}_{a_t})) 
        + \sum_{t> t_0} (\zeta(X_{t}^{\top} \theta^{*}_{t}) -\zeta(X_{t}^{\top} \theta^{*}_{a_t}))  \bm{1}\{E_t \cup E_{t,1}^{c}\cup E_{t,2}^{c}\}. 
     \end{align*}
    \endgroup
    Note that the choice of $t_0=O(\frac{\log(dT/\alpha)\log(T)}{\varepsilon^2})$
    as stated in Appendix~\ref{sec:Proof_of_bandit_estimation} and Theorem~\ref{lm:induction-text} and the range of the reward can be bounded. 
    Now it remains to bound the second term. 
    Let $A_t = \frac{\Lx \smooth \sqrt{\Cpara}}{\sqrt{t}}\bm{1}\{\Delta_t\le \frac{\Lx \smooth \sqrt{\Cpara}}{\sqrt{t}}\}$ and the second term is upper bounded by $\sum_{t_0<t\le T}A_t$. 
    We have $\sum_{t_0<t\le T} (\frac{\Lx \smooth \sqrt{\Cpara}}{\sqrt{t}})^2 \le \Lx^2 \smooth^2 \Cpara \log(T)$. 
    Note that $P(A_t=\frac{\Lx \smooth \sqrt{\Cpara}}{\sqrt{t}})\le \margin \frac{\Lx \smooth \sqrt{\Cpara}}{\sqrt{t}}$ by Assumption~\ref{as:margin-condition} and thus $\E[\sum_{t_0<t\le T}A_t]\le \margin \Lx^2 \smooth^2 \Cpara\log(T)$. 
    We apply Hoeffding's inequality and can conclude that with probability at least $1-\alpha/4$
    \begin{align*}
        \sum_{t\ge t_0}A_t \le \margin \Lx^2 \smooth^2 \Cpara \log(T) + 2\Lx \smooth \sqrt{\Cpara\log(T)\log(4/\alpha)} .
        \end{align*}
    Putting all the terms together, and we arrive the desired conclusion.
    \end{proof}

\addcontentsline{toc}{section}{References}
\bibliography{ref}
\bibliographystyle{plainnat}
\end{document}